\documentclass[twoside]{article}
\usepackage{lmodern}
\usepackage{microtype}
\usepackage{amsmath,amssymb,mathtools,amsthm}
\usepackage{xcolor}
\usepackage{enumitem}
\usepackage{booktabs}
\usepackage{algorithm}
\usepackage{algpseudocode}
\usepackage{nicefrac}
\usepackage{float}
\usepackage{hyperref}
\allowdisplaybreaks

\newtheorem{theorem}{Theorem}
\newtheorem{corollary}{Corollary}
\newtheorem{lemma}{Lemma}
\newtheorem{assumption}{Assumption}

\newcommand{\mvalue}{m_{\mathrm{val}}}
\newcommand{\mupdate}{m_{\mathrm{up}}}
\newcommand{\Bvalue}{B_{\mathrm{val}}}
\newcommand{\Bgradient}{B_{\mathrm{g}}}
\newcommand{\qclient}{q_{\mathrm{c}}}
\newcommand{\qserver}{q_{0}}
\newcommand{\sigmavalue}{\sigma_{\mathrm{val}}}
\newcommand{\sigmagradient}{\sigma_{\mathrm{g}}}
\newcommand{\sigmaparticipation}{\sigma_{\mathrm{pp}}}

\newcommand{\E}{\mathbb{E}}
\newcommand{\Prob}{\mathbb{P}}
\newcommand{\ind}{\mathbf{1}}
\newcommand{\calQ}{\mathcal{Q}}
\DeclareMathOperator*{\argmin}{argmin}
\DeclareMathOperator{\dist}{dist}

\usepackage[preprint]{aistats2027}
\usepackage[round]{natbib}

\begin{document}

%
\runningtitle{Variance Reduction for Stochastic Switching in Functionally Constrained Federated Optimization}

%

\twocolumn[

\aistatstitle{Variance Reduction for Stochastic Switching in Functionally Constrained Federated Optimization}

\aistatsauthor{ Antesh Upadhyay \And Sang Bin Moon \And  Zhankun Luo \And Abolfazl Hashemi}

\aistatsaddress{ Elmore Family School of Electrical and Computer Engineering, Purdue University} ]

\begin{abstract}
Federated learning (FL) is increasingly used in safety-critical and fairness-sensitive applications, where learning objectives must be optimized subject to functional constraints. This setting is challenging due to costly projection onto the resulting feasible set, stochastic constraint information, partial client participation, and limited communication at edge devices. We propose a federated stochastic switching framework based on primal-only updates, avoiding dual-variable tuning for functionally constrained federated optimization. For convex objectives, the framework further supports compression with error feedback to control compression-induced bias. We establish high-probability convergence and feasibility guarantees under stochastic gradients and partial client participation, and extend the analysis to weakly convex objectives through proximal-stationarity guarantees. Our bounds isolate the effects of stochastic gradients, constraint-evaluation noise, and partial participation, and show that direct switching can retain a nonvanishing constraint-estimation error even with exact local constraint evaluations. To overcome this limitation, we develop variance-reducing constraint estimators that yield vanishing feasibility and stationarity guarantees with fixed per-round client participation and constraint-evaluation batch sizes. Experiments on Neyman-Pearson classification, fairness-constrained learning on the ADULT dataset, and constrained deep reinforcement learning demonstrate the effectiveness of the proposed framework.
\end{abstract}

\section{Introduction}
\label{sec:intro}
FL training over distributed clients, while preserving data locality, has become a standard paradigm in privacy, latency, and safety-sensitive applications such as mobile devices, autonomous fleets, and battery management systems~\citep{konevcny2016federated,mcmahan2017communication,kalashnikov2018scalable,kairouz2021advances,brohan2022rt,zhang2023batteryml}. Evidently, these applications demand learning that goes beyond loss minimization and enforces additional \emph{functional constraints}, such as fairness requirements, risk limits, or safety margins~\citep{agarwal2018reductions,zhang2023batteryml}. In this work, we study \emph{federated functionally constrained optimization} problems of the form
\begin{equation}
\label{eq:op_prb}
\begin{aligned}
    w^\star \in \argmin_{w\in\mathcal{X}}\quad
    &F(w):=\frac{1}{n}\sum_{i=1}^{n}f_i(w)\\
    \text{s.t.}\quad
    &G(w):=\frac{1}{n}\sum_{i=1}^{n}g_i(w)\leq 0,
\end{aligned}
\end{equation}
where $\mathcal{X}\subseteq\mathbb{R}^d$ is compact and convex, and $f_i$, $g_i$ denote the expected objective and constraint functions of client $i\in[n]$, respectively.

\noindent\textbf{Why this setting is challenging.}
Unlike \emph{simple set constraints} (e.g., box or $\ell_2$-ball constraints), the feasible region $\mathcal{W}:=\{w\in\mathcal{X}:G(w)\leq 0\}$ induced by the constraint function $G(w)$ typically does not admit an inexpensive closed-form Euclidean projection. Computing projections, proximal steps, or linear minimization oracles over $\mathcal{W}$ may itself require solving a nontrivial constrained optimization problem; see, e.g., \citet{condat2022murana,pokutta2024frank}. This difficulty is exacerbated in FL because $G(w)$ is distributed across clients and the server does not directly observe the population constraint value and must instead infer it from participating clients and stochastic data.

\begin{figure}
    \centering
    \includegraphics[width=\linewidth]{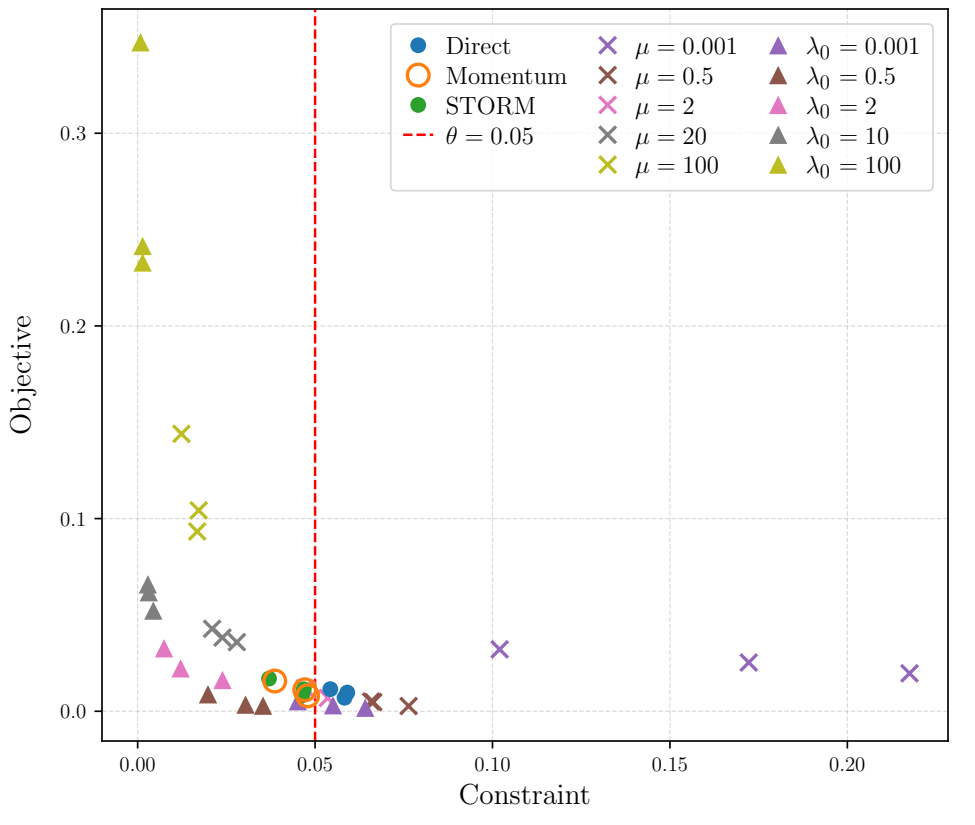}
    \vspace{-7mm}
    \caption{\textbf{NP classification.} Comparison of our method against penalty-based FedAvg and the primal-dual method. Each point represents a result from a separate experimental run with different random seeds. Circles (`$\bullet$') represent ours (direct, momentum, and STORM), while crosses (`$\times$') and triangles (`$\blacktriangle$') represent penalty-based ($\mu)$ and primal-dual ($\lambda_0$) FedAvg, respectively.}
    \label{fig:NPC_tradeoff_last}
    \vspace{-5mm}
\end{figure}

Other approaches, such as penalty methods or primal-dual schemes~\citep{hamedani2021primal,he2024federated,muller2024truly,kim2024fast}, avoid explicit projection onto $\mathcal{W}$ but introduce their own instabilities. As we demonstrate empirically, penalty-based/primal-dual heuristics in FL suffer from a ``tuning curse'': small penalties yield infeasible solutions, while large penalties induce ill-conditioning that hinders convergence; see Figure~\ref{fig:NPC_tradeoff_last}.  An appealing alternative is a \emph{switching} method~\citep{polyak1967general,lan2020algorithms,upadhyay2025optimization}. When the current iterate is sufficiently feasible, the method follows an objective (sub-)gradient; otherwise, it takes a constraint (sub-)gradient step. Switching therefore remains primal-only and avoids repeatedly projecting onto the functionally defined feasible region.

While \citet{he2024federated} address functional constraints in FL, their augmented-Lagrangian approach requires dual tuning and assumes full participation and uncompressed communication. Closest to our setting is Safe-EF~\citep{islamov2025safe}, for which we identify errors in both the deterministic and stochastic analyses; we detail these issues in Appendix~\ref{sec:safeef_discussion}. Even setting these errors aside, Safe-EF assumes full participation and does not account for client-sampling noise, while its stochastic guarantee requires an accuracy-dependent constraint-evaluation batch of order $\widetilde{\mathcal O}(\epsilon^{-2})$. In contrast, our estimators achieve vanishing feasibility with partial participation and fixed per-round constraint batches, and our analysis further covers $\rho$-weakly convex objectives through constrained proximal-stationarity guarantees.

To address these limitations, we develop a duality-free stochastic switching framework for functionally constrained FL. The method retains the simplicity of primal switching while allowing stochastic constraint observations, partial client participation, and communication compression. The main difficulty is that, in a federated stochastic implementation, the \emph{switching decision itself is noisy}. At round $k$, the server does not know $G(w_k)$ exactly; instead, it observes an estimate formed from a subset of clients and finitely many local samples. Hence, the resulting switch may differ from the one that would be made using the exact constraint.

This observation exposes a statistical limitation of classical switching methods in FL. With a fresh constraint estimate at every round, which we refer to as \emph{direct switching}, the estimation error does not vanish when the participation and local batch size are fixed. Consequently, the feasibility guarantee retains a non-vanishing statistical term even as the number of rounds grows. Under partial participation, this persists \emph{even with exact local constraint evaluations} because client sampling itself introduces uncertainty. This brings the central question of our work, can the statistical error in the switching estimate vanish while keeping the per-round constraint-evaluation budget fixed?

\noindent\textbf{Constraint tracking for stochastic switching.}
We address this question by tracking the evolving population constraint $G(w_k)$ across rounds. Our approach is motivated by recursive variance-reduction methods such as momentum~\citep{cutkosky2020momentum} and STORM~\citep{cutkosky2019momentum}, as well as the recent unified high-probability analysis of stochastic variance-reduced estimation~\citep{luo2026unified}. We consider three estimators. Direct switching uses only the current empirical constraint value, momentum aggregates constraint information across rounds, and the STORM estimator additionally corrects for the movement from $w_{k-1}$ to $w_k$.

\noindent	\textbf{Contributions.}
Our results answer three questions.

\begin{enumerate}[leftmargin=*,label=\textbf{Q\arabic*.},labelsep=*,nosep]

\item 	\textbf{Does direct stochastic switching remain reliable with a fixed per-round constraint-evaluation budget?}
Not in general. With fixed batch constraint estimation, direct switching retains a non-vanishing constraint-estimation term. Under partial participation, client sampling introduces an additional source of uncertainty that persists even when selected clients evaluate their local constraints exactly.

\item 	\textbf{Can vanishing feasibility be achieved without increasing the constraint-evaluation budget?}
Yes. In the convex setting, momentum achieves $\widetilde{\mathcal O}(K^{-1/4})$ objective and feasibility guarantees, while STORM improves both to $\widetilde{\mathcal O}(K^{-1/3})$. Both use fixed client participation and fixed per-round constraint mini-batches, showing that vanishing feasibility does not require gigantic batch size or full participation.

\item \textbf{How broadly does this stochastic switching framework apply?}
Our high-probability guarantees allow for partial participation, stochastic subgradients, and adaptive switching. We further demonstrate that the framework is compatible with bi-directional compression and error feedback. Separately, we extend stochastic switching to $\rho$-weakly convex objectives with convex functional constraints, establishing feasibility and constrained proximal-stationarity guarantees under Slater's condition. In this regime, STORM improves the oracle complexity from $\widetilde{\mathcal O}(\epsilon^{-8})$ to $\widetilde{\mathcal O}(\epsilon^{-6})$ relative to existing stochastic switching guarantees~\citep{huang2023oracle}.
\end{enumerate}

We validate our framework on Neyman--Pearson classification, fairness-constrained learning on the Adult dataset, and constrained deep reinforcement learning. Across these tasks, constraint tracking substantially improves the objective-constraint trade-off over direct switching and yields competitive or improved performance relative to penalty, primal-dual, and full-information baselines.

\section{Problem Setup and Preliminaries}
\label{sec:prelim}
\noindent \textbf{Notation.} We denote the set of clients as $[n]:=\{1,\ldots,n\}$ and communication rounds as $[K]:=\{0,\ldots,K-1\}$. At round $k$, a subset $C_k\subseteq[n]$ of size $\mvalue$ participates in the constraint evaluation, while a subset $S_k\subseteq[n]$ of size $\mupdate$ participates in the model update. Let $\Pi_{\mathcal{X}}(\cdot)$ be the Euclidean projection onto $\mathcal{X}$, and let $\ind\{\cdot\}$ denote the indicator function. Moreover, $\|\cdot\|$ denotes the $\ell_2$-norm. Finally, we use $\partial h(w)$ to denote the subdifferential of $h$ at $w$, and $\nabla h(w)$ to denote a selected subgradient.

\noindent \textbf{Sampling model.}
The server generally does not observe $F$, $G$, or their subgradients exactly. In round $k$, it first samples $C_k$ uniformly without replacement where each $i\in C_k$ draws samples $\{\xi_{k,i,b}\}_{b=1}^{B_{\text{val}}}\overset{\text{i.i.d}}{\sim}\mathcal D_i$ and evaluates stochastic constraint values $g_i(\cdot,\xi_{k,i,b})$. After the switching decision, the server samples $S_k$ uniformly without replacement, independently of $C_k$ and the evaluation samples; each $i\in S_k$ draws samples $\{\zeta_{k,i,b}\}_{b=1}^{\Bgradient} \overset{\text{i.i.d}}{\sim}\mathcal{D}_i$ and evaluates stochastic subgradients. Conditional on the selected clients, all fresh samples are independent.

We next introduce the switching estimators and assumptions shared across the convex and weakly convex settings.

\noindent \textbf{Constraint Estimators.}
The server instantiates a switching decision is based on an estimate of $G(w_k)$ built using the values returned by $C_k$,
\begin{equation}
\label{eq:Ghat}
\widehat G_k
:=
\frac{1}{\mvalue\Bvalue}
\sum_{i\in C_k}\sum_{b=1}^{\Bvalue}
g_i(w_k,\xi_{k,i,b}).
\end{equation}
We consider three estimators. For \emph{direct switching}, we simply use
\begin{equation}
\label{eq:est_direct}
Z_k^{\mathrm{dir}}
:=
\widehat G_k.
\end{equation}
For \emph{momentum switching}, we define
\begin{equation}
\label{eq:est_momentum}
Z_0^{\mathrm{mom}}:=\widehat G_0,
\quad
Z_k^{\mathrm{mom}}
:=
(1-\alpha)Z_{k-1}^{\mathrm{mom}}
+\alpha\widehat G_k,
\end{equation}
where $\alpha\in(0,1]$ is the momentum weight.

Finally, for \emph{STORM switching}, we additionally correct for the movement from $w_{k-1}$ to $w_k$. For $k\geq1$, define
\begin{equation}
\label{eq:Ghat_minus}
\widehat G_k^{-}
:=
\frac{1}{\mvalue\Bvalue}
\sum_{i\in C_k}\sum_{b=1}^{\Bvalue}
g_i(w_{k-1},\xi_{k,i,b}),
\end{equation}
and set
\begin{equation}
\label{eq:est_storm}
Z_0^{\mathrm{storm}}:=\widehat G_0,
\quad
Z_k^{\mathrm{storm}}
:=
\widehat G_k
+(1-\alpha)
\left(
Z_{k-1}^{\mathrm{storm}}
-\widehat G_k^{-}
\right).
\end{equation}
Accordingly, throughout the paper we write
\begin{equation}
\label{eq:unified_estimator}
Z_k (\mathtt{mode})
=
\begin{cases}
Z_k^{\mathrm{dir}},
&\mathtt{mode}=\mathrm{direct},\\
Z_k^{\mathrm{mom}},
&\mathtt{mode}=\mathrm{momentum},\\
Z_k^{\mathrm{storm}},
&\mathtt{mode}=\mathrm{STORM},
\end{cases}
\end{equation}
and set $I_k:=\ind\{Z_k(\mathtt{mode})\leq\theta\}$ where $I_k=1$ indicates an objective round and $I_k=0$ indicates a constraint-correction round.

Direct switching discards past information, so its error does not decrease over rounds when $\mvalue$ and $\Bvalue$ are fixed. Momentum averages this noise across rounds but lags behind $G(w_k)$ as the iterate moves. STORM removes the lag, since $Z_k^{\mathrm{storm}} =(1-\alpha)\bigl(Z_{k-1}^{\mathrm{storm}}+\widehat G_k-\widehat G_k^-\bigr)+\alpha\widehat G_k$, it shifts the previous estimate by a same-sample estimate of $G(w_k)-G(w_{k-1})$ before averaging.

\noindent \textbf{Communication protocol.}
Each client in $C_k$ returns one empirical constraint value under direct and momentum switching, and two under STORM, corresponding to the same-sample evaluations entering $\widehat G_k$ and $\widehat G_k^-$. For the model update, each client in $S_k$ receives the switching bit $I_k$ and returns one $d$-dimensional update message. Sampled clients evaluate at the current iterate $w_k$; under STORM, clients in $C_k$ additionally evaluate at $w_{k-1}$. In the uncompressed variant, the server provides the required model state only to the sampled cohorts. In the compressed variant, the server broadcasts the compressed increment $d_k$ to all clients so that they maintain a synchronized copy of the current iterate, while computation and uploads remain restricted to $C_k$ and $S_k$. Under STORM, clients additionally retain the previous synchronized iterate for one round.
\begin{algorithm*}[t]
\caption{Federated Stochastic Switching with Bi-Directional Compression}
\label{alg:fed_switch}
\algrenewcommand{\algorithmiccomment}[1]{\hfill{\color{blue}\small$\triangleright$~#1}}
\begin{algorithmic}[1]
\State \textbf{Input:} $K,\eta,\theta,\Bvalue,\Bgradient,
\mvalue,\mupdate,\alpha$,
$\mathtt{mode}\in\{\mathrm{direct},\mathrm{momentum},\mathrm{STORM}\}$,
and $\mathtt{comp}\in\{\mathrm{off},\mathrm{on}\}$.
\State \textbf{Initialize:} $w_0\in\mathcal{X}$.
If $\mathtt{comp}=\mathrm{on}$, set
$y_0=w_0$ and $r_0^{(i)}=0$ for all $i\in[n]$. 
\For{$k=0,\ldots,K-1$}
    \State Sample $C_k\subseteq[n]$ uniformly, $|C_k|=\mvalue$; provide $w_k$ (and $w_{k-1}$ under STORM) to $C_k$, form $Z_k(\mathtt{mode})$ as in~\eqref{eq:unified_estimator}, and set $I_k\gets\ind\{Z_k\leq\theta\}$. \Comment{Constraint Query}
    \State Independently sample $S_k\subseteq[n]$ uniformly, $|S_k|=\mupdate$, and provide $(w_k,I_k)$ to $S_k$. 
    \ForAll{$i\in S_k$ \textbf{in parallel}} \Comment{Client Update}
        \State Compute
        $
        \widehat h_k^i = \frac{1}{\Bgradient} \sum_{b=1}^{\Bgradient} \left[ I_k\nabla f_i(w_k;\zeta_{k,i,b}) +(1-I_k)\nabla g_i(w_k;\zeta_{k,i,b}) \right].
        $
        \State Send $v_k^i$ to the server, where
        \[
        v_k^i =
        \begin{cases}
        \widehat h_k^i, & \mathtt{comp}=\mathrm{off}, \\
        \mathcal{Q}_i\bigl(r_k^{(i)}+\widehat h_k^i\bigr), & \mathtt{comp}=\mathrm{on},
        \end{cases}
        \qquad \text{and, if $\mathtt{comp}=\mathrm{on}$, set }
        r_{k+1}^{(i)} \gets r_k^{(i)}+\widehat h_k^i-v_k^i.
        \]
    \EndFor
    \State If $\mathtt{comp}=\mathrm{on}$, clients $i\notin S_k$ keep $r_{k+1}^{(i)}=r_k^{(i)}$.
    \State $v_k\gets
    \frac{1}{\mupdate}\sum_{i\in S_k}v_k^i$. \Comment{Server Side Update}
    \If{$\mathtt{comp}=\mathrm{off}$}
        \State $w_{k+1}\gets\Pi_\mathcal{X}(w_k-\eta v_k)$.
    \Else
        \State $y_{k+1}\gets y_k-\eta v_k$,
        \quad
        $d_k\gets
        \calQ_0\!\left(\Pi_\mathcal{X}(y_{k+1})-w_k\right)$.
        \State Broadcast $d_k$ to all clients; the server and clients set
        $w_{k+1}\gets\Pi_\mathcal{X}(w_k+d_k)$. 
    \EndIf
\EndFor
\end{algorithmic}
\end{algorithm*}

We now state the assumptions shared by the convex and weakly convex analyses~\citep{beck2017first,davis2018complexity,davis2019stochastic}.
\begin{assumption}
\label{as:bdd_domain}
We have that $\mathcal{X}\subseteq\mathbb{R}^d$ is convex and compact, meaning the set is bounded, thus $\operatorname{diam}(\mathcal{X}) := \sup_{w_1,w_2\in\mathcal{X}} \|w_1-w_2\| \leq D.$ Moreover, there exists an optimal feasible point $w^\star\in\mathcal{X}$.
\end{assumption}

\begin{assumption}[Function Properties]
\label{as:func_props}
For all $i\in[n]$,
\begin{enumerate}[leftmargin=*,label=(\roman*),nosep]
    \item \textbf{Regularity:} $f_i$ and $g_i$ are $L$-Lipschitz continuous on an open convex set containing $\mathcal{X}$.
    \item \textbf{Weakly Convexity:} $f_i$ is $\rho$-weakly convex on the same set for some $\rho\geq0$. That is, $f_i+\frac{\rho}{2}\|\cdot\|^2$ is convex there; equivalently, $f_i(y)\geq f_i(w)+\langle\nabla f_i(w),y-w\rangle-\frac{\rho}{2}\|y-w\|^2$ for all $w,y$ in this set. Note that $\rho=0$ imply standard convexity.
    \item \textbf{Convexity:} $g_i$ is convex on the same set.
\end{enumerate}
\end{assumption}
The error in estimating $G(w_k)$ via $\widehat G_k$ in~\eqref{eq:Ghat} has two sources: client sampling, i.e., the gap between the participant average $\frac{1}{\mvalue}\sum_{i\in C_k}g_i(w_k)$ and $G(w_k)$, and mini-batch noise in each local evaluation. The following assumption controls both.
\begin{assumption}[Constraint Noise Models]
\label{as:cons_noise}
We model the noise at each level of stochasticity as follows.
\begin{enumerate}[leftmargin=*,label=(\roman*),nosep]
    \item \textbf{Data noise:}
    For every client $i$, deterministic $w\in\mathcal{X}$, and
    $\xi\sim\mathcal D_i$,
    $\E[g_i(w,\xi)]=g_i(w)$,
    \begin{equation*}
        \E\exp\!\left(
    \frac{|g_i(w,\xi)-g_i(w)|^2}
    {\sigmavalue^2}
    \right)\leq2.
    \end{equation*}

    \item \textbf{Participation noise:}
    If $C$ is sampled uniformly among all subsets of $[n]$ of size $\mvalue$, then
    \begin{equation*}
        \E_C\left[
    \frac{1}{\mvalue}\sum_{i\in C}g_i(w)-G(w)
    \right]=0.
    \end{equation*}
    When $\mvalue<n$, we additionally assume
    \begin{equation*}
        \E_C\exp\!\left(
    \frac{
    \left|
    \frac{1}{\mvalue}\sum_{i\in C}g_i(w)-G(w)
    \right|^2}
    {\sigmaparticipation^2/\mvalue}
    \right)
    \leq2.
    \end{equation*}
    Under full participation, the discrepancy is identically zero, and we set $\sigmaparticipation=0$.
\end{enumerate}
\end{assumption}

Together, these give the total constraint-evaluation noise level~\citep{rivasplata2012subgaussian}
\begin{equation}
\label{eq:value_proxy}
V_{\mathrm{val}} := \frac{\sigmaparticipation^2}{\mvalue} + \frac{\sigmavalue^2}{\mvalue\Bvalue},
\end{equation}
which reduces to $\sigmavalue^2/(n\Bvalue)$ under full participation.

Assumption~\ref{as:cons_noise} is the sub-Gaussian counterpart of the bounded-variance assumptions commonly used in analyses of stochastic and federated optimization~\citep{Li2020On,karimireddy2020scaffold,li2020federated}; such light-tail conditions are standard in high-probability analyses~\citep{liu2023high}. Intuitively, part~(ii) imposes a statistical limit on client heterogeneity in the constraint, when the local constraints differ substantially, the average over a small cohort can be far from $G(w)$, and $\sigmaparticipation$ measures this effect.

\begin{assumption}[Stochastic Gradient Noise]
\label{as:sgd_noise}
There exists $\sigmagradient>0$ such that, for every client $i$, deterministic $w\in\mathcal{X}$, and $\zeta\sim\mathcal D_i$, the stochastic subgradients are unbiased, i.e., $\E[\nabla f_i(w;\zeta)] = \nabla f_i(w)$ and $\E[\nabla g_i(w;\zeta)] = \nabla g_i(w).$ Moreover, 
\begin{align*}
\E\exp\!\left(
\frac{\|\nabla f_i(w;\zeta)-\nabla f_i(w)\|^2}
{\sigmagradient^2}
\right)
&\leq2,\\
\E\exp\!\left(
\frac{\|\nabla g_i(w;\zeta)-\nabla g_i(w)\|^2}
{\sigmagradient^2}
\right)
&\leq2.
\end{align*}
\end{assumption}

The STORM estimator additionally requires regularity of the same-sample difference appearing in~\eqref{eq:est_storm}. This is the sub-Gaussian analogue of the sample-wise smoothness condition (a.k.a. mean squared smoothness) used by STORM~\citep{cutkosky2019momentum}.
\begin{assumption}[Same-Sample Regularity]
\label{as:same_sample}
For some $\sigma_\Delta>0$, every client $i$, deterministic $w,w'\in\mathcal{X}$, and one fresh sample $\xi\sim\mathcal D_i$ satisfy whenever $w\neq w'$
\[
\E\exp\!\left(
\frac{
\left|
[g_i(w,\xi)-g_i(w',\xi)]
-
[g_i(w)-g_i(w')]
\right|^2}
{\sigma_\Delta^2\|w-w'\|^2}
\right)
\leq2 .
\]

\end{assumption}

\section{Federated Stochastic Switching: Main Results}
\label{sec:main_res}

We now present the convergence guarantees for Algorithm~\ref{alg:fed_switch}. Unlike the previous formulation, the convex and weakly convex regimes require different primal geometries. We therefore discuss them separately. We first consider the convex setting, where both the fully projected and bi-directionally compressed variants admit objective-gap and feasibility guarantees. We then consider $\rho$-weakly convex objectives, where we restrict attention to the fully projected recursion and establish proximal-stationarity guarantees.

\subsection{Convex Objectives}
\label{subsec:convex}

We first consider $\rho=0$. For a reporting window $\mathcal{K}\subseteq[K]$, define
\begin{equation}
\label{eq:A_convex}
\mathcal A
:=
\{k\in\mathcal K:I_k=1\},
\qquad
\bar w
:=
\frac{1}{|\mathcal A|}
\sum_{k\in\mathcal A}w_k.
\end{equation}
For direct switching, $\mathcal K=[K]$. For momentum and STORM switching, we instead use the reporting window $\mathcal K=\{k_0,\ldots,K-1\}$ to control the initialization error of  the recursive constraint estimator. Our goal is to find an approximately feasible solution satisfying
\[
F(\bar w)-F(w^\star)\leq\epsilon,
\qquad
G(\bar w)\leq\epsilon.
\]

\noindent \textbf{Optional bi-directional compression.} 
All of our guarantees hold without compression. Nonetheless, when communication is limited, Algorithm~\ref{alg:fed_switch} allows for both uplink (client-to-server) and downlink (server-to-client) transmissions to be compressed \citep{tang2019doublesqueeze, philippenko2021preserved, gruntkowska2023ef21}. Our convex guarantees cover this variant. 

\begin{assumption}[Biased Bi-Directional Compression]
\label{as:compression_biased}
The compression operators $\calQ_0$ (server) and $\calQ_i$ (client $i$) are deterministic, measurable, and contractive:
\begin{align}
\|\calQ_i(v)-v\|^2
&\leq
(1-\qclient)\|v\|^2,
\qquad i\in[n],
\label{eq:q_client}\\
\|\calQ_0(v)-v\|^2
&\leq
(1-\qserver)\|v\|^2,
\label{eq:q_server}
\end{align}
for every $v\in\mathbb R^d$, with
$\qclient,\qserver\in(0,1]$.
No unbiasedness is assumed. For example, Top-$s$ sparsification satisfies~\eqref{eq:q_client} with $\qclient=s/d$.
\end{assumption}
To mitigate the bias introduced by compression \citep{seide20141,karimireddy2019error}, we employ error feedback on the uplink path, where each client maintains a residual vector $r_k^{(i)}$ to recover information lost during compression. On the downlink path, the server compresses the difference between the projected auxiliary state and the synchronized model~\citep{gruntkowska2023ef21}. This corrected-state formulation allows us to control biased compression in both directions.

\begin{theorem}[Convex objectives, informal]
\label{thm:convex_rates}
Let Assumptions~\ref{as:bdd_domain}--\ref{as:sgd_noise} hold with $\rho=0$, together with Assumption~\ref{as:same_sample} for STORM and Assumption~\ref{as:compression_biased} if $\mathtt{comp}=\mathrm{on}$. Fix $\delta\in(0,1)$ and keep $\mvalue,\mupdate,\Bvalue,\Bgradient$ fixed as $K$ grows. For notational convenience, set $\qclient=\qserver=1$ when $\mathtt{comp}=\mathrm{off}$. For either branch of Algorithm~\ref{alg:fed_switch}, there is an explicit threshold $\theta$ such that, with probability at least $1-\delta$, the objective-round set $\mathcal A$ is nonempty and $\bar w$ in~\eqref{eq:A_convex} satisfies:
\begin{enumerate}[leftmargin=*,label=(\roman*),nosep]

\item \textbf{Direct}
$\left(\eta \asymp \nicefrac{D\sqrt{\qclient\qserver}}{(L\sqrt K)}\right)$:
\[
\max\{F(\bar w)-F(w^\star),G(\bar w)\}
\le
\widetilde{\mathcal O}\!\left(
\tfrac{K^{-1/2}}{\sqrt{\qclient\qserver}}
+\sqrt{V_{\mathrm{val}}}
\right),
\]
where the $\sqrt{V_{\mathrm{val}}}$ term does not vanish as $K$ grows.

\item \textbf{Momentum}
$\left(\alpha=K^{\nicefrac{-1}{2}},\,
\eta \asymp \nicefrac{D\sqrt{\qclient\qserver}}{(LK^{3/4})}\right)$:
\[
\max\{F(\bar w)-F(w^\star),\,G(\bar w)\}
\le
\widetilde{\mathcal O}\!\left(
\nicefrac{K^{\nicefrac{-1}{4}}}{\sqrt{\qclient\qserver}}
\right).
\]

\item \textbf{STORM}
$\left(\alpha=K^{\nicefrac{-2}{3}},\,
\eta \asymp \nicefrac{D\sqrt{\qclient\qserver}}{(LK^{2/3})}\right)$:
\[
\max\{F(\bar w)-F(w^\star),\,G(\bar w)\}
\le
\widetilde{\mathcal O}\!\left(
\nicefrac{K^{-1/3}}{\sqrt{\qclient\qserver}}
\right).
\]
\end{enumerate}

For momentum and STORM, $\bar w$ averages over the post-burn-in window with
$k_0=\lceil\log K/(2\alpha)\rceil$. Here $\widetilde{\mathcal O}$ hides factors polylogarithmic in $K/\delta$ and constants. Full exposition is in Theorems~\ref{thm:direct-combined}-\ref{thm:storm-combined} in the Appendix~\ref{sec:proof_details}.
\end{theorem}


Momentum and STORM keep $\mvalue$ and $\Bvalue$ fixed as $K$ grows: momentum filters persistent constraint-estimation noise but incurs tracking lag, whereas STORM corrects for target movement and improves the rate from $K^{-1/4}$ to $K^{-1/3}$ (up to logarithms). Even under full participation, our results differ from Safe-EF~\citep{islamov2025safe}, which uses an accuracy-dependent $\widetilde{\mathcal O}(\epsilon^{-2})$ constraint-value batch and reports $\widetilde{\mathcal O}(\epsilon^{-4})$ sample complexity. STORM instead uses a fixed per-round budget and achieves $\widetilde{\mathcal O}(\epsilon^{-3})$ oracle complexity. Under partial participation, tracking additionally averages the client-sampling term $\sigmaparticipation^2/\mvalue$, which cannot be removed by increasing $\Bvalue$ alone. This trades additional rounds for a constant per-round evaluation budget, which is the relevant regime when clients cannot afford gigantic batch size for constraint computation per round. See Appendix~\ref{sec:theoretical_challenges} for a detailed discussion of the associated theoretical challenges.

\subsection{Weakly Convex Objectives}
\label{subsec:weakly_convex}
We next consider the more challenging case where $F$ is $\rho$-weakly convex with $\rho>0$. Since the objective is nonconvex, finding a globally optimal solution is generally intractable. Instead, our goal is to find an approximately stationary point satisfying the KKT conditions of~\eqref{eq:op_prb}. For nonsmooth weakly convex problems, directly measuring approximate KKT stationarity is difficult. We therefore follow the standard proximal-stationarity framework for weakly convex constrained optimization~\citep{boob2023stochastic,huang2023oracle}.

\noindent \textbf{Constrained proximal stationarity.}
For any $w\in\mathcal{X}$ and $\widehat\rho>\rho$, define
\begin{equation}
\label{eq:prox_map}
\widehat w(w) := \argmin_{y\in\mathcal{X}} \left\{ F(y) + \frac{\widehat\rho}{2} \|y-w\|^2 \,\, \text{s.t.}\,\, G(y)\leq0 \right\}.
\end{equation}
The proximal problem provides a strongly convex proxy for the original weakly convex problem, and 
$\|w-\widehat w(w)\|$ distance will serve as our measure of stationarity.

To obtain a well-behaved Lagrange multiplier for~\eqref{eq:prox_map}, we impose Slater's condition, as is standard in weakly convex constrained optimization~\citep{boob2023stochastic,huang2023oracle}.

\begin{assumption}[Slater Condition]
\label{as:slater}
There exist
$w_{\mathrm{feas}}\in\operatorname{relint}(\mathcal{X})$
and $\sigma_{\mathrm{s}}>0$ such that $G(w_{\mathrm{feas}})\leq-\sigma_{\mathrm{s}}<0$.
\end{assumption}
Under Assumptions~\ref{as:bdd_domain}, \ref{as:func_props}, and~\ref{as:slater}, the objective in~\eqref{eq:prox_map} is $(\widehat\rho-\rho)$-strongly convex over a nonempty compact convex feasible set. Hence $\widehat w(w)$ exists and is unique.

Let $\lambda(w)\geq0$ denote a corresponding Lagrange multiplier. The proximal point satisfies
\begin{equation}
\label{eq:prox_kkt}
\begin{aligned}
0 \in {}&
\partial F(\widehat w(w))
+\widehat\rho\bigl(\widehat w(w)-w\bigr) \\
&+\lambda(w)\partial G(\widehat w(w))
+\mathcal{N}_\mathcal{X}(\widehat w(w)),
\end{aligned}
\end{equation}
where $\mathcal{N}_\mathcal{X}(\cdot)$ denote the normal cone of $\mathcal{X}$. Moreover, $G(\widehat w(w))\leq0$ and $\lambda(w)G(\widehat w(w))=0$.
Slater's condition further yields a uniform bound on $\lambda(w)$.

Thus, a small value of $\|w-\widehat w(w)\|$ implies approximate KKT stationarity at the nearby feasible point $\widehat w(w)$. Accordingly, for the objective-round set $\mathcal A$ defined in~\eqref{eq:A_convex}, we measure stationarity through $\frac{1}{|\mathcal A|} \sum_{k\in\mathcal A} \|w_k-\widehat w(w_k)\|.$ We consider this quantity together with the feasibility of $\bar w=|\mathcal A|^{-1}\sum_{k\in\mathcal A}w_k$. We now state the informal Theorem statement below and defer the complete exposition and proof to Appendix~\ref{sec:proof_details}.
\begin{theorem}[Weakly convex objectives, informal]
\label{thm:weak_rates}
Let Assumptions~\ref{as:bdd_domain}-\ref{as:sgd_noise} and~\ref{as:slater} hold, together
with Assumption~\ref{as:same_sample} for STORM. Fix $\delta\in(0,1)$ and
$\widehat\rho>\rho$, and keep
$\mvalue,\mupdate,\Bvalue,\Bgradient$ fixed as $K$ grows. For the uncompressed branch
($\mathtt{comp}=\mathrm{off}$) of Algorithm~\ref{alg:fed_switch}, there is an explicit
threshold $\theta$ such that, with probability at least $1-\delta$, the objective-round set
$\mathcal A$ is nonempty and:
\begin{enumerate}[leftmargin=*,label=(\roman*),nosep]
\item \textbf{Direct} $\left(\eta \asymp D/(L\sqrt K)\right)$:
\[
\begin{aligned}
\tfrac1{|\mathcal A|}\textstyle\sum_{k\in\mathcal A} G(w_k)\le\widetilde{\mathcal O}\bigl(K^{-1/2}+\sqrt{V_{\mathrm{val}}}\bigr),\\
\tfrac1{|\mathcal A|}\textstyle\sum_{k\in\mathcal A}\|w_k-\widehat w(w_k)\|
\le\widetilde{\mathcal O}\bigl(K^{-1/4}+V_{\mathrm{val}}^{1/4}\bigr),
\end{aligned}
\]
where the $V_{\mathrm{val}}$ in~\eqref{eq:value_proxy} do not vanish as $K$ grows; even with exact local
evaluations, $V_{\mathrm{val}}\geq\sigmaparticipation^2/\mvalue$ under partial participation.
\item \textbf{Momentum} $\left(\alpha=K^{-1/2},\, \eta \asymp D/(LK^{3/4})\right)$:
\[
\begin{aligned}
\tfrac1{|\mathcal A|}\textstyle\sum_{k\in\mathcal A} G(w_k)\le\widetilde{\mathcal O}(K^{-1/4}),\\
\tfrac1{|\mathcal A|}\textstyle\sum_{k\in\mathcal A}\|w_k-\widehat w(w_k)\|
\le\widetilde{\mathcal O}(K^{-1/8}).
\end{aligned}
\]
\item \textbf{STORM} $\left(\alpha=K^{-2/3},\,\eta \asymp D/(LK^{2/3})\right)$:
\[
\begin{aligned}
\tfrac1{|\mathcal A|}\textstyle\sum_{k\in\mathcal A} G(w_k)\le\widetilde{\mathcal O}(K^{-1/3}),\\
\tfrac1{|\mathcal A|}\textstyle\sum_{k\in\mathcal A}\|w_k-\widehat w(w_k)\|
\le\widetilde{\mathcal O}(K^{-1/6}).
\end{aligned}
\]
\end{enumerate}
The burn-in $k_0$ and the meaning of $\widetilde{\mathcal O}$ are as in
Theorem~\ref{thm:convex_rates}; here the constants also depend on $1/(\widehat\rho-\rho)$ and
on $1/\sigma_{\mathrm s}$ from Assumption~\ref{as:slater}. Explicit bounds are given in
Theorems~\ref{thm:direct-combined}-\ref{thm:storm-combined} in Appendix~\ref{sec:proof_details}.
\end{theorem}

Since $G$ is convex, $G(\bar w) \leq |\mathcal{{A}}|^{-1} \sum_{k\in\mathcal A}G(w_k)$. Moreover, on the event of Theorem~\ref{thm:weak_rates}, let $\tau$ be sampled uniformly from $\mathcal A$ and set $p_\tau:=\widehat w(w_\tau)$. Then $p_\tau$ is feasible and $\E_\tau \left[\|w_\tau-p_\tau\|\right] = \frac{1}{|\mathcal A|} \sum_{k\in\mathcal A} \|w_k-\widehat w(w_k)\|$.  By~\eqref{eq:prox_kkt}, $\dist \left( 0,\, \partial F(p_\tau) +\lambda(w_\tau)\partial G(p_\tau) +\mathcal{N}_\mathcal{X}(p_\tau) \right) \leq \widehat\rho\,\|w_\tau-p_\tau\|$. Thus, the stationarity bounds correspond to approximate KKT stationarity at nearby feasible proximal points.

\noindent\textbf{Oracle-complexity comparison.}
For direct switching, the weakly convex stationarity requirement is more stringent. In \citet{huang2023oracle}, achieving an $\epsilon$-stationary solution with a stochastic constraint-value oracle requires $\widetilde{\mathcal O}(\epsilon^{-4})$ constraint-value samples per iteration over $\mathcal O(\epsilon^{-4})$ iterations, yielding $\widetilde{\mathcal O}(\epsilon^{-8})$ oracle complexity. In contrast, our STORM estimator uses a fixed per-round constraint-evaluation batch and achieves $\widetilde{\mathcal O}(\epsilon^{-6})$ complexity. Thus, even in the centralized setting, our method improves the stochastic constraint-value oracle complexity from $\widetilde{\mathcal O}(\epsilon^{-8})$ to $\widetilde{\mathcal O}(\epsilon^{-6})$. To the best of our knowledge, this is the first such improvement for a switching-gradient method with a weakly convex objective and convex functional constraint.

\section{Experiments}
\label{sec:exp}
In this section, we evaluate our framework across three different constrained optimization tasks. We begin with the \emph{Neyman-Pearson (NP) classification} setting. Next, we extend our numerical experiment to the non-convex setting of \emph{fair classification} using deep neural networks. Finally, we demonstrate the efficacy of our method in highly stochastic tasks through experiments in \emph{constrained Markov decision processes (CMDP)}. Further experimental details and CMDP experiment results are presented in Appendix~\ref{sec:exp_apdx}.

\paragraph{NP Classification.}
NP classification studies the constrained optimization problem in~\ref{eq:op_prb}, where the goal is to minimize the empirical loss on the majority class while bounding the minority class loss below a threshold. In the experiments, we set the switching threshold $\theta$ directly to the prescribed application-level constraint tolerance. For each client $i \in [n]$, we define the local objective and constraint as $f_i(w) := \tfrac{1}{|\mathcal{D}_j^{(0)}|} \sum_{x \in \mathcal{D}_i^{(0)}} \phi(w;(x,0)), \ g_iw) := \frac{1}{|\mathcal{D}_i^{(1)}|} \sum_{x \in \mathcal{D}_i^{(1)}} \phi(w;(x,1)),$ where $\mathcal{D}_i^{(0)}$ and $\mathcal{D}_i^{(1)}$ denote the local datasets for class-0 and class-1, respectively, and $\phi$ represents the binary logistic loss. We evaluate performance on the breast cancer dataset~\citep{breast_cancer_wisconsin_dataset}, reporting mean and variance bands across three random seeds.

\begin{figure}
    \centering
    \includegraphics[width=\linewidth]{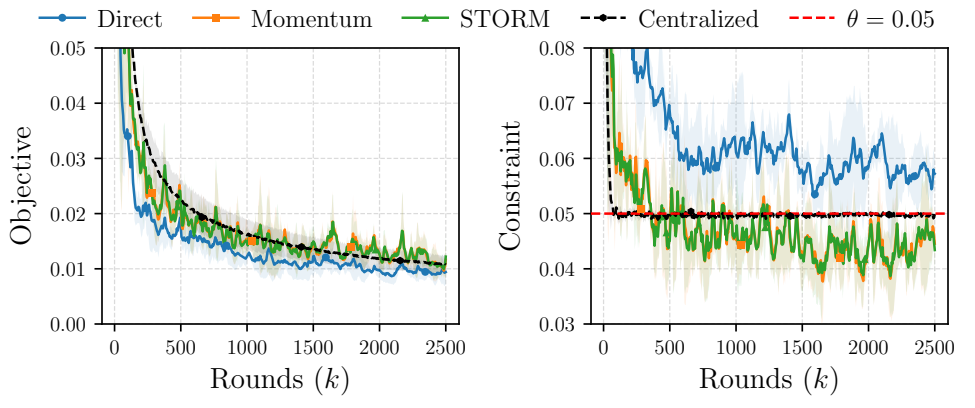}
    \vspace{-7mm}
    \caption{\textbf{NP classification.} Evolution of the objective $F(w_k)$ (left) and constraint $G(w_k)$ (right) over communication rounds. We compare direct, momentum, and STORM switching variants and the centralized setting. The red dashed line marks the constraint threshold.}
    \label{fig:NPC_main_rounds}
    \vspace{-4mm}
\end{figure}

Figure~\ref{fig:NPC_main_rounds} shows the result of our federated algorithms under partial participation $0.3$ ($\mvalue=6$ of $20$ clients) and $\Bvalue=1$, with a centralized reference with $\mvalue=20$ and full batch. Both the momentum and STORM variants successfully minimize the majority-class objective $F(w)$ while satisfying the constraint $G(w)$ at comparable rates to the centralized setting but with higher variance. Meanwhile, direct switching suffers from high estimation error for the switching decision, so it did not satisfy the constraint within the specified rounds.


\begin{figure*}[htbp]
    \centering
    \includegraphics[width=\linewidth]{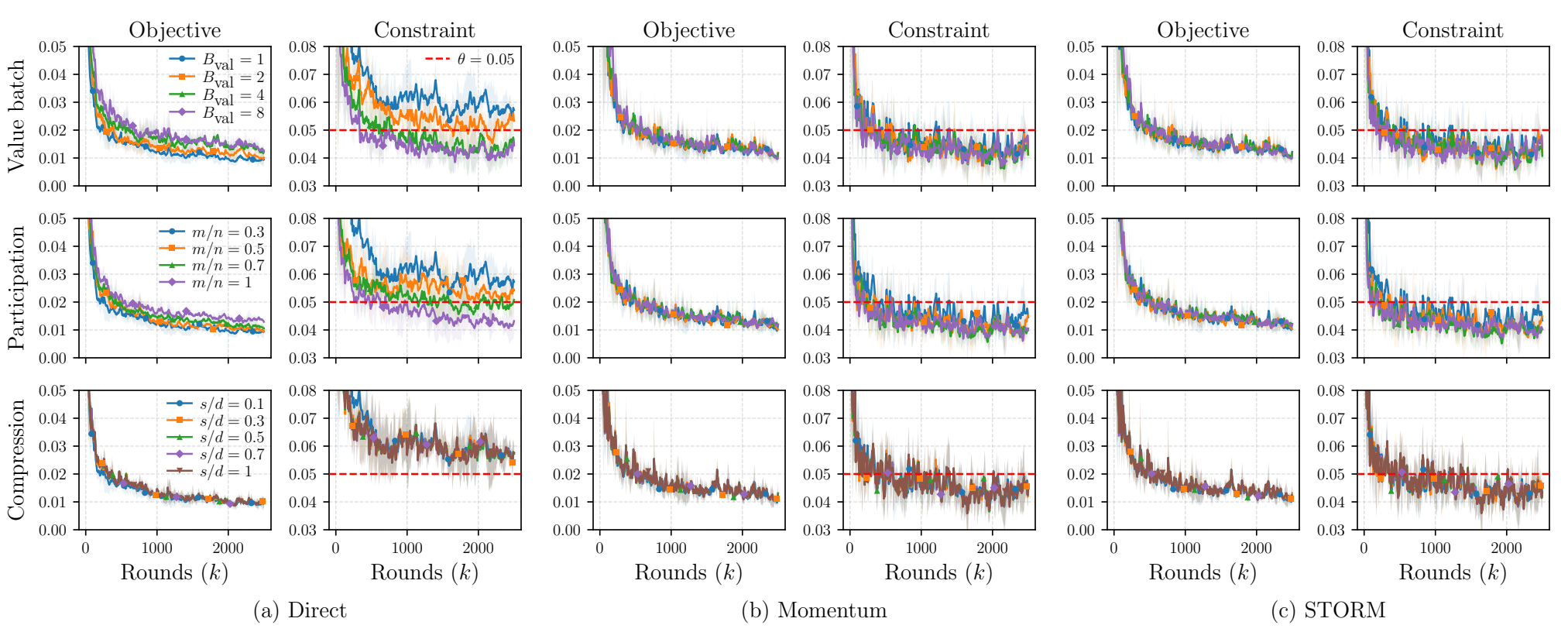}
    \vspace{-7mm}
    \caption{\textbf{NP classification.} Sensitivity of (a) direct, (b) momentum, and (c) STORM switching. Within each panel, the left and right columns show the objective $F(w_k)$ and constraint $G(w_k)$ over communication rounds, respectively. From top to bottom, the rows vary the constraint-evaluation batch size $\Bvalue$, client participation rate $m/n$, and compression rate $s/d$. The red dashed line marks the constraint threshold $\theta=0.05$.}
    \label{fig:NPC_sensitivity}
    \vspace{-5mm}
\end{figure*}
Figure~\ref{fig:NPC_sensitivity} reports the sensitivity of direct, momentum, and STORM switching to the constraint-evaluation batch size $\Bvalue$, client participation rate $\mvalue/n$, and compression rate $s/d$. Every other conditions that is not tested by each row follows the main experiment's settings (Table~\ref{tab:detail_NP}). Direct estimation suffer from small sample sizes represented by constraint value batch size (first row) and client participation (middle row). At the same time, these limitations are greatly reduced by momentum and STORM estimations. In all switching variants, compression rate did not have noticeable impact thanks to the effective error feedback mechanism.

\paragraph{Fair Classification.}
We formulate the fair classification task as minimizing the binary cross-entropy loss subject to a demographic parity constraint. For each client $i\in[n]$, the local objective is defined as $f_i(w) := \tfrac{1}{|\mathcal{D}_i|}\sum_{(x,y)\in\mathcal{D}_i}\ell_\text{BCE}(\pi(x;w),y))$.
Due to the absolute value operator in demographic parity, the constraint is defined only globally $G(w) := \left|\frac{1}{|\mathcal{D}_p|}\sum_{x\in\mathcal{D}_p} \pi(x;w) - \frac{1}{|\mathcal{D}_u|}\sum_{x\in\mathcal{D}_u} \pi(x;w)\right|,$
where subscripts $p$ and $u$ denote the protected and unprotected subgroups, respectively, instead of as the average of local disparities.
For this task, we explore the setting of stochastic data sampling, heterogeneous client data distribution, and deep neural network, making the problem highly non-convex, stochastic, skewed, and non-smooth. We conduct experiments on the Adult dataset \citep{kohavi1996adult}.

\begin{figure}
    \centering
    \includegraphics[width=\linewidth]{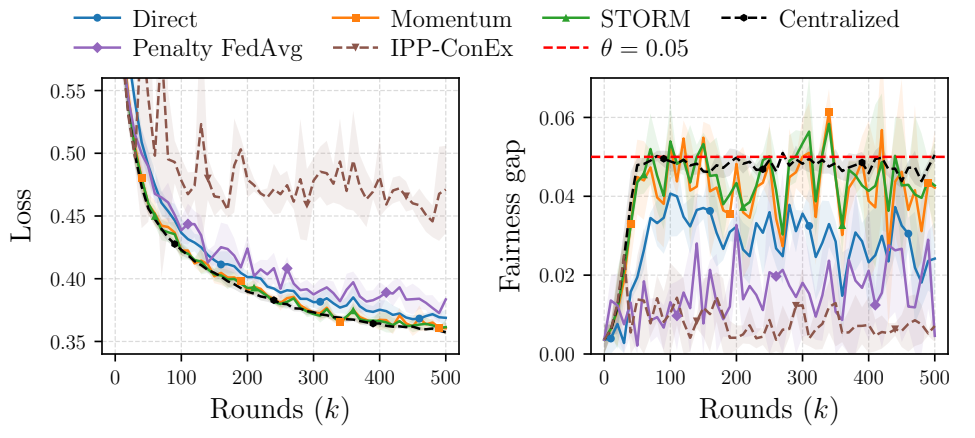}
    \vspace{-7mm}
    \caption{\textbf{Fair classification.} We compare our methods (blue/orange/green solid lines) against the centralized setting (black dashed line), federated penalty-based baseline (purple solid line), and IPP-ConEx (brown dashed line). Our momentum and STORM switching approach achieves the best trade-off: they converge to the lowest loss while satisfying the fairness constraint, matching the results of the full-information centralized setting.}
    \label{fig:DP_main_rounds}
    \vspace{-5mm}
\end{figure}

In the absence of an established federated baseline, we evaluate our framework against two benchmarks: penalty-based FedAvg, representing a naive federated extension of the standard penalty method, and IPP-ConEx (Inexact Proximal Point Constraint Extrapolation) \cite{boob2023stochastic}, an established centralized algorithm. As illustrated in Figure~\ref{fig:DP_main_rounds}, direct switching (blue solid line), penalty FedAvg (purple solid line), and IPP-ConEx (brown dashed line) are overly conservative. While the conservatism of penalty FedAvg and IPP-ConEx stems from the sub-optimality of their penalty parameters and dual variables, the limitations of direct switching are driven specifically by estimation errors in the constraint function. By tracking the evolving constraint, the momentum and STORM estimators (orange and green solid lines) successfully mitigate this statistical noise. This variance reduction allows both methods to outperform direct switching, recovering a loss convergence rate comparable to the full-information centralized setting (dashed black line) while stably satisfying the fairness threshold.

\section{Conclusion}
\label{sec:conclusion}
We study functionally constrained federated optimization and develop a primal-only stochastic switching framework that avoids costly projection and dual-variable tuning. Our analysis accommodates stochastic constraint evaluations, partial client participation, stochastic gradients, and bi-directional compression in the convex setting. We show that direct switching suffers from a non-vanishing constraint-estimation error under fixed per-round budgets, and introduce variance-reducing constraint estimators that achieve vanishing feasibility with rates $\widetilde{\mathcal O}(K^{-1/4})$ and $\widetilde{\mathcal O}(K^{-1/3})$, respectively. We further extend the framework to weakly convex objectives and establish feasibility and constrained proximal-stationarity guarantees. Experiments across constrained classification, fairness-constrained learning, and constrained reinforcement learning demonstrate the effectiveness of the proposed approach. Overall, our results show that accurate constraint tracking is key to reliable stochastic switching in federated settings, allowing uncertainty to be reduced across rounds without increasing the per-round evaluation budget. This provides a practical route to primal-only constrained learning under limited computation and communication.

A natural direction for future work is to extend stochastic switching to fully decentralized, serverless settings. Without a central server, clients must jointly track the global constraint and coordinate switching decisions through peer-to-peer communication. Developing decentralized momentum and STORM estimators while controlling both stochastic estimation and network consensus errors is an important open direction.
\subsection*{AI use statement}
In this work, we used generative AI tools to formulate mathematical claims, to provide ingredients for proving mathematical claims, to assist in the writing of proofs, to design and implement experiments, and to interpret their results. We have not used generative AI tools to help develop theoretical models or conceptual frameworks, propose or refine hypotheses, to design or provide feedback on research methodology, or to assist with translation. Generation of synthetic data sets or cleaning and reformatting of datasets do not apply to this work. Additionally, we used generative AI tools for drafting parts of the research paper, summarizing or analyzing existing literature, and editing the research paper to improve readability. We have reviewed all AI-assisted work. All AI-provided proof ingredients and arguments, as well as the experimental code, were thoroughly checked by the authors. All AI-drafted sections were completely checked and rewritten using human judgment. All existing literature pointed to by AI tools was read and reviewed by the authors, and claims attributed thereto were confirmed to be accurate. The authors made all final scientific and editorial decisions and take responsibility for the final content of this work, including text, claims, or artifacts produced with the aid of generative AI.

\bibliography{refs}
\newpage

\clearpage
\appendix
\thispagestyle{empty}

\onecolumn
\aistatstitle{Variance Reduction for Stochastic Switching in Functionally Constrained Federated Optimization: Supplementary Materials}

\section{Proofs}

\subsection{Proof Preliminaries}
\label{sec:proof_prelim}

We use the problem formulation, notation, constraint estimators,
Algorithm~\ref{alg:fed_switch}, and assumptions from the main paper.
Unless stated otherwise, Assumptions~\ref{as:bdd_domain},
\ref{as:func_props}, \ref{as:cons_noise}, and~\ref{as:sgd_noise}
are in force (and we call them \emph{standing assumptions}). The compressed results additionally use Assumption~\ref{as:compression_biased}; the STORM results additionally use Assumption~\ref{as:same_sample}; and the weakly convex proximal-stationarity results additionally use Assumption~\ref{as:slater}.

For the probabilistic arguments below, all subset draws are fresh across
rounds. In each round, the update cohort $S_k$ is sampled only after the
constraint-evaluation phase and switching decision are complete.
Conditional on the selected client identities, all fresh data draws are
mutually independent, except that STORM uses the same evaluation sample
at $w_k$ and $w_{k-1}$ as specified in~\eqref{eq:est_storm}.
We fix measurable selections of the subgradients and stochastic oracle
maps used below.

For the compressed branch, clients not selected in round $k$ retain
their residuals, i.e.,
\[
r_{k+1}^{(i)}=r_k^{(i)},\qquad i\notin S_k.
\]

For $1\leq m\leq n$, we define
\begin{equation}
\label{eq:chi}
\chi_{n,m}^2
:=
\begin{cases}
\displaystyle
\sum_{r=1}^{m}
\left(\frac{n-m}{m(n-r)}\right)^2,
& m<n,\\[2ex]
0,&m=n.
\end{cases}
\end{equation}

We use $V_{\mathrm{val}}$ from~\eqref{eq:value_proxy} and the constrained proximal map $\widehat w(\cdot)$ from~\eqref{eq:prox_map} wherever applicable.
\subsection{Filtration, measurability, and common estimates}

The stochastic analysis uses the following filtration explicitly. Let \(\mathcal F_k^{\mathrm{pre}}\) be the sigma-field generated by all randomness revealed through the completion of round \(k-1\). Thus the current synchronized model \(w_k\), and in the compressed branch the server state \(y_k\) and all client residuals \(r_k^{(i)}\), are \(\mathcal F_k^{\mathrm{pre}}\)-measurable; under momentum or STORM, \(Z_{k-1}\) is also \(\mathcal F_k^{\mathrm{pre}}\)-measurable. Let
\[
\mathcal F_k^{\mathrm{eval}}
:=
\mathcal F_k^{\mathrm{pre}}
\vee
\sigma\!\left(C_k,\{\xi_{k,i,b}:i\in C_k,\ b\in[\Bvalue]\}\right).
\]
The quantities \(\widehat G_k\), and for STORM \(\widehat G_k^-\), together with \(Z_k\) and \(I_k\), are \(\mathcal F_k^{\mathrm{eval}}\)-measurable. After the switch, \(S_k\) is sampled independently of \(\mathcal F_k^{\mathrm{eval}}\); write
\[
\mathcal F_k^{\mathrm{sel}}
:=
\mathcal F_k^{\mathrm{eval}}\vee\sigma(S_k).
\]
Conditional on \(\mathcal F_k^{\mathrm{sel}}\), the gradient mini-batches are fresh. When a proof exposes samples or selected clients sequentially, the filtration is refined by adjoining them one at a time. At the end of the round, all newly revealed randomness is included in \(\mathcal F_{k+1}^{\mathrm{pre}}\).

\begin{lemma}[Adaptedness and basic deterministic bounds]
\label{lem:adapted-basic}
Under the standing assumptions, and additionally
Assumption~\ref{as:compression_biased} when
$\mathtt{comp}=\mathrm{on}$, the iterates generated by
Algorithm~\ref{alg:fed_switch} are adapted to the filtration introduced above.
In particular, $w_k$ is $\mathcal F_k^{\mathrm{pre}}$-measurable, while
$\widehat G_k$, $Z_k$, and $I_k$ are
$\mathcal F_k^{\mathrm{eval}}$-measurable. For STORM switching,
$\widehat G_k^{-}$ is also $\mathcal F_k^{\mathrm{eval}}$-measurable.
Moreover, for every $i\in[n]$ and $w\in\mathcal{X}$,
\[
\|\nabla f_i(w)\|_2\le L,
\qquad
\|\nabla g_i(w)\|_2\le L.
\]
\end{lemma}

\begin{proof}
We first verify the adaptedness statements. By definition, $\mathcal F_k^{\mathrm{pre}}$ contains all randomness revealed through the end of round $k-1$. Since $w_k$ is obtained from the states and random
variables generated during rounds $0,\ldots,k-1$, it is $\mathcal F_k^{\mathrm{pre}}$-measurable. In the compressed branch, the same argument applies to $y_k$ and the residuals $r_k^{(i)}$. Euclidean projection is continuous, the compressors and oracle maps are measurable, and finite sums and compositions preserve measurability.

At round $k$, the clients used for constraint evaluation $C_k$ and the corresponding fresh
constraint samples are included in $\mathcal F_k^{\mathrm{eval}}$. Hence, the empirical constraint value 
$$\widehat G_k = \frac{1}{\mvalue\Bvalue} \sum_{i\in C_k}\sum_{b=1}^{\Bvalue} g_i(w_k,\xi_{k,i,b})$$
is determined after the evaluation phase and is therefore $\mathcal F_k^{\mathrm{eval}}$-measurable. Under STORM switching, $w_{k-1}$ has already been determined before round $k$, and the same round-$k$ evaluation samples are used to form
$$
\widehat G_k^{-} = \frac{1}{\mvalue\Bvalue}\sum_{i\in C_k}\sum_{b=1}^{\Bvalue} g_i(w_{k-1},\xi_{k,i,b}), $$
so $\widehat G_k^{-}$ is also $\mathcal F_k^{\mathrm{eval}}$-measurable.

For direct switching, $Z_k=\widehat G_k$. For momentum switching, $Z_k$ is obtained from $\widehat G_k$ and the previously determined $Z_{k-1}$. For STORM switching, it is obtained from $\widehat G_k$, $\widehat G_k^{-}$, and $Z_{k-1}$. Thus in all three cases $Z_k$ is $\mathcal F_k^{\mathrm{eval}}$-measurable, and therefore so is $I_k=\ind\{Z_k\le\theta\}.$
After the switching decision, $S_k$ and the stochastic gradient mini-batches are revealed and the next state is formed according to Algorithm~\ref{alg:fed_switch}. Consequently, all round-$k$ updates are included in $\mathcal F_{k+1}^{\mathrm{pre}}$, completing the induction.

We next establish the deterministic subgradient bounds. The argument for the convex constraint follows the standard Lipschitz-continuity/bounded-subgradient argument; see, e.g., Theorem~3.61 of \citet{beck2017first}. We adapt the same directional argument to the $\rho$-weakly convex objective. 

Fix $i\in[n]$ and $w\in\mathcal{X}$. Let $v=\nabla f_i(w).$
If $v=0$, the claim is immediate. Suppose $v\neq0$, and we define $d:=\frac{v}{\|v\|_2}.$

By Assumption~\ref{as:func_props}, $f_i$ is $L$-Lipschitz and
$\rho$-weakly convex on an open set containing $\mathcal{X}$. Hence, because
this set is open, $w+td$ belongs to it for all sufficiently small
$t>0$. Applying the weak-convexity inequality in Assumption~\ref{as:func_props} at $w$ with $y=w+td$ gives
\[
f_i(w+td)
\ge
f_i(w)
+t\langle v,d\rangle
-\frac{\rho}{2}t^2
=
f_i(w)+t\|v\|_2-\frac{\rho}{2}t^2.
\]
On the other hand, $L$-Lipschitz continuity gives
\[
f_i(w+td)-f_i(w)
\le
L\|td\|_2
=
Lt.
\]
Combining the two inequalities above yields
\[
t\|v\|_2-\frac{\rho}{2}t^2
\le Lt,
\]
and therefore
\[
\|v\|_2
\le
L+\frac{\rho}{2}t.
\]
Letting $t\downarrow0$ proves $\|\nabla f_i(w)\|_2\le L.$

The argument for $g_i$ is the convex special case. If
$v=\nabla g_i(w)\neq0$, set $d=v/\|v\|_2$. Convexity gives
\[
g_i(w+td)
\ge
g_i(w)+t\|v\|_2,
\]
whereas $L$-Lipschitz continuity gives
\[
g_i(w+td)-g_i(w)\le Lt.
\]
Thus $\|v\|_2\le L$. This proves the result.
\end{proof}

\begin{lemma}[Conditional exponential-square bound]
\label{lem:square-to-mgf}
Let $\mathcal A$ be a sigma-field and let $X$ be a real random variable
such that
\[
\E[X\mid\mathcal A]=0,
\qquad
\E[e^{X^2}\mid\mathcal A]\le2.
\]
Then, for every real $\mathcal A$-measurable random variable $T$,
\begin{equation}
\E[e^{TX}\mid\mathcal A]
\le
e^{T^2/2}.
\label{eq:proof-square-to-mgf}
\end{equation}
Consequently, if $Z\in\mathbb R^d$ satisfies
\[
\E[Z\mid\mathcal A]=0,
\qquad
\E[e^{\|Z\|_2^2}\mid\mathcal A]\le2,
\]
then for every $\mathcal A$-measurable $a\in\mathbb R^d$ and
$\lambda\in\mathbb R$,
\begin{equation}
\E[e^{\lambda\langle a,Z\rangle}\mid\mathcal A]
\le
\exp\!\left(
\frac{\lambda^2\|a\|_2^2}{2}
\right).
\label{eq:proof-directional-mgf}
\end{equation}
\end{lemma}

\begin{proof}
The scalar argument follows the proof of Lemma~B.7 in
\citet{pmlr-v337-luo26a}. In particular, for any $u,x\in\mathbb R$,
\begin{equation}
e^{ux}-ux-1
\le
\bigl(e^{u^2/2}-1\bigr)
\bigl(e^{x^2}-1\bigr).
\label{eq:elementary-square-mgf}
\end{equation}
Since $T$ is $\mathcal A$-measurable, applying
\eqref{eq:elementary-square-mgf} with $u=T$ and $x=X$, and then taking
conditional expectation given $\mathcal A$, gives
\begin{align*}
\E[e^{TX}\mid\mathcal A]
&\le
1
+
T\E[X\mid\mathcal A]
+
\bigl(e^{T^2/2}-1\bigr)
\E[e^{X^2}-1\mid\mathcal A] \\
&\le
e^{T^2/2},
\end{align*}
where we used
$\E[X\mid\mathcal A]=0$ and
$\E[e^{X^2}\mid\mathcal A]\le2$.
This proves \eqref{eq:proof-square-to-mgf}.

For the vector statement, on $\{a\neq0\}$ we define $ X_a := \left\langle \frac{a}{\|a\|_2},Z \right\rangle.$
Because $a$ is $\mathcal A$-measurable,
\[
\E[X_a\mid\mathcal A]=0,
\qquad
X_a^2\le\|Z\|_2^2,
\]
and hence $\E[e^{X_a^2}\mid\mathcal A]\le2.$ Applying \eqref{eq:proof-square-to-mgf} with
$T=\lambda\|a\|_2$ yields
\[
\E[e^{\lambda\langle a,Z\rangle}\mid\mathcal A]
\le
\exp\!\left(
\frac{\lambda^2\|a\|_2^2}{2}
\right).
\]
The claim is immediate on $\{a=0\}$.
\end{proof}
\paragraph{Conditional freezing.}
We repeatedly use the following standard freezing argument; see Lemma~B.9 of \citet{pmlr-v337-luo26a}. If $W$ is $\mathcal A$-measurable, $U$ is independent of $\mathcal A$, and
$\varphi(W,U)$ is integrable, then
\begin{equation}
\E[\varphi(W,U)\mid\mathcal A] = \left.
\E[\varphi(w,U)] \right|_{w=W}
\qquad\text{a.s.}
\label{eq:proof-freezing}
\end{equation}
Therefore, whenever the current iterate is measurable with respect to the
conditioning sigma-field and the subsequent oracle sample is fresh, an
oracle assumption stated for deterministic $w$ may be applied
conditionally at the realized iterate.

\begin{lemma}[Conditional concentration under adaptive iterates]
\label{lem:common-concentration}
Under the standing assumptions and the filtration above, the conditional constraint-value MGF bound~\eqref{eq:common-value-mgf} and the stochastic-gradient exponential-square bound~\eqref{eq:common-gradient-exp} hold almost surely.

Moreover, fix $\delta'\in(0,1)$ and any consecutive interval of $J$ rounds. Then each of the windowed stochastic-gradient bound~\eqref{eq:common-window-gradient-square}, the squared-server-direction bound~\eqref{eq:common-server-square}, the directional-gradient error bound~\eqref{eq:common-directional}, and the update-participation bound~\eqref{eq:common-participation} holds with probability at least $1-\delta'$. The uniform stochastic-gradient bound~\eqref{eq:common-uniform-gradient} holds simultaneously over all selected clients and all $K$ rounds with probability at least $1-\delta'$.

For the directional-gradient bound~\eqref{eq:common-directional}, a random comparator in round $k$ may be any $\mathcal F_k^{\mathrm{sel}}$-measurable point of $\mathcal{X}$ chosen before the gradient samples are drawn. For the update-participation bound~\eqref{eq:common-participation}, the comparator must be $\mathcal F_k^{\mathrm{eval}}$-measurable, because it must be fixed before $S_k$ is sampled.
\end{lemma}

\begin{proof}
\noindent\textbf{Constraint-value observations.}\quad
Condition on $\mathcal F_k^{\mathrm{pre}}$. Then $w_k$ is fixed, while $C_k$ and the round-$k$ constraint-evaluation samples are fresh and independent of $\mathcal F_k^{\mathrm{pre}}$. Decompose
$$
\widehat G_k-G(w_k)
=
A_k+B_k,
$$
where $ A_k := \frac1{\mvalue}\sum_{i\in C_k}g_i(w_k)-G(w_k) $ is the error due to partial client participation, and $B_k := \frac1{\mvalue\Bvalue} \sum_{i\in C_k}\sum_{b=1}^{\Bvalue} [g_i(w_k,\xi_{k,i,b})-g_i(w_k)]$ is the stochastic constraint-value error.

Since $w_k$ is $\mathcal F_k^{\mathrm{pre}}$-measurable and $C_k$ is fresh, the freezing identity~\eqref{eq:proof-freezing} allows Assumption~\ref{as:cons_noise}(ii) to be applied at the realized iterate $w_k$. Thus, $\E[A_k\mid\mathcal F_k^{\mathrm{pre}}]=0.$ When $\mvalue<n$, Assumption~\ref{as:cons_noise}(ii) and Lemma~\ref{lem:square-to-mgf} further give
$$
\E\!\left[
e^{\lambda A_k}
\middle|
\mathcal F_k^{\mathrm{pre}}
\right]
\le
\exp\!\left(
\frac{\lambda^2\sigmaparticipation^2}{2\mvalue}
\right).
$$
When $\mvalue=n$, $A_k=0$ identically and the same bound holds with the participation contribution interpreted as zero.

Next we additionally condition on $C_k$. For every $i\in C_k$ and $b\in[\Bvalue]$, we define
$$
X_{k,i,b}
:=
g_i(w_k,\xi_{k,i,b})-g_i(w_k).
$$
The samples are conditionally independent given
$\mathcal F_k^{\mathrm{pre}}\vee\sigma(C_k)$. By the freezing identity~\eqref{eq:proof-freezing} and Assumption~\ref{as:cons_noise}(i),
$$
\E[
X_{k,i,b}
\mid
\mathcal F_k^{\mathrm{pre}}\vee\sigma(C_k)
]
= 0
\quad
\text{and}
\quad
\E\!\left[
\exp\!\left(
\frac{X_{k,i,b}^2}{\sigmavalue^2}
\right)
\middle|
\mathcal F_k^{\mathrm{pre}}\vee\sigma(C_k)
\right]
\le2.
$$
Hence Lemma~\ref{lem:square-to-mgf} gives, for every $\lambda\in\mathbb R$,
$$
\E\!\left[
\exp\!\left(
\frac{\lambda}{\mvalue\Bvalue}X_{k,i,b}
\right)
\middle|
\mathcal F_k^{\mathrm{pre}}\vee\sigma(C_k)
\right]
\le
\exp\!\left(
\frac{\lambda^2\sigmavalue^2}
{2\mvalue^2\Bvalue^2}
\right).
$$
Multiplying these conditional MGF bounds over the $\mvalue\Bvalue$ independent samples yields
$ \E\!\left[ e^{\lambda B_k} \middle| \mathcal F_k^{\mathrm{pre}}\vee\sigma(C_k) \right] \le \exp\!\left( \frac{\lambda^2\sigmavalue^2} {2\mvalue\Bvalue} \right). $

Since $A_k$ is measurable with respect to
$\mathcal F_k^{\mathrm{pre}}\vee\sigma(C_k)$, the tower property gives
\begin{align*}
\E\!\left[
e^{\lambda(A_k+B_k)}
\middle|
\mathcal F_k^{\mathrm{pre}}
\right] &=
\E\!\left[
e^{\lambda A_k}
\E\!\left[
e^{\lambda B_k}
\middle|
\mathcal F_k^{\mathrm{pre}}\vee\sigma(C_k)
\right]
\middle|
\mathcal F_k^{\mathrm{pre}}
\right]\\
&\quad\le
\exp\!\left(
\frac{\lambda^2\sigmavalue^2}
{2\mvalue\Bvalue}
\right)
\E\!\left[
e^{\lambda A_k}
\middle|
\mathcal F_k^{\mathrm{pre}}
\right]\\
&\quad\le
\exp\!\left(
\frac{\lambda^2}{2}
\left[
\frac{\sigmaparticipation^2}{\mvalue}
+
\frac{\sigmavalue^2}{\mvalue\Bvalue}
\right]
\right).
\end{align*}
Therefore, recalling the definition of $V_{\mathrm{val}}$,
\begin{equation}
\E\!\left[
\exp\!\left(
\lambda[\widehat G_k-G(w_k)]
\right)
\middle|
\mathcal F_k^{\mathrm{pre}}
\right]
\le
\exp\!\left(
\frac{\lambda^2V_{\mathrm{val}}}{2}
\right).
\label{eq:common-value-mgf}
\end{equation}

The standard Chernoff argument applied to both signs of
$\widehat G_k-G(w_k)$ consequently gives
$$
\Prob\!\left(
|\widehat G_k-G(w_k)|\ge t
\middle|
\mathcal F_k^{\mathrm{pre}}
\right)
\le
2e^{-t^2/(2V_{\mathrm{val}})}.
$$
Taking $t=\sqrt{2V_{\mathrm{val}}\log\frac{2K}{\delta'}}$ and applying a union bound over $k=0,\ldots,K-1$ gives, with probability at least $1-\delta'$,
\begin{equation}
\left|
\widehat G_k-G(w_k)
\right|
\le
\sqrt{
2V_{\mathrm{val}}
\log\frac{2K}{\delta'}
}
\qquad
\text{for all }k=0,\ldots,K-1.
\label{eq:common-uniform-value}
\end{equation}

\noindent\textbf{Gradient mini-batches.}\quad
For a selected client $i\in S_k$, we define the mini-batch error
$
\Delta_k^i := \widehat h_k^i - I_k\nabla f_i(w_k) - (1-I_k)\nabla g_i(w_k).
$
We first obtain a conditional exponential-square bound for $\Delta_k^i$.

For $b\in[\Bgradient]$, let us define
$$
q_b
:=
\frac1{\sigmagradient}
\Big(
I_k[
\nabla f_i(w_k;\zeta_{k,i,b})
-
\nabla f_i(w_k)
]
+
(1-I_k)[
\nabla g_i(w_k;\zeta_{k,i,b})
-
\nabla g_i(w_k)
]
\Big).
$$
For this fixed round and client, let
$$
\mathcal H_b
:=
\mathcal F_k^{\mathrm{sel}}
\vee
\sigma(
\zeta_{k,i,1},
\ldots,
\zeta_{k,i,b}
),
\qquad
\mathcal H_0
:=
\mathcal F_k^{\mathrm{sel}}.
$$
Conditional on $\mathcal H_{b-1}$, the iterate $w_k$, switch $I_k$, and client identity $i$ are fixed, while $\zeta_{k,i,b}$ is fresh. Hence the freezing identity~\eqref{eq:proof-freezing} and Assumption~\ref{as:sgd_noise} give
$$
\E[q_b\mid\mathcal H_{b-1}]=0
\quad
\text{and}
\quad
\E[
e^{\|q_b\|_2^2}
\mid
\mathcal H_{b-1}
]
\le2.
$$

Applying Lemma~B.8 of \citet{pmlr-v337-luo26a} to the sequence
$q_1,\ldots,q_{\Bgradient}$ and the filtration
$\mathcal H_0,\ldots,\mathcal H_{\Bgradient}$ yields
$$
\E\!\left[
\exp\!\left(
\frac1{4\Bgradient}
\left\|
\sum_{b=1}^{\Bgradient}q_b
\right\|_2^2
\right)
\middle|
\mathcal F_k^{\mathrm{sel}}
\right]
\le2.
$$
Since
$
\Delta_k^i
=
\frac{\sigmagradient}{\Bgradient}
\sum_{b=1}^{\Bgradient}q_b,
$
we obtain
\begin{equation}
\E\!\left[
\exp\!\left(
\frac{\Bgradient}{4\sigmagradient^2}
\|\Delta_k^i\|_2^2
\right)
\middle|
\mathcal F_k^{\mathrm{sel}}
\right]
\le2.
\label{eq:common-gradient-exp}
\end{equation}
The same sequential conditioning, together with $\E[q_b\mid\mathcal H_{b-1}]=0$, also gives $ \E[ \Delta_k^i \mid \mathcal F_k^{\mathrm{sel}} ] = 0. $

We next convert \eqref{eq:common-gradient-exp} into a bound that holds simultaneously over a reporting interval. Consider any consecutive interval of $J$ rounds and expose the selected-client gradient mini-batches sequentially, in any fixed order within each round. Since each newly exposed mini-batch is fresh conditional on all previously revealed information, \eqref{eq:common-gradient-exp} continues to imply
$$
\E\!\left[
\exp\!\left(
\frac{\Bgradient}{4\sigmagradient^2}
\|\Delta_k^i\|_2^2
-\log2
\right)
\middle|
\text{previously revealed information}
\right]
\le1.
$$
Multiplying these conditional factors over the $J\mupdate$ selected-client mini-batches and iterating conditional expectations gives
$$
\E\!\left[
\exp\!\left(
\frac{\Bgradient}{4\sigmagradient^2}
\sum_{k\text{ in the J-length interval}}
\sum_{i\in S_k}
\|\Delta_k^i\|_2^2
-
J\mupdate\log2
\right)
\right]
\le1.
$$
By Markov's inequality, with probability at least $1-\delta'$,
$$
\frac{\Bgradient}{4\sigmagradient^2}
\sum_{k\text{ in the J-length interval}}
\sum_{i\in S_k}
\|\Delta_k^i\|_2^2
\le
J\mupdate\log2
+
\log\frac1\delta'.
$$
Dividing by $\mupdate$ yields
\begin{align}
&\sum_{k\text{ in the J-length interval}}
\frac1{\mupdate}
\sum_{i\in S_k}
\left\|
\widehat h_k^i
-
I_k\nabla f_i(w_k)
-
(1-I_k)\nabla g_i(w_k)
\right\|_2^2
\le
\frac{4\sigmagradient^2}{\Bgradient}
\left(
J\log2
+
\frac1{\mupdate}\log\frac1\delta'
\right).
\label{eq:common-window-gradient-square}
\end{align}

For a uniform bound, \eqref{eq:common-gradient-exp} and Markov's inequality imply, for every $r>0$,
$$
\Prob\!\left(
\|\Delta_k^i\|_2
\ge
\frac{2\sigmagradient}{\sqrt{\Bgradient}}r
\middle|
\mathcal F_k^{\mathrm{sel}}
\right)
\le
2e^{-r^2}.
$$
Setting
$
r
=
\sqrt{
\log\frac{2K\mupdate}{\delta'}
}
$
and applying a union bound over the $K\mupdate$ selected-client mini-batches gives, with probability at least $1-\delta'$,
\begin{equation}
\left\|
\widehat h_k^i
-
I_k\nabla f_i(w_k)
-
(1-I_k)\nabla g_i(w_k)
\right\|_2
\le
\frac{2\sigmagradient}{\sqrt{\Bgradient}}
\sqrt{
\log\frac{2K\mupdate}{\delta'}
}
\label{eq:common-uniform-gradient}
\end{equation}
simultaneously for all selected clients in all rounds.

\noindent\textbf{Squared server directions.}\quad
Define $h_k^i := I_k\nabla f_i(w_k) + (1-I_k)\nabla g_i(w_k),$ so that $\widehat h_k^i=h_k^i+\Delta_k^i.$ By Lemma~\ref{lem:adapted-basic}, $\|h_k^i\|_2\le L. $
Hence
$$
\left\|
\frac1{\mupdate}\sum_{i\in S_k}h_k^i
\right\|_2
\le L.
$$
Therefore,
\begin{align*}
\left\|
\frac1{\mupdate}
\sum_{i\in S_k}\widehat h_k^i
\right\|_2^2
&\le
L^2
+
2L
\left\|
\frac1{\mupdate}
\sum_{i\in S_k}\Delta_k^i
\right\|_2 +
\left\|
\frac1{\mupdate}
\sum_{i\in S_k}\Delta_k^i
\right\|_2^2.
\end{align*}
By Jensen's inequality,
$$
\left\|
\frac1{\mupdate}
\sum_{i\in S_k}\Delta_k^i
\right\|_2^2
\le
\frac1{\mupdate}
\sum_{i\in S_k}
\|\Delta_k^i\|_2^2.
$$
Summing over the $J$ rounds and using Cauchy--Schwarz,
$$
\sum_{k\text{ in the J-length interval}}
\left\|
\frac1{\mupdate}
\sum_{i\in S_k}\Delta_k^i
\right\|_2
\le
\sqrt{
J
\sum_{k\text{ in the J-length interval}}
\frac1{\mupdate}
\sum_{i\in S_k}
\|\Delta_k^i\|_2^2
}.
$$
Combining these inequalities with
\eqref{eq:common-window-gradient-square} gives
\begin{align}
\sum_{k\text{ in the J-length interval}}
\left\|
\frac1{\mupdate}
\sum_{i\in S_k}\widehat h_k^i
\right\|_2^2
\le
JL^2
+
\frac{4L\sigmagradient}{\sqrt{\Bgradient}}
\sqrt{
J\left(
J\log2
+
\frac1{\mupdate}\log\frac1\delta'
\right)
}
+
\frac{4\sigmagradient^2}{\Bgradient}
\left(
J\log2
+
\frac1{\mupdate}\log\frac1\delta'
\right).
\label{eq:common-server-square}
\end{align}

\noindent\textbf{Directional gradient errors.}\quad
Let $p_k\in\mathcal{X}$ be any comparator that is
$\mathcal F_k^{\mathrm{sel}}$-measurable before the gradient samples are drawn, and we define $l_k:=w_k-p_k.$ Since both $w_k$ and $p_k$ belong to $\mathcal{X}$, $\|l_k\|_2\le D. $
Equation~\eqref{eq:common-gradient-exp} implies
$$
\E\!\left[
\exp\!\left(
\left\|
\frac{\sqrt{\Bgradient}}
{2\sigmagradient}
\Delta_k^i
\right\|_2^2
\right)
\middle|
\mathcal F_k^{\mathrm{sel}}
\right]
\le2.
$$
Moreover, $$ \E[ \Delta_k^i \mid \mathcal F_k^{\mathrm{sel}} ] = 0. $$
Since the gradient mini-batches of distinct selected clients are conditionally independent given $\mathcal F_k^{\mathrm{sel}}$, the same conditional bounds remain valid after previously exposed client mini-batches are adjoined to the conditioning sigma-field. Thus Lemma~\ref{lem:square-to-mgf}, applied in the direction $l_k$, gives
$$
\E\!\left[
\exp\!\left(
\frac{\lambda}{\mupdate}
\langle
\Delta_k^i,l_k
\rangle
\right)
\middle|
\text{previously revealed information}
\right]
\le
\exp\!\left(
\frac{\lambda^2}{2}
\frac{4\sigmagradient^2D^2}
{\Bgradient\mupdate^2}
\right).
$$
Exposing the $J\mupdate$ selected-client mini-batches sequentially and iterating these conditional MGF bounds shows that
$
\sum_{k\text{ in the J-length interval}}
\frac1{\mupdate}
\sum_{i\in S_k}
\langle
\Delta_k^i,l_k
\rangle
$
is conditionally sub-Gaussian with variance proxy at most
$
\frac{4D^2\sigmagradient^2J}
{\Bgradient\mupdate}.
$
A one-sided Chernoff bound therefore gives, with probability at least $1-\delta'$,
\begin{align}
&-
\sum_{k\text{ in the J-length interval}}
\frac1{\mupdate}
\sum_{i\in S_k}
\Big\langle
\widehat h_k^i
-
I_k\nabla f_i(w_k)
-
(1-I_k)\nabla g_i(w_k),
\,w_k-p_k
\Big\rangle
\le
\frac{2D\sigmagradient}{\sqrt{\Bgradient}}
\sqrt{
\frac{2J}{\mupdate}
\log\frac1\delta'
}.
\label{eq:common-directional}
\end{align}

\noindent\textbf{Update-participation error.}\quad
Fix a round $k$ and condition on $\mathcal F_k^{\mathrm{eval}}$.
Let $p_k\in\mathcal{X}$ be any
$\mathcal F_k^{\mathrm{eval}}$-measurable comparator and define
$$
a_i
:=
I_k[f_i(w_k)-f_i(p_k)]
+
(1-I_k)[g_i(w_k)-g_i(p_k)],
\qquad i\in[n].
$$
Conditional on $\mathcal F_k^{\mathrm{eval}}$, the quantities
$a_1,\ldots,a_n$ are fixed, while $S_k$ is a fresh uniformly sampled subset of $[n]$ of size $\mupdate$. Since both $f_i$ and $g_i$ are $L$-Lipschitz on $\mathcal{X}$ and
$\|w_k-p_k\|_2\le D$, $ |a_i|\le LD, $ so $ a_i\in[-LD,LD]. $

If $\mupdate=n$, then
$$
\frac1{\mupdate}\sum_{i\in S_k}a_i
=
\frac1n\sum_{i=1}^n a_i
$$
identically. Hence the participation discrepancy is zero, and the desired bound follows immediately from
$\chi_{n,n}=0$. We therefore assume below that $\mupdate<n$.

Represent $S_k$ as the first $\mupdate$ elements
$J_1,\ldots,J_{\mupdate}$ of a uniformly random ordering sampled without replacement. Define
$$
\bar a
:=
\frac1n\sum_{i=1}^n a_i
\quad \text{and} \quad
X
:=
\frac1{\mupdate}
\sum_{r=1}^{\mupdate}a_{J_r}
-
\bar a.
$$
For $r=0,\ldots,\mupdate$, let
$
\mathcal G_r
:=
\mathcal F_k^{\mathrm{eval}}
\vee
\sigma(J_1,\ldots,J_r)
$
and let us define the Doob martingale
$
M_r
:=
\E[X\mid\mathcal G_r].
$
Before revealing $J_r$, let $\mu_{r-1}$ denote the average of the
$n-r+1$ values that remain available. A direct calculation gives
$$
M_r-M_{r-1}
=
\frac{n-\mupdate}
{\mupdate(n-r)}
\bigl(
a_{J_r}-\mu_{r-1}
\bigr).
$$
Conditional on $\mathcal G_{r-1}$, the random variable
$a_{J_r}$ ranges over the remaining $a_i$'s. Since all of these values belong to $[-LD,LD]$, the conditional range length of the martingale increment is at most
$
2LD
\frac{n-\mupdate}
{\mupdate(n-r)}.
$
Hoeffding's lemma therefore gives
$$
\E\!\left[
e^{\lambda(M_r-M_{r-1})}
\middle|
\mathcal G_{r-1}
\right]
\le
\exp\!\left(
\frac{\lambda^2L^2D^2}{2}
\left(
\frac{n-\mupdate}
{\mupdate(n-r)}
\right)^2
\right).
$$
Multiplying these conditional MGF bounds over $r=1,\ldots,\mupdate$ gives $ \E\!\left[ e^{\lambda X} \middle| \mathcal F_k^{\mathrm{eval}} \right] \le \exp\!\left( \frac{\lambda^2L^2D^2\chi_{n,\mupdate}^2}{2} \right). $

Applying this conditional MGF bound successively over any interval of $J$ rounds yields variance proxy
$
J L^2D^2\chi_{n,\mupdate}^2
$
for the sum of the participation discrepancies. Hence, by a one-sided Chernoff bound, with probability at least $1-\delta'$,
\begin{align}
&-
\sum_{k\text{ in J-length interval}}
\Bigg[
\frac1{\mupdate}
\sum_{i\in S_k}
\Big(
I_k[f_i(w_k)-f_i(p_k)]
+
(1-I_k)[g_i(w_k)-g_i(p_k)]
\Big)
\nonumber\\
&\hspace{8em}
-
\frac1n
\sum_{i=1}^n
\Big(
I_k[f_i(w_k)-f_i(p_k)]
+
(1-I_k)[g_i(w_k)-g_i(p_k)]
\Big)
\Bigg]
\nonumber\\
&\qquad\le
LD\,\chi_{n,\mupdate}
\sqrt{
2J\log\frac1\delta'
}.
\label{eq:common-participation}
\end{align}

This proves all stated concentration estimates.
\end{proof}

\begin{lemma}[Bidirectional compression, corrected state, and dual-averaging geometry]
\label{lem:compressed-geometry}
Under the standing assumptions and Assumption~\ref{as:compression_biased}, suppose $\mathtt{comp}=\mathrm{on}$ and $\rho=0$, and fix $\delta'\in(0,1)$. On the uniform-gradient event~\eqref{eq:common-uniform-gradient}, the residual, compressed-message, downlink-tracking, and operational-tracking bounds \eqref{eq:proof-residual}--\eqref{eq:proof-operational-tracking} hold. The corrected-state identities \eqref{eq:proof-corrected-recursion}--\eqref{eq:proof-corrected-cumulative} and the dual-averaging bounds \eqref{eq:proof-da-full}--\eqref{eq:proof-da-late} hold deterministically.
\end{lemma}

\begin{proof}

Assume now \(\mathrm{comp}=\mathrm{on}\) and \(\rho=0\). Fix $\delta'\in(0,1)$ and work on the event in \eqref{eq:common-uniform-gradient}, which holds with probability at least $1-\delta'$. On this event, simultaneously for every $k\in\{0,\ldots,K-1\}$ and $i\in S_k$,
$$
\left\|
\widehat h_k^i
-
I_k\nabla f_i(w_k)
-
(1-I_k)\nabla g_i(w_k)
\right\|_2
\le
\frac{2\sigmagradient}{\sqrt{\Bgradient}}
\sqrt{\log\frac{2K\mupdate}{\delta'}}.
$$
By Lemma~\ref{lem:adapted-basic} and the triangle inequality, it follows
that
$$
\|\widehat h_k^i\|_2
\le
L+
\frac{2\sigmagradient}{\sqrt{\Bgradient}}
\sqrt{\log\frac{2K\mupdate}{\delta'}}.
$$
For a selected client, the residual update and compressor inequality give
\begin{align*}
\|r_{k+1}^{(i)}\|_2^2
&\overset{\eqref{eq:q_client}}{\le}
(1-\qclient)
\|r_k^{(i)}+\widehat h_k^i\|_2^2\\
&\leq
\left(1-\frac{\qclient}{2}\right)
\|r_k^{(i)}\|_2^2
+
\frac{(1-\qclient)(2-\qclient)}{\qclient}
\|\widehat h_k^i\|_2^2,
\end{align*}
where the second inequality follows from Young's inequality $\|a+b\|_2^2 \le (1+\gamma)\|a\|_2^2 + \left(1+\frac1\gamma\right)\|b\|_2^2$ with $\gamma=\frac{\qclient}{2(1-\qclient)}$, when $\qclient<1$. If $\qclient=1$, then \eqref{eq:q_client} implies $r_{k+1}^{(i)}=0$, so the displayed inequality holds trivially.

Since clients that are not selected retain their residuals, we can just iterate the preceding recursion over the rounds in which client $i$ is selected. Using $r_0^{(i)}=0$ and summing the resulting geometric series gives
$$
\begin{aligned}
\|r_k^{(i)}\|_2^2
&\le
\frac{(1-\qclient)(2-\qclient)}{\qclient}
\left(
L+\frac{2\sigmagradient}{\sqrt{\Bgradient}}
\sqrt{\log\frac{2K\mupdate}{\delta'}}
\right)^2
\sum_{j=0}^{\infty}
\left(1-\frac{\qclient}{2}\right)^j\\
&=
\frac{2(1-\qclient)(2-\qclient)}{\qclient^2}
\left(
L+\frac{2\sigmagradient}{\sqrt{\Bgradient}}
\sqrt{\log\frac{2K\mupdate}{\delta'}}
\right)^2\\
&\le
\frac{4(1-\qclient)}{\qclient^2}
\left(
L+\frac{2\sigmagradient}{\sqrt{\Bgradient}}
\sqrt{\log\frac{2K\mupdate}{\delta'}}
\right)^2.
\end{aligned}
$$
Therefore,
\begin{equation}
\|r_k^{(i)}\|_2
\le
\frac{2\sqrt{1-\qclient}}{\qclient}
\left(
L+\frac{2\sigmagradient}{\sqrt{\Bgradient}}
\sqrt{\log\frac{2K\mupdate}{\delta'}}
\right).
\label{eq:proof-residual}
\end{equation}
By the triangle inequality and \eqref{eq:q_client},
$\|\calQ_i(a)\|_2\le\|a\|_2+\|\calQ_i(a)-a\|_2\le(1+\sqrt{1-\qclient})\|a\|_2\le2\|a\|_2$.
Moreover, by \eqref{eq:proof-residual} and the uniform bound on
$\|\widehat h_k^i\|_2$,
$$
\|r_k^{(i)}+\widehat h_k^i\|_2^2
\le
2\|r_k^{(i)}\|_2^2+2\|\widehat h_k^i\|_2^2
\le
\frac{8}{\qclient^2}
\left(
L+\frac{2\sigmagradient}{\sqrt{\Bgradient}}
\sqrt{\log\frac{2K\mupdate}{\delta'}}
\right)^2.
$$
Hence, for every selected client,
$$
\|v_k^i\|_2
\le
\frac{4\sqrt2}{\qclient}
\left(
L+\frac{2\sigmagradient}{\sqrt{\Bgradient}}
\sqrt{\log\frac{2K\mupdate}{\delta'}}
\right).
$$
Therefore, by the triangle inequality,
\begin{equation}
\left\|
\frac1{\mupdate}\sum_{i\in S_k}v_k^i
\right\|_2
\le
\frac{4\sqrt2}{\qclient}
\left(
L+\frac{2\sigmagradient}{\sqrt{\Bgradient}}
\sqrt{\log\frac{2K\mupdate}{\delta'}}
\right).
\label{eq:proof-message}
\end{equation}

Let $z_k:=\Pi_\mathcal{X}(y_k).$ By the compressed server update in Algorithm~\ref{alg:fed_switch},
$$
d_k=\calQ_0(z_{k+1}-w_k),
\qquad
w_{k+1}=\Pi_\mathcal{X}(w_k+d_k).
$$
Since $z_{k+1}\in\mathcal{X}$, we have $\Pi_\mathcal{X}(z_{k+1})=z_{k+1}$. Hence, by projection nonexpansiveness and the downlink compressor bound~\eqref{eq:q_server},
\begin{align*}
\|w_{k+1}-z_{k+1}\|_2
&=
\|\Pi_\mathcal{X}(w_k+d_k)-\Pi_\mathcal{X}(z_{k+1})\|_2\\
&\le
\|w_k+d_k-z_{k+1}\|_2\\
&=
\left\|
\calQ_0(z_{k+1}-w_k)
-(z_{k+1}-w_k)
\right\|_2\\
&\le
\sqrt{1-\qserver}\,
\|z_{k+1}-w_k\|_2.
\end{align*}
Therefore,
$$
\|w_{k+1}-z_{k+1}\|_2^2
\le
(1-\qserver)\|z_{k+1}-w_k\|_2^2.
$$

Next, we decompose
$$
z_{k+1}-w_k
=
(z_{k+1}-z_k)+(z_k-w_k).
$$
For $\qserver<1$, Young's inequality with $\gamma=\frac{\qserver}{2(1-\qserver)}$ gives
\begin{align*}
(1-\qserver)\|z_{k+1}-w_k\|_2^2
&\le
(1-\qserver)(1+\gamma)
\|z_k-w_k\|_2^2 +
(1-\qserver)
\left(1+\frac1\gamma\right)
\|z_{k+1}-z_k\|_2^2\\
&=
\left(1-\frac{\qserver}{2}\right)
\|z_k-w_k\|_2^2 +
\frac{(1-\qserver)(2-\qserver)}{\qserver}
\|z_{k+1}-z_k\|_2^2.
\end{align*}
Moreover, projection nonexpansiveness and $ y_{k+1}-y_k = -\eta\frac1{\mupdate} \sum_{i\in S_k}v_k^i $
imply
\begin{align*}
\|z_{k+1}-z_k\|_2
&=
\|\Pi_\mathcal{X}(y_{k+1})-\Pi_\mathcal{X}(y_k)\|_2\\
&\le
\|y_{k+1}-y_k\|_2\\
&=
\eta
\left\|
\frac1{\mupdate}
\sum_{i\in S_k}v_k^i
\right\|_2.
\end{align*}
Combining the preceding inequalities yields
\begin{align*}
\|w_{k+1}-\Pi_\mathcal{X}(y_{k+1})\|_2^2
&\le
\left(1-\frac{\qserver}{2}\right)
\|w_k-\Pi_\mathcal{X}(y_k)\|_2^2 +
\frac{(1-\qserver)(2-\qserver)}{\qserver}
\eta^2
\left\|
\frac1{\mupdate}\sum_{i\in S_k}v_k^i
\right\|_2^2.
\end{align*}
When $\qserver=1$, \eqref{eq:q_server} implies
$\calQ_0(a)=a$ for every $a$, so
$w_{k+1}=\Pi_\mathcal{X}(y_{k+1})$ and the same inequality holds trivially.
Starting from $w_0=y_0$, we have $w_0=\Pi_\mathcal{X}(y_0)$ because $w_0\in\mathcal{X}$. Hence, iterating the preceding recursion, using \eqref{eq:proof-message}, and summing the geometric series gives
$$
\|w_k-\Pi_\mathcal{X}(y_k)\|_2^2
\le
\frac{2(1-\qserver)(2-\qserver)}{\qserver^2}\eta^2
\left(\frac{4\sqrt2}{\qclient}\right)^2
\left(
L+\frac{2\sigmagradient}{\sqrt{\Bgradient}}
\sqrt{\log\frac{2K\mupdate}{\delta'}}
\right)^2.
$$
Therefore,
\begin{equation}
\|w_k-\Pi_\mathcal{X}(y_k)\|_2
\le
\frac{8\eta\sqrt{(1-\qserver)(2-\qserver)}}{\qserver\qclient}
\left(
L+\frac{2\sigmagradient}{\sqrt{\Bgradient}}
\sqrt{\log\frac{2K\mupdate}{\delta'}}
\right).
\label{eq:proof-downlink}
\end{equation}

Only selected clients change their residuals, while $r_{k+1}^{(i)}=r_k^{(i)}$ for $i\notin S_k$. Therefore, $ \sum_{i=1}^n\bigl(r_{k+1}^{(i)}-r_k^{(i)}\bigr) = \sum_{i\in S_k}\bigl(\widehat h_k^i-v_k^i\bigr). $ Using $y_{k+1}=y_k-\eta\mupdate^{-1}\sum_{i\in S_k}v_k^i$, we obtain the exact corrected-state recursion
\begin{equation}
y_{k+1}-\frac{\eta}{\mupdate}\sum_{i=1}^n r_{k+1}^{(i)}
=
y_k-\frac{\eta}{\mupdate}\sum_{i=1}^n r_k^{(i)}
-\eta\frac1{\mupdate}\sum_{i\in S_k}\widehat h_k^i.
\label{eq:proof-corrected-recursion}
\end{equation}
Since $y_0=w_0$ and $r_0^{(i)}=0$ for every $i\in[n]$, iterating \eqref{eq:proof-corrected-recursion} gives
\begin{equation}
y_k-\frac{\eta}{\mupdate}\sum_{i=1}^n r_k^{(i)}
=
w_0-\eta\sum_{t=0}^{k-1}
\frac1{\mupdate}\sum_{i\in S_t}\widehat h_t^i.
\label{eq:proof-corrected-cumulative}
\end{equation}

Next, by the triangle inequality and projection nonexpansiveness,
$$
\left\|
w_k-
\Pi_\mathcal{X}\!\left(
y_k-\frac{\eta}{\mupdate}\sum_{i=1}^n r_k^{(i)}
\right)
\right\|_2
\le
\|w_k-\Pi_\mathcal{X}(y_k)\|_2
+
\frac{\eta}{\mupdate}
\left\|\sum_{i=1}^n r_k^{(i)}\right\|_2.
$$
Using the triangle inequality once more, $\left\|\sum_{i=1}^n r_k^{(i)}\right\|_2 \le \sum_{i=1}^n\|r_k^{(i)}\|_2.$
Applying \eqref{eq:proof-residual} to every client therefore gives
$$
\frac{\eta}{\mupdate}
\left\|\sum_{i=1}^n r_k^{(i)}\right\|_2
\le
\eta
\frac{2n\sqrt{1-\qclient}}{\mupdate\qclient}
\left(
L+\frac{2\sigmagradient}{\sqrt{\Bgradient}}
\sqrt{\log\frac{2K\mupdate}{\delta'}}
\right).
$$
Combining this with \eqref{eq:proof-downlink} yields
\begin{align}
&\left\|
w_k-
\Pi_\mathcal{X}\!\left(
y_k-\frac{\eta}{\mupdate}\sum_{i=1}^n r_k^{(i)}
\right)
\right\|_2
\le
\eta
\left(
L+\frac{2\sigmagradient}{\sqrt{\Bgradient}}
\sqrt{\log\frac{2K\mupdate}{\delta'}}
\right)
\left[
\frac{2n\sqrt{1-\qclient}}{\mupdate\qclient}
+
\frac{8\sqrt{(1-\qserver)(2-\qserver)}}{\qserver\qclient}
\right].
\label{eq:proof-operational-tracking}
\end{align}

Define $\bar{h}_k:=\tfrac{1}{\mupdate}\sum_{i\in S_k}\widehat h_k^i$ and $u_k:=-\eta\sum_{t=0}^{k-1}\bar{h}_t$, with $u_0=0$. By \eqref{eq:proof-corrected-cumulative},
$$
\Pi_\mathcal{X}\!\left(y_k-\frac{\eta}{\mupdate}\sum_{i=1}^n r_k^{(i)}\right)=\Pi_\mathcal{X}(w_0+u_k).
$$

Consider the function
$$
\Psi(u):=\sup_{w\in\mathcal{X}}\left\{\langle u,w\rangle-\frac12\|w-w_0\|_2^2\right\}.
$$
Completing the square shows that the unique maximizer is $\Pi_\mathcal{X}(w_0+u)$. Hence, by Danskin's theorem~\citep{bertsekas2009convex}, $\Psi$ is differentiable and $\nabla\Psi(u)=\Pi_\mathcal{X}(w_0+u)$. Since Euclidean projection is nonexpansive, $\nabla\Psi$ is one-Lipschitz, so $\Psi$ is one-smooth.

Since $u_{k+1}=u_k-\eta \bar{h}_k$, the smoothness inequality gives
$$
\Psi(u_{k+1})\le\Psi(u_k)-\eta\langle \bar{h}_k,\nabla\Psi(u_k)\rangle+\frac{\eta^2}{2}\|\bar{h}_k\|_2^2.
$$
Therefore,
$$
\langle \bar{h}_k,\nabla\Psi(u_k)\rangle\le\frac{\Psi(u_k)-\Psi(u_{k+1})}{\eta}+\frac{\eta}{2}\|\bar{h}_k\|_2^2.
$$

Summing over $k=0,\ldots,K-1$ and subtracting $\sum_{k=0}^{K-1}\langle \bar{h}_k,p\rangle$ for any fixed $p\in\mathcal{X}$ yields
$$
\sum_{k=0}^{K-1}\langle \bar{h}_k,\nabla\Psi(u_k)-p\rangle
\le
\frac{\Psi(u_0)-\Psi(u_K)+\langle u_K,p\rangle}{\eta}
+\frac{\eta}{2}\sum_{k=0}^{K-1}\|\bar{h}_k\|_2^2.
$$
Because $u_0=0$ and $w_0\in\mathcal{X}$, we have $\Psi(u_0)=0$. Moreover, evaluating the supremum defining $\Psi(u_K)$ at $p$ gives
$$
\Psi(u_K)\ge\langle u_K,p\rangle-\frac12\|p-w_0\|_2^2,
$$
and therefore $-\Psi(u_K)+\langle u_K,p\rangle\le\frac12\|p-w_0\|_2^2\le D^2/2$. Using $\nabla\Psi(u_k)=\Pi_\mathcal{X}(w_0+u_k)$ together with \eqref{eq:proof-corrected-cumulative}, we obtain
\begin{align}
&\sum_{k=0}^{K-1}
\left\langle
\frac1{\mupdate}\sum_{i\in S_k}\widehat h_k^i,
\Pi_\mathcal{X}\!\left(
y_k-\frac{\eta}{\mupdate}\sum_{i=1}^n r_k^{(i)}
\right)-p
\right\rangle
\le
\frac{D^2}{2\eta}
+\frac{\eta}{2}\sum_{k=0}^{K-1}
\left\|
\frac1{\mupdate}\sum_{i\in S_k}\widehat h_k^i
\right\|_2^2.
\label{eq:proof-da-full}
\end{align}

For the later interval $k=k_0,\ldots,K-1$, summing the same one-step inequality gives
$$
\sum_{k=k_0}^{K-1}\langle \bar{h}_k,\nabla\Psi(u_k)-p\rangle
\le
\frac{\Psi(u_{k_0})-\Psi(u_K)-\langle u_{k_0}-u_K,p\rangle}{\eta}
+\frac{\eta}{2}\sum_{k=k_0}^{K-1}\|\bar{h}_k\|_2^2.
$$
Equivalently,
$$
\sum_{k=k_0}^{K-1}\langle \bar{h}_k,\nabla\Psi(u_k)-p\rangle
\le
\frac{\Psi(u_{k_0})-\langle u_{k_0},p\rangle}{\eta}
+\frac{-\Psi(u_K)+\langle u_K,p\rangle}{\eta}
+\frac{\eta}{2}\sum_{k=k_0}^{K-1}\|\bar{h}_k\|_2^2.
$$
The second boundary term is bounded as before by $D^2/(2\eta)$. For the first one, using $u_{k_0}=-\eta\sum_{t=0}^{k_0-1}\bar{h}_t$ gives
$$
\Psi(u_{k_0})-\langle u_{k_0},p\rangle
=
\sup_{w\in\mathcal{X}}
\left\{
\langle u_{k_0},w-p\rangle-\frac12\|w-w_0\|_2^2
\right\}.
$$
Dropping the nonpositive quadratic term and using $\|w-p\|_2\le D$ for $w,p\in\mathcal{X}$ yields
$$
\Psi(u_{k_0})-\langle u_{k_0},p\rangle
\le
D\|u_{k_0}\|_2
=
\eta D\left\|
\sum_{t=0}^{k_0-1}
\frac1{\mupdate}\sum_{i\in S_t}\widehat h_t^i
\right\|_2.
$$
Consequently,
\begin{align}
&\sum_{k=k_0}^{K-1}
\left\langle
\frac1{\mupdate}\sum_{i\in S_k}\widehat h_k^i,
\Pi_\mathcal{X}\!\left(
y_k-\frac{\eta}{\mupdate}\sum_{i=1}^n r_k^{(i)}
\right)-p
\right\rangle
\le
\frac{D^2}{2\eta}
+
D\left\|
\sum_{t=0}^{k_0-1}
\frac1{\mupdate}\sum_{i\in S_t}\widehat h_t^i
\right\|_2
+
\frac{\eta}{2}\sum_{k=k_0}^{K-1}
\left\|
\frac1{\mupdate}\sum_{i\in S_k}\widehat h_k^i
\right\|_2^2.
\label{eq:proof-da-late}
\end{align}

These are the dual-averaging inequalities used in the compressed convex analysis.
\end{proof}

\subsection{Constrained proximal geometry for the fully projected branch}

\begin{lemma}[Constrained proximal map: existence, measurability, KKT system, and multiplier bound]
\label{lem:prox-geometry}
Under the standing assumptions, additionally suppose Assumption~\ref{as:slater} holds and fix $\widehat\rho>\rho$. For every \(w\in\mathcal{X}\), the point \(\widehat w(w)\) in \eqref{eq:prox_map} exists and is unique, and the mapping \(w\mapsto\widehat w(w)\) is continuous on \(\mathcal{X}\). There exists a multiplier \(\lambda(w)\ge0\) satisfying the KKT system, complementarity, the uniform bound proved in \eqref{eq:proof-multiplier-bound}, and the proximal-gap inequality \eqref{eq:proof-prox-gap}.
\end{lemma}

\begin{proof}
The constrained proximal construction and the use of Slater's condition to obtain a uniformly bounded multiplier follow the standard framework for weakly convex functional constraints; see, e.g., Section~4 and Lemma~4.2 of \citet{huang2023oracle}. 
For every \(w\in\mathcal{X}\),
\[
y\longmapsto F(y)+\frac{\widehat\rho}{2}\|y-w\|_2^2
\]
is \((\widehat\rho-\rho)\)-strongly convex, because \(F+(\rho/2)\|\cdot\|_2^2\) is convex. The feasible set
\(\{y\in\mathcal{X}:G(y)\le0\}\) is nonempty, compact, and convex. Hence \eqref{eq:prox_map} has a unique minimizer. The proximal map is continuous. Let \(w_j\to w\) in \(\mathcal{X}\). The compactness of the fixed feasible set implies that every subsequence of \(\widehat w(w_j)\) has a convergent further subsequence. If \(\widehat w(w_{j_\ell})\to p\), then \(G(p)\le0\) by continuity of \(G\). For every feasible \(y\), optimality gives
\[
F(\widehat w(w_{j_\ell}))
+\frac{\widehat\rho}{2}\|\widehat w(w_{j_\ell})-w_{j_\ell}\|_2^2
\le
F(y)+\frac{\widehat\rho}{2}\|y-w_{j_\ell}\|_2^2.
\]
Passing to the limit shows that \(p\) minimizes the proximal problem centered at \(w\). Uniqueness gives \(p=\widehat w(w)\). Hence every convergent subsequence has the same limit, so \(\widehat w(w_j)\to\widehat w(w)\). Thus \(w\mapsto\widehat w(w)\) is continuous and Borel measurable. Because \(w_k\) is \(\mathcal F_k^{\mathrm{pre}}\)-measurable, \(\widehat w(w_k)\) is also \(\mathcal F_k^{\mathrm{pre}}\)-measurable and is a legitimate predictable comparator in \eqref{eq:common-directional} and \eqref{eq:common-participation}.

We next prove the existence of a multiplier. We fix $w\in\mathcal{X}$ and define, only within this argument, the perturbation value function
\[
V_w(\tau):=
\inf\left\{
F(y)+\frac{\widehat\rho}{2}\|y-w\|_2^2:
y\in\mathcal{X},\ G(y)\le \tau
\right\}.
\]
Convexity of $F(\cdot)+(\widehat\rho/2)\|\cdot-w\|_2^2$, convexity of $G$, and convexity of $\mathcal{X}$ imply that $V_w$ is convex; it is also nonincreasing. Since $\mathcal{X}$ is compact and the objective and constraint are continuous, the infimum defining $V_w(\tau)$ is attained whenever the feasible set is nonempty. Indeed, the epigraph of $V_w$ is the projection onto the $(\tau,\beta)$-coordinates of the convex set of triples $(y,\tau,\beta)$ satisfying $y\in\mathcal{X}$, $G(y)\le\tau$, and $F(y)+(\widehat\rho/2)\|y-w\|_2^2\le\beta$. Since Assumption~\ref{as:slater} gives $G(w_{\mathrm{feas}})\le-\sigma_{\mathrm s}$, the feasible set is nonempty for every $\tau$ in an open interval containing zero. Compactness of $\mathcal{X}$ and continuity of the objective then imply that $V_w(\tau)$ is finite on this interval. Hence $0\in\operatorname{int}(\operatorname{dom}V_w)$. By Theorems~10.1 and~23.4 of \citet{rockafellar_convex_1970}, $V_w$ is continuous and subdifferentiable at zero. Since $V_w$ is nonincreasing, every subgradient of $V_w$ at zero is nonpositive; choose $-\lambda(w)\in\partial V_w(0)$ with $\lambda(w)\ge0$.

For every $y\in\mathcal{X}$, the definition of $V_w$ gives $F(y)+\frac{\widehat\rho}{2}\|y-w\|_2^2\ge V_w(G(y)).$ By the subgradient inequality at zero, $V_w(G(y))\ge V_w(0)-\lambda(w)G(y).$
Since $\widehat w(w)$ is the unique minimizer of the proximal problem corresponding to $\tau=0$,
\[
V_w(0)
=
F(\widehat w(w))
+\frac{\widehat\rho}{2}\|\widehat w(w)-w\|_2^2.
\]
Therefore,
\[
F(y)+\frac{\widehat\rho}{2}\|y-w\|_2^2
\ge
F(\widehat w(w))
+\frac{\widehat\rho}{2}\|\widehat w(w)-w\|_2^2
-\lambda(w)G(y).
\]
Equivalently,
\begin{equation}
F(y)+\frac{\widehat\rho}{2}\|y-w\|_2^2+\lambda(w)G(y)
\ge
F(\widehat w(w))
+\frac{\widehat\rho}{2}\|\widehat w(w)-w\|_2^2.
\label{eq:proof-lagrangian-global}
\end{equation}
At $y=\widehat w(w)$, feasibility and $\lambda(w)\ge 0$ make the left-hand side of \eqref{eq:proof-lagrangian-global} no larger than the right-hand side. Hence equality holds and $\lambda(w)G(\widehat w(w))=0.$
Thus $\widehat w(w)$ minimizes the convex Lagrangian over $\mathcal{X}$. Define
$$
\Phi_w(y):=
F(y)+\frac{\widehat\rho}{2}\|y-w\|_2^2.
$$
Then $\Phi_w$ is finite and convex on the open neighborhood from Assumption~\ref{as:func_props}, while $G$ is finite and convex there. Moreover, the qualification for the convex subdifferential sum rule holds because $w_{\mathrm{feas}}\in\operatorname{relint}(\mathcal{X})$ and both $\Phi_w$ and $G$ are finite on an open neighborhood of $\mathcal{X}$. Hence, applying the convex Fermat rule and the convex subdifferential sum rule~\citep{rockafellar_convex_1970} to $\Phi_w+\lambda(w)G+\iota_\mathcal{X}$, where $\iota_\mathcal{X}$ denotes the indicator function of $\mathcal{X}$, gives
$$
0\in
\partial \Phi_w(\widehat w(w))
+\lambda(w)\partial G(\widehat w(w))
+\mathcal{N}_\mathcal{X}(\widehat w(w)).
$$
Since the quadratic term is smooth, the standard smooth-sum rule for the limiting subdifferential gives
$$
\partial \Phi_w(\widehat w(w))
=
\partial F(\widehat w(w))
+\widehat\rho(\widehat w(w)-w).
$$
Therefore,
\begin{equation}
0\in
\partial F(\widehat w(w))
+\widehat\rho(\widehat w(w)-w)
+\lambda(w)\partial G(\widehat w(w))
+\mathcal{N}_\mathcal{X}(\widehat w(w)).
\label{eq:proof-kkt}
\end{equation}
Together with feasibility and complementarity, this is the KKT system for the constrained proximal subproblem.

To bound the multiplier, use Lagrangian minimality at the strict Slater point:
\begin{align*}
F(\widehat w(w))
+\frac{\widehat\rho}{2}\|\widehat w(w)-w\|_2^2
&\le
F(w_{\mathrm{feas}})
+\frac{\widehat\rho}{2}\|w_{\mathrm{feas}}-w\|_2^2
+\lambda(w)G(w_{\mathrm{feas}}).
\end{align*}
Since \(G(w_{\mathrm{feas}})\le-\sigma_{\mathrm s}\), \(F\) is \(L\)-Lipschitz on \(\mathcal{X}\), and all three points belong to a set of diameter at most \(D\),
\[
\lambda(w)\sigma_{\mathrm s}
\le
F(w_{\mathrm{feas}})-F(\widehat w(w))
+\frac{\widehat\rho}{2}
\left(\|w_{\mathrm{feas}}-w\|_2^2-\|\widehat w(w)-w\|_2^2\right)
\le
LD+\frac{\widehat\rho}{2}D^2.
\]
Therefore
\begin{equation}
0\le\lambda(w)
\le
\frac{LD+(\widehat\rho/2)D^2}{\sigma_{\mathrm s}}.
\label{eq:proof-multiplier-bound}
\end{equation}

Finally, the function
\[
y\longmapsto
F(y)+\frac{\widehat\rho}{2}\|y-w\|_2^2
+\lambda(w)G(y)+\iota_\mathcal{X}(y)
\]
is \((\widehat\rho-\rho)\)-strongly convex and minimized at \(\widehat w(w)\). Evaluating its strong-convexity inequality at \(y=w\), using \(w\in\mathcal{X}\) and complementarity, gives
\[
F(w)+\lambda(w)G(w)
\ge
F(\widehat w(w))
+\frac{\widehat\rho}{2}\|\widehat w(w)-w\|_2^2
+\frac{\widehat\rho-\rho}{2}\|w-\widehat w(w)\|_2^2.
\]
Rearranging and subtracting \((\rho/2)\|w-\widehat w(w)\|_2^2\) gives
\begin{equation}
F(w)-F(\widehat w(w))
-\frac\rho2\|w-\widehat w(w)\|_2^2
\ge
(\widehat\rho-\rho)\|w-\widehat w(w)\|_2^2
-\lambda(w)G(w).
\label{eq:proof-prox-gap}
\end{equation}

\end{proof}

\begin{lemma}[Fully projected Moreau master inequality]
\label{lem:projected-master}
Under the standing assumptions and Assumption~\ref{as:slater}, suppose $\mathtt{comp}=\mathrm{off}$ and fix $\widehat\rho>\rho$ and $\delta'\in(0,1)$. On any consecutive reporting interval $k=a,\ldots,b-1$, the upper estimate \eqref{eq:proof-projected-upper} holds deterministically. On the intersection of the directional-gradient and update-participation events from Lemma~\ref{lem:common-concentration}, the lower estimate \eqref{eq:proof-projected-lower} holds. On the intersection of the windowed squared-gradient, directional-gradient, and update-participation events, the master inequality \eqref{eq:proof-projected-master} holds.
\end{lemma}

\begin{proof}

For clarity, we write $\widehat w_k:=\widehat w(w_k)$ and $\bar h_k:=\mupdate^{-1}\sum_{i\in S_k}\widehat h_k^i$. Since $w_{k+1}=\Pi_\mathcal{X}(w_k-\eta\bar h_k),$
optimality of $\widehat w_{k+1}$ in \eqref{eq:prox_map} with center $w_{k+1}$ and feasibility of $\widehat w_k$ give
$$
F(\widehat w_{k+1})+\frac{\widehat\rho}{2}\|\widehat w_{k+1}-w_{k+1}\|_2^2
\le
F(\widehat w_k)+\frac{\widehat\rho}{2}\|\widehat w_k-w_{k+1}\|_2^2.
$$
Moreover, since $\widehat w_k\in\mathcal{X}$, projection nonexpansiveness yields
$$
\|\widehat w_k-w_{k+1}\|_2
\le
\|\widehat w_k-w_k+\eta\bar h_k\|_2.
$$
Therefore,
\begin{align*}
F(\widehat w_{k+1})
+\frac{\widehat\rho}{2}\|\widehat w_{k+1}-w_{k+1}\|_2^2
&\le
F(\widehat w_k)
+\frac{\widehat\rho}{2}\|\widehat w_k-w_k\|_2^2
-\widehat\rho\eta
\langle\bar h_k,w_k-\widehat w_k\rangle
+\frac{\widehat\rho\eta^2}{2}\|\bar h_k\|_2^2.
\end{align*}

Summing over $k=a,\ldots,b-1$ telescopes the proximal-value terms. Since $w^\star$ is feasible, optimality of $\widehat w_a$ in \eqref{eq:prox_map} gives
$$
F(\widehat w_a)+\frac{\widehat\rho}{2}\|\widehat w_a-w_a\|_2^2
\le
F(w^\star)+\frac{\widehat\rho}{2}\|w^\star-w_a\|_2^2
\le
F(w^\star)+\frac{\widehat\rho D^2}{2}.
$$
Since $\widehat w_b$ is feasible and $w^\star$ minimizes $F$ over the feasible set,
$$
F(\widehat w_b)+\frac{\widehat\rho}{2}\|\widehat w_b-w_b\|_2^2
\ge
F(\widehat w_b)
\ge
F(w^\star).
$$
Consequently,
\begin{align}
&\sum_{k=a}^{b-1}
\left\langle
\frac1{\mupdate}\sum_{i\in S_k}\widehat h_k^i,
\,w_k-\widehat w(w_k)
\right\rangle
\le
\frac{D^2}{2\eta}
+\frac{\eta}{2}\sum_{k=a}^{b-1}
\left\|
\frac1{\mupdate}\sum_{i\in S_k}\widehat h_k^i
\right\|_2^2.
\label{eq:proof-projected-upper}
\end{align}

We next derive the corresponding lower bound. We define $l_k:=w_k-\widehat w_k$. For every client $i$, apply the subgradient inequality implied by $\rho$-weak convexity of $f_i$ with $w=w_k$ and $y=\widehat w_k$. Rearranging gives
$$
\langle\nabla f_i(w_k),l_k\rangle
\ge
f_i(w_k)-f_i(\widehat w_k)-\frac{\rho}{2}\|l_k\|_2^2.
$$
Similarly, applying the convex subgradient inequality for $g_i$ with $w=w_k$ and $y=\widehat w_k$ gives
$$
\langle\nabla g_i(w_k),l_k\rangle
\ge
g_i(w_k)-g_i(\widehat w_k).
$$

Define the noiseless switched gradient update
$$
h_k^i:=I_k\nabla f_i(w_k)+(1-I_k)\nabla g_i(w_k).
$$
Combining the preceding two inequalities yields
$$
\langle h_k^i,l_k\rangle
\ge
I_k\left[f_i(w_k)-f_i(\widehat w_k)-\frac{\rho}{2}\|l_k\|_2^2\right]
+(1-I_k)\left[g_i(w_k)-g_i(\widehat w_k)\right].
$$

By Lemma~\ref{lem:prox-geometry}, $\widehat w_k$ is $\mathcal F_k^{\mathrm{pre}}$-measurable, and therefore it is a valid comparator in both \eqref{eq:common-directional} and \eqref{eq:common-participation}. Taking $p_k=\widehat w_k$ in \eqref{eq:common-directional} gives
\begin{align*}
&\sum_{k=a}^{b-1}
\left\langle
\frac1{\mupdate}\sum_{i\in S_k}\widehat h_k^i,l_k
\right\rangle \ge
\sum_{k=a}^{b-1}
\left\langle
\frac1{\mupdate}\sum_{i\in S_k}h_k^i,l_k
\right\rangle
-\frac{2D\sigmagradient}{\sqrt{\Bgradient}}
\sqrt{\frac{2(b-a)}{\mupdate}\log\frac1{\delta'}}.
\end{align*}

For the participation term, we define
$$
a_k^i:=
I_k\bigl[f_i(w_k)-f_i(\widehat w_k)\bigr]
+(1-I_k)\bigl[g_i(w_k)-g_i(\widehat w_k)\bigr].
$$
Averaging the preceding clientwise subgradient inequality over $i\in S_k$ gives
$$
\left\langle
\frac1{\mupdate}\sum_{i\in S_k}h_k^i,l_k
\right\rangle
\ge
\frac1{\mupdate}\sum_{i\in S_k}a_k^i
-\frac{\rho}{2}I_k\|l_k\|_2^2.
$$
Taking $p_k=\widehat w_k$ in \eqref{eq:common-participation} gives
\begin{align*}
\sum_{k=a}^{b-1}\frac1{\mupdate}\sum_{i\in S_k}a_k^i
&\ge
\sum_{k=a}^{b-1}
\Bigl\{
I_k\bigl[F(w_k)-F(\widehat w_k)\bigr]
+(1-I_k)\bigl[G(w_k)-G(\widehat w_k)\bigr]
\Bigr\}
-LD\,\chi_{n,\mupdate}
\sqrt{2(b-a)\log\frac1{\delta'}}.
\end{align*}
Combining the last three inequalities yields
\begin{align*}
\sum_{k=a}^{b-1}
\left\langle
\frac1{\mupdate}\sum_{i\in S_k}\widehat h_k^i,l_k
\right\rangle &\ge
\sum_{k=a}^{b-1}
\Bigg\{
I_k\left[
F(w_k)-F(\widehat w_k)
-\frac{\rho}{2}\|l_k\|_2^2
\right]
+(1-I_k)\left[
G(w_k)-G(\widehat w_k)
\right]
\Bigg\}\\
&\qquad\quad
-\frac{2D\sigmagradient}{\sqrt{\Bgradient}}
\sqrt{\frac{2(b-a)}{\mupdate}\log\frac1{\delta'}}
-LD\,\chi_{n,\mupdate}
\sqrt{2(b-a)\log\frac1{\delta'}}.
\end{align*}

Since $\widehat w_k$ is feasible, $G(\widehat w_k)\le0$, and hence
$G(w_k)-G(\widehat w_k)\ge G(w_k).$
Moreover, applying \eqref{eq:proof-prox-gap} with $w=w_k$ gives
$$
F(w_k)-F(\widehat w_k)-\frac{\rho}{2}\|l_k\|_2^2
\ge
(\widehat\rho-\rho)\|l_k\|_2^2-\lambda(w_k)G(w_k).
$$
Therefore, on the intersection of the directional-gradient and update-participation events from Lemma~\ref{lem:common-concentration},
\begin{align}
\sum_{k=a}^{b-1}
\left\langle
\frac1{\mupdate}\sum_{i\in S_k}\widehat h_k^i,
\,w_k-\widehat w(w_k)
\right\rangle
&\ge
\sum_{k=a}^{b-1}
\Bigg\{
I_k\Big[
(\widehat\rho-\rho)\|w_k-\widehat w(w_k)\|_2^2
-\lambda(w_k)G(w_k)
\Big]
+(1-I_k)G(w_k)
\Bigg\}
\nonumber\\
&\qquad
-\frac{2D\sigmagradient}{\sqrt{\Bgradient}}
\sqrt{\frac{2(b-a)}{\mupdate}\log\frac1{\delta'}}
-LD\,\chi_{n,\mupdate}
\sqrt{2(b-a)\log\frac1{\delta'}}.
\label{eq:proof-projected-lower}
\end{align}

Combining \eqref{eq:proof-projected-upper} and \eqref{eq:proof-projected-lower}, and applying \eqref{eq:common-server-square} on the same interval with $J=b-a$, gives, on the intersection of the windowed squared-gradient, directional-gradient, and update-participation events from Lemma~\ref{lem:common-concentration},
\begin{align}
&\sum_{k=a}^{b-1}
\Bigg\{
I_k\Big[
(\widehat\rho-\rho)\|w_k-\widehat w(w_k)\|_2^2
-\lambda(w_k)G(w_k)
\Big]
+(1-I_k)G(w_k)
\Bigg\}
\nonumber\\
&\quad\le
\frac{D^2}{2\eta}
+\frac{\eta}{2}
\Bigg[
(b-a)L^2
+\frac{4L\sigmagradient}{\sqrt{\Bgradient}}
\sqrt{
(b-a)\left(
(b-a)\log 2
+\frac1{\mupdate}\log\frac1{\delta'}
\right)}
+\frac{4\sigmagradient^2}{\Bgradient}
\left(
(b-a)\log 2
+\frac1{\mupdate}\log\frac1{\delta'}
\right)
\Bigg]
\nonumber\\
&\qquad
+\frac{2D\sigmagradient}{\sqrt{\Bgradient}}
\sqrt{
\frac{2(b-a)}{\mupdate}\log\frac1{\delta'}
}
+LD\,\chi_{n,\mupdate}
\sqrt{
2(b-a)\log\frac1{\delta'}
}.
\label{eq:proof-projected-master}
\end{align}
This is the weakly convex master inequality used in the fully projected direct, momentum, and STORM analyses.
\end{proof}

\subsection{Detailed Theorem Statements and Proofs}
\label{sec:proof_details}
\subsubsection{Direct switching}

\begin{theorem}[Direct switching: compressed convex and fully projected weakly convex guarantees]
\label{thm:direct-combined}
Suppose Assumptions~\ref{as:bdd_domain}--\ref{as:sgd_noise}, let $\eta>0$, and let $\delta\in(0,1)$. For the compressed branch, additionally suppose Assumption~\ref{as:compression_biased} holds and $\rho=0$. For the proximal-stationarity guarantee in the fully projected branch, additionally suppose Assumption~\ref{as:slater} holds and fix $\widehat\rho>\rho$. The convex objective-gap specialization of the fully projected branch does not require Assumption~\ref{as:slater}.

\noindent\textbf{(a) Compressed convex branch.}
Suppose \(\mathrm{comp}=\mathrm{on}\), \(\rho=0\), and choose
\begin{align}
\theta
={}&
\sqrt{2V_{\mathrm{val}}\log\frac{10K}{\delta}}
+\frac1K\Bigg\{
\frac{D^2}{2\eta}
+\frac\eta2\Bigg[
KL^2
+\frac{4L\sigmagradient}{\sqrt{\Bgradient}}
\sqrt{K\left(K\log2+\frac1{\mupdate}\log\frac5\delta\right)}
+\frac{4\sigmagradient^2}{\Bgradient}
\left(K\log2+\frac1{\mupdate}\log\frac5\delta\right)
\Bigg]
\nonumber\\
&+\eta\left(
L+\frac{2\sigmagradient}{\sqrt{\Bgradient}}
\sqrt{\log\frac{10K\mupdate}{\delta}}
\right)
\left[
\frac{2n\sqrt{1-\qclient}}{\mupdate\qclient}
+\frac{8\sqrt{(1-\qserver)(2-\qserver)}}{\qserver\qclient}
\right]
\nonumber\\
&\hspace{5em}\times
\sqrt{K\Bigg[KL^2+\frac{4L\sigmagradient}{\sqrt{\Bgradient}}\sqrt{K\left(K\log2+\frac1{\mupdate}\log\frac5\delta\right)}+\frac{4\sigmagradient^2}{\Bgradient}\left(K\log2+\frac1{\mupdate}\log\frac5\delta\right)\Bigg]}
\nonumber\\
&+\frac{2D\sigmagradient}{\sqrt{\Bgradient}}
\sqrt{\frac{2K}{\mupdate}\log\frac5\delta}
+LD\,\chi_{n,\mupdate}\sqrt{2K\log\frac5\delta}
\Bigg\}.
\label{eq:theta-direct-on}
\end{align}
Then, with probability at least \(1-\delta\), there is at least one objective round. For the average of the objective-round iterates,
\begin{equation}
\bar w=\frac1{|\mathcal A|}\sum_{k\in\mathcal A}w_k,
\label{eq:direct-average}
\end{equation}
we have
\begin{equation}
F(\bar w)-F(w^\star)
\le
\theta-\sqrt{2V_{\mathrm{val}}\log\frac{10K}{\delta}},
\label{eq:direct-on-objective}
\end{equation}
and
\begin{equation}
G(\bar w)
\le
\theta+\sqrt{2V_{\mathrm{val}}\log\frac{10K}{\delta}}.
\label{eq:direct-on-feasibility}
\end{equation}

\noindent\textbf{(b) Uncompressed branch.}
Suppose \(\mathrm{comp}=\mathrm{off}\), allow arbitrary \(\rho\ge0\), and choose
\begin{align}
\theta
={}&
\sqrt{2V_{\mathrm{val}}\log\frac{10K}{\delta}}
+\frac1K\Bigg\{
\frac{D^2}{2\eta}
+\frac\eta2\Bigg[
KL^2
+\frac{4L\sigmagradient}{\sqrt{\Bgradient}}
\sqrt{K\left(K\log2+\frac1{\mupdate}\log\frac5\delta\right)}
+\frac{4\sigmagradient^2}{\Bgradient}
\left(K\log2+\frac1{\mupdate}\log\frac5\delta\right)
\Bigg]
\nonumber\\
&+\frac{2D\sigmagradient}{\sqrt{\Bgradient}}
\sqrt{\frac{2K}{\mupdate}\log\frac5\delta}
+LD\,\chi_{n,\mupdate}\sqrt{2K\log\frac5\delta}
\Bigg\}.
\label{eq:theta-direct-off}
\end{align}
Then, with probability at least \(1-\delta\), there is at least one objective round and
\begin{align}
\frac1{|\mathcal A|}
\sum_{k\in\mathcal A}
\|w_k-\widehat w(w_k)\|_2^2
\le
\frac{1}{\widehat\rho-\rho}
\Bigg[
\theta-\sqrt{2V_{\mathrm{val}}\log\frac{10K}{\delta}}
+\frac{LD+(\widehat\rho/2)D^2}{\sigma_{\mathrm s}}
\left(
\theta+\sqrt{2V_{\mathrm{val}}\log\frac{10K}{\delta}}
\right)
\Bigg],
\label{eq:direct-off-stationarity-square}
\end{align}
so that the corresponding average proximal distance is bounded by the square root of the right-hand side of \eqref{eq:direct-off-stationarity-square}. Moreover,
\begin{equation}
G(\bar w)
\le
\theta+\sqrt{2V_{\mathrm{val}}\log\frac{10K}{\delta}}.
\label{eq:direct-off-feasibility}
\end{equation}
If \(\rho=0\), then the same run also satisfies the convex objective-gap bound
\begin{equation}
F(\bar w)-F(w^\star)
\le
\theta-\sqrt{2V_{\mathrm{val}}\log\frac{10K}{\delta}}.
\label{eq:direct-off-convex-objective}
\end{equation}
The objective-gap statement \eqref{eq:direct-off-convex-objective} does not require Slater's condition.
\end{theorem}

\begin{proof}
\noindent\textbf{Compressed branch.}\quad Take $\delta'=\delta/5$ in
\eqref{eq:common-uniform-value} and in the windowed squared-gradient,
uniform-gradient, directional-gradient, and update-participation events
from Lemma~\ref{lem:common-concentration}. Their intersection has
probability at least $1-\delta$. We work on this intersection throughout this branch.

Using \eqref{eq:proof-da-full} with $p=w^\star$, adding and subtracting
the corrected projected point, and applying Cauchy--Schwarz gives
\begin{align*}
\sum_{k=0}^{K-1}
\left\langle
\frac1{\mupdate}\sum_{i\in S_k}\widehat h_k^i,
\,w_k-w^\star
\right\rangle
&\le
\frac{D^2}{2\eta}
+\frac\eta2\sum_{k=0}^{K-1}
\left\|
\frac1{\mupdate}\sum_{i\in S_k}\widehat h_k^i
\right\|_2^2\\
&\qquad
+\left(
\max_{0\le k<K}
\left\|
w_k-
\Pi_\mathcal{X}\!\left(
y_k-\frac\eta{\mupdate}\sum_{i=1}^n r_k^{(i)}
\right)
\right\|_2
\right)
\sum_{k=0}^{K-1}
\left\|
\frac1{\mupdate}\sum_{i\in S_k}\widehat h_k^i
\right\|_2.
\end{align*}

Equations~\eqref{eq:common-server-square} and \eqref{eq:proof-operational-tracking}, followed by Cauchy--Schwarz,
\[
\sum_{k=0}^{K-1}
\left\|
\frac1{\mupdate}\sum_{i\in S_k}\widehat h_k^i
\right\|_2
\le
\sqrt{
K
\sum_{k=0}^{K-1}
\left\|
\frac1{\mupdate}\sum_{i\in S_k}\widehat h_k^i
\right\|_2^2
},
\]
give the first three optimization terms inside the braces of
\eqref{eq:theta-direct-on}.

We next derive the corresponding lower bound. Define
\[
h_k^i
:=
I_k\nabla f_i(w_k)
+
(1-I_k)\nabla g_i(w_k).
\]
Since $\rho=0$, both $f_i$ and $g_i$ are convex. Hence
\[
\langle h_k^i,w_k-w^\star\rangle
\ge
I_k\bigl[f_i(w_k)-f_i(w^\star)\bigr]
+
(1-I_k)\bigl[g_i(w_k)-g_i(w^\star)\bigr].
\]
Combining the clientwise inequality above with \eqref{eq:common-directional}, using the fixed comparator $p_k=w^\star$, gives
\begin{align*}
\sum_{k=0}^{K-1}
\left\langle
\frac1{\mupdate}\sum_{i\in S_k}\widehat h_k^i,
\,w_k-w^\star
\right\rangle &\ge
\sum_{k=0}^{K-1}
\frac1{\mupdate}\sum_{i\in S_k}
\Bigl(
I_k[f_i(w_k)-f_i(w^\star)]
+
(1-I_k)[g_i(w_k)-g_i(w^\star)]
\Bigr)
\\
& -\frac{2D\sigmagradient}{\sqrt{\Bgradient}}
\sqrt{\frac{2K}{\mupdate}\log\frac5\delta}.
\end{align*}
Applying \eqref{eq:common-participation}, again with $p_k=w^\star$, then gives
\begin{align*}
\sum_{k=0}^{K-1}
\left\langle
\frac1{\mupdate}\sum_{i\in S_k}\widehat h_k^i,
\,w_k-w^\star
\right\rangle &\ge
\sum_{k=0}^{K-1}
\Bigl(
I_k[F(w_k)-F(w^\star)]
+
(1-I_k)[G(w_k)-G(w^\star)]
\Bigr)\\
&\qquad\quad
-\frac{2D\sigmagradient}{\sqrt{\Bgradient}}
\sqrt{\frac{2K}{\mupdate}\log\frac5\delta}
-LD\,\chi_{n,\mupdate}\sqrt{2K\log\frac5\delta}.
\end{align*}
Since $w^\star$ is feasible, $G(w^\star)\le0$, and therefore
\[
G(w_k)-G(w^\star)\ge G(w_k).
\]
Combining this lower bound with the preceding upper estimate and
substituting the explicit bounds yields
\begin{equation}
\sum_{k=0}^{K-1}
\Bigl(
I_k[F(w_k)-F(w^\star)]
+
(1-I_k)G(w_k)
\Bigr)
\le
K\bigl(\theta-\sqrt{2V_{\mathrm{val}}\log\frac{10K}{\delta}}\bigr).
\label{eq:proof-direct-on-master}
\end{equation}

Recall that, for direct switching, $\mathcal A=\{0\le k<K:I_k=1\}$, and define $\mathcal B:=\{0,\ldots,K-1\}\setminus\mathcal A$. Now, if $\mathcal A=\varnothing$, then every round is a constraint round, so
$\widehat G_k>\theta$. By \eqref{eq:common-uniform-value} with
$\delta'=\delta/5$,
\[
G(w_k)
>
\theta-\sqrt{2V_{\mathrm{val}}\log\frac{10K}{\delta}}
\qquad
\text{for every }k.
\]
Hence
\[
\sum_{k=0}^{K-1}G(w_k)
>
K(\theta-\sqrt{2V_{\mathrm{val}}\log\frac{10K}{\delta}}),
\]
contradicting \eqref{eq:proof-direct-on-master}. Thus
$\mathcal A\neq\varnothing$.

Moreover, for every $k\in\mathcal B$,
$G(w_k)>\theta-\sqrt{2V_{\mathrm{val}}\log\frac{10K}{\delta}}$. Therefore,
\begin{align*}
\sum_{k\in\mathcal A}
\bigl[F(w_k)-F(w^\star)\bigr]
&\le
K(\theta-\sqrt{2V_{\mathrm{val}}\log\frac{10K}{\delta}})
-\sum_{k\in\mathcal B}G(w_k)\\
&<
K(\theta-\sqrt{2V_{\mathrm{val}}\log\frac{10K}{\delta}})
-|\mathcal B|(\theta-\sqrt{2V_{\mathrm{val}}\log\frac{10K}{\delta}})\\
&=
|\mathcal A|(\theta-\sqrt{2V_{\mathrm{val}}\log\frac{10K}{\delta}}).
\end{align*}
Dividing by $|\mathcal A|$ and using convexity of $F$ gives
\[
F(\bar w)-F(w^\star)
\le
\frac1{|\mathcal A|}
\sum_{k\in\mathcal A}
\bigl[F(w_k)-F(w^\star)\bigr]
\le
\theta-\sqrt{2V_{\mathrm{val}}\log\frac{10K}{\delta}},
\]
which is \eqref{eq:direct-on-objective}.

Finally, if $k\in\mathcal A$, then $\widehat G_k\le\theta$, and by \eqref{eq:common-uniform-value} with $\delta'=\delta/5$
\[
G(w_k)
\le
\widehat G_k+\sqrt{2V_{\mathrm{val}}\log\frac{10K}{\delta}}
\le
\theta+\sqrt{2V_{\mathrm{val}}\log\frac{10K}{\delta}}.
\]
Thus, by convexity of $G$,
\[
G(\bar w)
\le
\frac1{|\mathcal A|}
\sum_{k\in\mathcal A}G(w_k)
\le
\theta+\sqrt{2V_{\mathrm{val}}\log\frac{10K}{\delta}},
\]
which proves \eqref{eq:direct-on-feasibility}.

\noindent\textbf{Uncompressed branch.}\quad
Take $\delta'=\delta/5$ in \eqref{eq:common-uniform-value} and in the windowed squared-gradient, directional-gradient, and update-participation events from Lemma~\ref{lem:common-concentration}. Their intersection has probability at least $1-4\delta/5\ge1-\delta$. Applying \eqref{eq:proof-projected-master} on the full interval $a=0$, $b=K$ with $\delta'=\delta/5$ gives
\begin{align}
&\sum_{k=0}^{K-1}
\Bigg\{
I_k\Big[(\widehat\rho-\rho)\|w_k-\widehat w(w_k)\|_2^2-\lambda(w_k)G(w_k)\Big]
+(1-I_k)G(w_k)
\Bigg\}
\le
K\left(
\theta-\sqrt{2V_{\mathrm{val}}\log\frac{10K}{\delta}}
\right).
\label{eq:proof-direct-off-master}
\end{align}
If $\mathcal A=\varnothing$, then every round is a constraint round, so
$\widehat G_k>\theta$. By \eqref{eq:common-uniform-value} with
$\delta'=\delta/5$,
\[
G(w_k)>
\theta-\sqrt{2V_{\mathrm{val}}\log\frac{10K}{\delta}}
\qquad
\text{for every }k,
\]
which contradicts \eqref{eq:proof-direct-off-master}. Hence
$\mathcal A\neq\varnothing$.

For every $k\in\mathcal A$, we have $\widehat G_k\le\theta$, and therefore
\[
G(w_k)
\le
\theta+\sqrt{2V_{\mathrm{val}}\log\frac{10K}{\delta}}.
\]
For every $k\in\mathcal B$, the reverse switching inequality gives
\[
G(w_k)
>
\theta-\sqrt{2V_{\mathrm{val}}\log\frac{10K}{\delta}}.
\]
Hence \eqref{eq:proof-direct-off-master} implies
\begin{align*}
\sum_{k\in\mathcal A}
\Big[
(\widehat\rho-\rho)
\|w_k-\widehat w(w_k)\|_2^2
-\lambda(w_k)G(w_k)
\Big]
&\le
K\left(
\theta-\sqrt{2V_{\mathrm{val}}\log\frac{10K}{\delta}}
\right)
-\sum_{k\in\mathcal B}G(w_k)\\
&<
|\mathcal A|
\left(
\theta-\sqrt{2V_{\mathrm{val}}\log\frac{10K}{\delta}}
\right).
\end{align*}
Since $\theta+\sqrt{2V_{\mathrm{val}}\log(10K/\delta)}\ge 0$, the objective-round bound together with \eqref{eq:proof-multiplier-bound} implies that, for every $k\in\mathcal A$
$$
\lambda(w_k)G(w_k)
\le
\frac{LD+(\widehat\rho/2)D^2}{\sigma_{\mathrm s}}
\left(
\theta+\sqrt{2V_{\mathrm{val}}\log\frac{10K}{\delta}}
\right).
$$
Consequently,
\begin{align*}
(\widehat\rho-\rho)
\sum_{k\in\mathcal A}
\|w_k-\widehat w(w_k)\|_2^2
\le
|\mathcal A|
\Bigg[\theta-\sqrt{2V_{\mathrm{val}}\log\frac{10K}{\delta}}+
\frac{LD+(\widehat\rho/2)D^2}{\sigma_{\mathrm s}}
\left(
\theta+\sqrt{2V_{\mathrm{val}}\log\frac{10K}{\delta}}
\right)
\Bigg].
\end{align*}
Dividing by $(\widehat\rho-\rho)|\mathcal A|$ proves
\eqref{eq:direct-off-stationarity-square}. Cauchy--Schwarz then gives
\[
\frac1{|\mathcal A|}
\sum_{k\in\mathcal A}
\|w_k-\widehat w(w_k)\|_2
\le
\left(
\frac1{|\mathcal A|}
\sum_{k\in\mathcal A}
\|w_k-\widehat w(w_k)\|_2^2
\right)^{1/2}.
\]
Finally, convexity of $G$ and the objective-round bound above give
\[
G(\bar w)
\le
\frac1{|\mathcal A|}
\sum_{k\in\mathcal A}G(w_k)
\le
\theta+\sqrt{2V_{\mathrm{val}}\log\frac{10K}{\delta}},
\]
which proves \eqref{eq:direct-off-feasibility}.

When $\rho=0$, the convex objective-gap conclusion can be obtained without
Assumption~\ref{as:slater}. Indeed, projection nonexpansiveness with the
fixed feasible comparator $w^\star$ gives
\[
\sum_{k=0}^{K-1}
\left\langle
\frac1{\mupdate}\sum_{i\in S_k}\widehat h_k^i,
\,w_k-w^\star
\right\rangle
\le
\frac{D^2}{2\eta}
+
\frac{\eta}{2}
\sum_{k=0}^{K-1}
\left\|
\frac1{\mupdate}\sum_{i\in S_k}\widehat h_k^i
\right\|_2^2.
\]
Since $\rho=0$, convexity of $f_i$ and $g_i$, together with
\eqref{eq:common-directional} and \eqref{eq:common-participation} with
$p_k=w^\star$, gives the corresponding lower bound. Using
$G(w^\star)\le0$ and substituting \eqref{eq:common-server-square} therefore yields
\[
\sum_{k=0}^{K-1}
\Bigl(
I_k[F(w_k)-F(w^\star)]
+
(1-I_k)G(w_k)
\Bigr)
\le
K\left(
\theta-\sqrt{2V_{\mathrm{val}}\log\frac{10K}{\delta}}
\right).
\]
The same $\mathcal A$--$\mathcal B$ switching argument used above first shows that $\mathcal A\neq\varnothing$ and then gives \eqref{eq:direct-off-convex-objective}. No Slater condition or proximal multiplier is used in this argument.
\end{proof}

\subsubsection{Momentum switching}

\begin{lemma}[Momentum tracking in the two branches]
\label{lem:momentum-tracking-both}
Under the standing assumptions, fix $k_0\in\{1,\ldots,K-1\}$, $\alpha\in(0,1]$, and $\delta\in(0,1)$. For the compressed branch, additionally suppose Assumption~\ref{as:compression_biased} holds.
For either branch, with probability at least $1-\delta/2$, the uniform-gradient event with probability $\delta/6$ holds and, simultaneously, for every $k=k_0,\ldots,K-1$, the corresponding tracking bound below holds.

\noindent \textbf{Compressed branch}
\begin{align}
|Z_k-G(w_k)|
\le{}&
(1-\alpha)^{k_0}\sqrt{2V_{\mathrm{val}}\log\frac{12}{\delta}}
+\sqrt{\frac{2\alpha}{2-\alpha}V_{\mathrm{val}}\log\frac{12K}{\delta}}
\nonumber\\
&+\frac{1-\alpha}{\alpha}L\eta
\left(
L+\frac{2\sigmagradient}{\sqrt{\Bgradient}}
\sqrt{\log\frac{12K\mupdate}{\delta}}
\right)
\left[
\frac{16\sqrt{(1-\qserver)(2-\qserver)}}{\qserver\qclient}
+\frac{8\sqrt2}{\qclient}
\right].
\label{eq:momentum-track-on}
\end{align}
\noindent \textbf{Uncompressed branch},
\begin{align}
|Z_k-G(w_k)|
\le{}&
(1-\alpha)^{k_0}\sqrt{2V_{\mathrm{val}}\log\frac{12}{\delta}}
+\sqrt{\frac{2\alpha}{2-\alpha}V_{\mathrm{val}}\log\frac{12K}{\delta}}
+\frac{1-\alpha}{\alpha}L\eta
\left(
L+\frac{2\sigmagradient}{\sqrt{\Bgradient}}
\sqrt{\log\frac{12K\mupdate}{\delta}}
\right).
\label{eq:momentum-track-off}
\end{align}
\end{lemma}

\begin{proof}
\noindent\textbf{Compressed branch.}\quad

For \(k\ge1\),
\begin{align*}
Z_k-G(w_k) = (1-\alpha)(Z_{k-1}-G(w_{k-1})) +\alpha(\widehat G_k-G(w_k))+(1-\alpha)(G(w_{k-1})-G(w_k)).
\end{align*}
Iterating this recursion separates the initialization error, a geometrically weighted value-noise sum, and a geometrically weighted model-movement term. To see that, let us define $e_k:=Z_k-G(w_k)$ and $ \Xi_k:=\widehat G_k-G(w_k).$ Since $e_0 = \Xi_0$,
\[ e_k = (1-\alpha)^k\Xi_0 + \alpha\sum_{t=1}^k(1-\alpha)^{k-t}\Xi_t
+ (1-\alpha)\sum_{t=1}^k (1-\alpha)^{k-t} \bigl(G(w_{t-1})-G(w_t)\bigr).\]

For every $k\ge k_0$, the initialization term is bounded, on an event of failure probability at most $\delta/6$, by $(1-\alpha)^{k_0}\sqrt{2V_{\mathrm{val}}\log\frac{12}{\delta}}.$
For each fixed $k$, the conditional MGF bound \eqref{eq:common-value-mgf} shows that the weighted value-noise sum is sub-Gaussian with variance proxy
$
\alpha^2V_{\mathrm{val}}
\sum_{t=1}^k(1-\alpha)^{2(k-t)}
\le
\frac{\alpha}{2-\alpha}V_{\mathrm{val}}.
$
A union bound over $k=1,\ldots,K-1$ therefore gives, with failure
probability at most $\delta/6$,
\[
\sqrt{\frac{2\alpha}{2-\alpha}V_{\mathrm{val}}\log\frac{12K}{\delta}}.
\]
Let $z_k:=\Pi_\mathcal{X}(y_k)$. By the compressed server update and projection
nonexpansiveness,
\[
\|w_{k+1}-w_k\|_2
\le
\|\mathcal Q_0(z_{k+1}-w_k)\|_2
\le
2\|z_{k+1}-w_k\|_2,
\]
where the last inequality follows from
$\|\mathcal Q_0(a)\|_2\le2\|a\|_2$. Hence
\[
\|w_{k+1}-w_k\|_2
\le
2\|z_{k+1}-z_k\|_2
+
2\|z_k-w_k\|_2.
\]
Moreover, projection nonexpansiveness gives
\[
\|z_{k+1}-z_k\|_2
\le
\eta
\left\|
\frac1{\mupdate}\sum_{i\in S_k}v_k^i
\right\|_2.
\]
Therefore, on the uniform-gradient event,
\eqref{eq:proof-message} and \eqref{eq:proof-downlink} yield
\begin{align}
\|w_{k+1}-w_k\|_2
\le{}&
\eta
\left(
L+\frac{2\sigmagradient}{\sqrt{\Bgradient}}
\sqrt{\log\frac{12K\mupdate}{\delta}}
\right)
\left[
\frac{16\sqrt{(1-\qserver)(2-\qserver)}}{\qserver\qclient}
+\frac{8\sqrt2}{\qclient}
\right].
\label{eq:movement-on}
\end{align}

\noindent\textbf{Uncompressed branch.}\quad On the uniform-gradient event with failure probability \(\delta/6\),
\[
\left\|\frac1{\mupdate}\sum_{i\in S_k}\widehat h_k^i\right\|_2
\le
L+\frac{2\sigmagradient}{\sqrt{\Bgradient}}
\sqrt{\log\frac{12K\mupdate}{\delta}},
\]
so projection nonexpansiveness gives
\begin{equation}
\|w_{k+1}-w_k\|_2
\le
\eta\left(
L+\frac{2\sigmagradient}{\sqrt{\Bgradient}}
\sqrt{\log\frac{12K\mupdate}{\delta}}
\right).
\label{eq:movement-off}
\end{equation}
Repeating the same momentum recursion therefore gives, simultaneously for every \(k\ge k_0\),
\begin{align*}
|Z_k-G(w_k)|
\le{}&
(1-\alpha)^{k_0}\sqrt{2V_{\mathrm{val}}\log\frac{12}{\delta}}
+\sqrt{\frac{2\alpha}{2-\alpha}V_{\mathrm{val}}\log\frac{12K}{\delta}}+\frac{1-\alpha}{\alpha}L\eta
\left(
L+\frac{2\sigmagradient}{\sqrt{\Bgradient}}
\sqrt{\log\frac{12K\mupdate}{\delta}}
\right).
\end{align*}

This is \eqref{eq:momentum-track-off}. The same three-event allocation gives total failure probability at most \(\delta/2\).
\end{proof}

\begin{theorem}[Momentum switching: compressed convex and uncompressed guarantees]
\label{thm:momentum-combined}
Assume Assumptions~\ref{as:bdd_domain}--\ref{as:sgd_noise}, let \(\alpha\in(0,1]\), \(k_0\in\{1,\ldots,K-1\}\), \(\eta>0\), and \(\delta\in(0,1)\). For the compressed statement also assume Assumption~\ref{as:compression_biased} and \(\rho=0\). For the weakly convex proximal-stationarity statement, additionally
suppose Assumption~\ref{as:slater} holds and fix
$\widehat\rho>\rho$.

\noindent\textbf{(a) Compressed branch.} Set
\begin{align}
\theta &= (1-\alpha)^{k_0}\sqrt{2V_{\mathrm{val}}\log\frac{12}{\delta}}
+\sqrt{\frac{2\alpha}{2-\alpha}V_{\mathrm{val}}\log\frac{12K}{\delta}}
\nonumber\\
&\quad
+\frac{1-\alpha}{\alpha}L\eta
\left(
L+\frac{2\sigmagradient}{\sqrt{\Bgradient}}
\sqrt{\log\frac{12K\mupdate}{\delta}}
\right)
\left[
\frac{16\sqrt{(1-\qserver)(2-\qserver)}}{\qserver\qclient}
+\frac{8\sqrt2}{\qclient}
\right]
\nonumber\\
&\quad
+\frac1{K-k_0}\Bigg\{
\frac{D^2}{2\eta}
+Dk_0\left(
L+\frac{2\sigmagradient}{\sqrt{\Bgradient}}
\sqrt{\log\frac{12K\mupdate}{\delta}}
\right)
\nonumber\\
&
+\frac\eta2\Bigg[
(K-k_0)L^2
+\frac{4L\sigmagradient}{\sqrt{\Bgradient}}
\sqrt{(K-k_0)\left((K-k_0)\log2+\frac1{\mupdate}\log\frac6\delta\right)}
+\frac{4\sigmagradient^2}{\Bgradient}
\left((K-k_0)\log2+\frac1{\mupdate}\log\frac6\delta\right)
\Bigg]
\nonumber\\
&\qquad
+\eta\left(
L+\frac{2\sigmagradient}{\sqrt{\Bgradient}}
\sqrt{\log\frac{12K\mupdate}{\delta}}
\right)
\left[
\frac{2n\sqrt{1-\qclient}}{\mupdate\qclient}
+\frac{8\sqrt{(1-\qserver)(2-\qserver)}}{\qserver\qclient}
\right]
\nonumber\\
&\hspace{5em}\times
\left\{
(K-k_0)\left[
\begin{aligned}
&(K-k_0)L^2\\
&+\frac{4L\sigmagradient}{\sqrt{\Bgradient}}
\sqrt{(K-k_0)\left((K-k_0)\log2+\frac1{\mupdate}\log\frac6\delta\right)}\\
&+\frac{4\sigmagradient^2}{\Bgradient}
\left((K-k_0)\log2+\frac1{\mupdate}\log\frac6\delta\right)
\end{aligned}
\right]
\right\}^{1/2}
\nonumber\\
&\qquad
+\frac{2D\sigmagradient}{\sqrt{\Bgradient}}
\sqrt{\frac{2(K-k_0)}{\mupdate}\log\frac6\delta}
+LD\,\chi_{n,\mupdate}\sqrt{2(K-k_0)\log\frac6\delta}
\Bigg\}.
\label{eq:theta-momentum-on}
\end{align}
Then, with probability at least \(1-\delta\), there is at least one objective round in \(\{k_0,\ldots,K-1\}\). In both branches, set
\[
\bar w=\frac{1}{|\mathcal{A}|}\sum_{k \in \mathcal{A}}w_k.
\]
For the compressed branch this average satisfies
\begin{align}
F(\bar w)-F(w^\star)
\le{}&
\theta
-(1-\alpha)^{k_0}\sqrt{2V_{\mathrm{val}}\log\frac{12}{\delta}}
-\sqrt{\frac{2\alpha}{2-\alpha}V_{\mathrm{val}}\log\frac{12K}{\delta}}
\nonumber\\
&-\frac{1-\alpha}{\alpha}L\eta
\left(
L+\frac{2\sigmagradient}{\sqrt{\Bgradient}}
\sqrt{\log\frac{12K\mupdate}{\delta}}
\right)
\left[
\frac{16\sqrt{(1-\qserver)(2-\qserver)}}{\qserver\qclient}
+\frac{8\sqrt2}{\qclient}
\right],
\label{eq:momentum-on-objective}
\end{align}
while
\begin{align}
G(\bar w)
\le{}&
\theta
+(1-\alpha)^{k_0}\sqrt{2V_{\mathrm{val}}\log\frac{12}{\delta}}
+\sqrt{\frac{2\alpha}{2-\alpha}V_{\mathrm{val}}\log\frac{12K}{\delta}}
\nonumber\\
&+\frac{1-\alpha}{\alpha}L\eta
\left(
L+\frac{2\sigmagradient}{\sqrt{\Bgradient}}
\sqrt{\log\frac{12K\mupdate}{\delta}}
\right)
\left[
\frac{16\sqrt{(1-\qserver)(2-\qserver)}}{\qserver\qclient}
+\frac{8\sqrt2}{\qclient}
\right].
\label{eq:momentum-on-feasibility}
\end{align}

\noindent\textbf{(b) Uncompressed branch.} Allow arbitrary \(\rho\ge0\), and set
\begin{align}
\theta &= (1-\alpha)^{k_0}\sqrt{2V_{\mathrm{val}}\log\frac{12}{\delta}}
+\sqrt{\frac{2\alpha}{2-\alpha}V_{\mathrm{val}}\log\frac{12K}{\delta}}
+\frac{1-\alpha}{\alpha}L\eta
\left(
L+\frac{2\sigmagradient}{\sqrt{\Bgradient}}
\sqrt{\log\frac{12K\mupdate}{\delta}}
\right)
\nonumber\\
&\quad
+\frac1{K-k_0}\Bigg\{
\frac{D^2}{2\eta}
+\frac\eta2\Bigg[
(K-k_0)L^2
+\frac{4L\sigmagradient}{\sqrt{\Bgradient}}
\sqrt{(K-k_0)\left((K-k_0)\log2+\frac1{\mupdate}\log\frac6\delta\right)}
\nonumber\\
&
+\frac{4\sigmagradient^2}{\Bgradient}
\left((K-k_0)\log2+\frac1{\mupdate}\log\frac6\delta\right)
\Bigg]
+\frac{2D\sigmagradient}{\sqrt{\Bgradient}}
\sqrt{\frac{2(K-k_0)}{\mupdate}\log\frac6\delta}
+LD\,\chi_{n,\mupdate}\sqrt{2(K-k_0)\log\frac6\delta}
\Bigg\}.
\label{eq:theta-momentum-off}
\end{align}
Then, with probability at least \(1-\delta\), there is at least one objective round in \(\{k_0,\ldots,K-1\}\), and
\begin{align}
&\frac{1}{|\mathcal{A}|}
\sum_{k\in \mathcal{A}}
\|w_k-\widehat w(w_k)\|_2^2
\le
\frac1{\widehat\rho-\rho}
\Bigg\{
\theta
-(1-\alpha)^{k_0}\sqrt{2V_{\mathrm{val}}\log\frac{12}{\delta}}
-\sqrt{\frac{2\alpha}{2-\alpha}V_{\mathrm{val}}\log\frac{12K}{\delta}}
\nonumber\\
&\qquad
-\frac{1-\alpha}{\alpha}L\eta
\left(
L+\frac{2\sigmagradient}{\sqrt{\Bgradient}}
\sqrt{\log\frac{12K\mupdate}{\delta}}
\right)
+\frac{LD+(\widehat\rho/2)D^2}{\sigma_{\mathrm s}}
\Bigg[
\theta
+(1-\alpha)^{k_0}\sqrt{2V_{\mathrm{val}}\log\frac{12}{\delta}}
\nonumber\\
&+\sqrt{\frac{2\alpha}{2-\alpha}V_{\mathrm{val}}\log\frac{12K}{\delta}}
+\frac{1-\alpha}{\alpha}L\eta
\left(
L+\frac{2\sigmagradient}{\sqrt{\Bgradient}}
\sqrt{\log\frac{12K\mupdate}{\delta}}
\right)
\Bigg]
\Bigg\}.
\label{eq:momentum-off-stationarity-square}
\end{align}
The average proximal distance is bounded by the square root of the right-hand side of \eqref{eq:momentum-off-stationarity-square}. The average accepted iterate satisfies
\begin{align}
G(\bar w)
\le
\theta
+(1-\alpha)^{k_0}\sqrt{2V_{\mathrm{val}}\log\frac{12}{\delta}}
+\sqrt{\frac{2\alpha}{2-\alpha}V_{\mathrm{val}}\log\frac{12K}{\delta}}
+\frac{1-\alpha}{\alpha}L\eta
\left(
L+\frac{2\sigmagradient}{\sqrt{\Bgradient}}
\sqrt{\log\frac{12K\mupdate}{\delta}}
\right).
\label{eq:momentum-off-feasibility}
\end{align}
If \(\rho=0\), the same run also satisfies
\begin{align}
F(\bar w)-F(w^\star)
\le{}&
\theta
-(1-\alpha)^{k_0}\sqrt{2V_{\mathrm{val}}\log\frac{12}{\delta}}
-\sqrt{\frac{2\alpha}{2-\alpha}V_{\mathrm{val}}\log\frac{12K}{\delta}}
-\frac{1-\alpha}{\alpha}L\eta
\left(
L+\frac{2\sigmagradient}{\sqrt{\Bgradient}}
\sqrt{\log\frac{12K\mupdate}{\delta}}
\right).
\label{eq:momentum-off-convex-objective}
\end{align}
This convex objective-gap statement does not require Slater's condition.

\end{theorem}

\begin{proof}
For either branch, we use Lemma~\ref{lem:momentum-tracking-both} to get the corresponding constraint-tracking bound together with the uniform-gradient event, with total failure probability at most $\delta/2$. Take failure probability $\delta/6$ in each of the windowed squared-gradient, directional-gradient, and update-participation events from Lemma~\ref{lem:common-concentration} on the interval $k=k_0,\ldots,K-1$. By a union bound, the intersection of these events has probability at least $1-\delta$. We work on this intersection throughout.

\noindent\textbf{Compressed branch.}\quad
Let $\Gamma_{\mathrm{on}}$ denote the right-hand side of \eqref{eq:momentum-track-on}. Thus, for every $k=k_0,\ldots,K-1$, $|Z_k-G(w_k)|\le \Gamma_{\mathrm{on}}.$

Using \eqref{eq:proof-da-late} with $p=w^\star$, and adding and subtracting the corrected projected point as in the direct-switching proof, gives
\begin{align*}
\sum_{k=k_0}^{K-1}
\left\langle
\frac1{\mupdate}\sum_{i\in S_k}\widehat h_k^i,
\,w_k-w^\star
\right\rangle &\le
\frac{D^2}{2\eta}
+
D\left\|
\sum_{t=0}^{k_0-1}
\frac1{\mupdate}\sum_{i\in S_t}\widehat h_t^i
\right\|_2
+
\frac{\eta}{2}
\sum_{k=k_0}^{K-1}
\left\|
\frac1{\mupdate}\sum_{i\in S_k}\widehat h_k^i
\right\|_2^2\\
&\qquad
+
\left(
\max_{k_0\le k<K}
\left\|
w_k-
\Pi_\mathcal{X}\!\left(
y_k-\frac{\eta}{\mupdate}\sum_{i=1}^n r_k^{(i)}
\right)
\right\|_2
\right)
\sum_{k=k_0}^{K-1}
\left\|
\frac1{\mupdate}\sum_{i\in S_k}\widehat h_k^i
\right\|_2.
\end{align*}
On the uniform-gradient event,
\[
D\left\|
\sum_{t=0}^{k_0-1}
\frac1{\mupdate}\sum_{i\in S_t}\widehat h_t^i
\right\|_2
\le
Dk_0
\left(
L+\frac{2\sigmagradient}{\sqrt{\Bgradient}}
\sqrt{\log\frac{12K\mupdate}{\delta}}
\right).
\]
Moreover, \eqref{eq:common-server-square} and
\eqref{eq:proof-operational-tracking}, together with
\[
\sum_{k=k_0}^{K-1}
\left\|
\frac1{\mupdate}\sum_{i\in S_k}\widehat h_k^i
\right\|_2
\le
\sqrt{
(K-k_0)
\sum_{k=k_0}^{K-1}
\left\|
\frac1{\mupdate}\sum_{i\in S_k}\widehat h_k^i
\right\|_2^2
},
\]
give exactly the optimization terms appearing inside the braces of
\eqref{eq:theta-momentum-on}.

For the corresponding lower bound, we have $h_k^i = I_k\nabla f_i(w_k) + (1-I_k)\nabla g_i(w_k).$
Since $\rho=0$, both $f_i$ and $g_i$ are convex, and hence
\[
\langle h_k^i,w_k-w^\star\rangle
\ge
I_k[f_i(w_k)-f_i(w^\star)]
+
(1-I_k)[g_i(w_k)-g_i(w^\star)].
\]
Combining this clientwise inequality with
\eqref{eq:common-directional}, using the fixed comparator
$p_k=w^\star$, and then applying
\eqref{eq:common-participation} with the same comparator yields
\begin{align*}
\sum_{k=k_0}^{K-1}
\left\langle
\frac1{\mupdate}\sum_{i\in S_k}\widehat h_k^i,
\,w_k-w^\star
\right\rangle
&\ge
\sum_{k=k_0}^{K-1}
\Bigl(
I_k[F(w_k)-F(w^\star)]
+
(1-I_k)[G(w_k)-G(w^\star)]
\Bigr)\\
&\qquad\quad
-\frac{2D\sigmagradient}{\sqrt{\Bgradient}}
\sqrt{
\frac{2(K-k_0)}{\mupdate}\log\frac6\delta
}
-
LD\,\chi_{n,\mupdate}
\sqrt{
2(K-k_0)\log\frac6\delta
}.
\end{align*}
Since $w^\star$ is feasible, $G(w^\star)\le0$, so $G(w_k)-G(w^\star)\ge G(w_k).$
Combining the preceding upper and lower estimates and substituting the explicit bounds therefore gives
\begin{equation}
\sum_{k=k_0}^{K-1}
\Bigl(
I_k[F(w_k)-F(w^\star)]
+
(1-I_k)G(w_k)
\Bigr)
\le
(K-k_0)(\theta-\Gamma_{\mathrm{on}}).
\label{eq:proof-momentum-on-master}
\end{equation}

Recall that, on the reporting window, $\mathcal A:=\{k_0\le k<K:I_k=1\}$ and $\mathcal B:=\{k_0,\ldots,K-1\}\setminus\mathcal A.$ Now, if $\mathcal A=\varnothing$, then $Z_k>\theta$ for every $k=k_0,\ldots,K-1$. Since $|Z_k-G(w_k)|\le\Gamma_{\mathrm{on}}$,
\[
G(w_k)>\theta-\Gamma_{\mathrm{on}}
\qquad
\text{for every }k.
\]
Consequently, $\sum_{k=k_0}^{K-1}G(w_k) > (K-k_0)(\theta-\Gamma_{\mathrm{on}}),$ contradicting \eqref{eq:proof-momentum-on-master}. Hence
$\mathcal A\neq\varnothing$.

Moreover, for every $k\in\mathcal B$,
\[
G(w_k)>\theta-\Gamma_{\mathrm{on}}.
\]
Therefore, by \eqref{eq:proof-momentum-on-master},
\begin{align*}
\sum_{k\in\mathcal A}
[F(w_k)-F(w^\star)]
&\le
(K-k_0)(\theta-\Gamma_{\mathrm{on}})
-\sum_{k\in\mathcal B}G(w_k)\\
&<
(K-k_0)(\theta-\Gamma_{\mathrm{on}})
-|\mathcal B|(\theta-\Gamma_{\mathrm{on}})\\
&=
|\mathcal A|(\theta-\Gamma_{\mathrm{on}}).
\end{align*}
Dividing by $|\mathcal A|$ and using convexity of $F$ gives
\[
F(\bar w)-F(w^\star)
\le
\frac1{|\mathcal A|}
\sum_{k\in\mathcal A}
[F(w_k)-F(w^\star)]
\le
\theta-\Gamma_{\mathrm{on}},
\]
which is \eqref{eq:momentum-on-objective}.

Finally, if $k\in\mathcal A$, then $Z_k\le\theta$, and therefore
\[
G(w_k)
\le
Z_k+\Gamma_{\mathrm{on}}
\le
\theta+\Gamma_{\mathrm{on}}.
\]
Hence, by convexity of $G$,
\[
G(\bar w)
\le
\frac1{|\mathcal A|}
\sum_{k\in\mathcal A}G(w_k)
\le
\theta+\Gamma_{\mathrm{on}},
\]
which is \eqref{eq:momentum-on-feasibility}.

\noindent\textbf{Fully projected branch.}\quad
Let $\Gamma_{\mathrm{off}}$ denote the right-hand side of
\eqref{eq:momentum-track-off}. Thus, $|Z_k-G(w_k)| \le \Gamma_{\mathrm{off}}$ for every $k=k_0,\ldots,K-1.$

Applying Lemma~\ref{lem:projected-master} on the interval
$a=k_0$, $b=K$ with $\delta'=\delta/6$ gives
\begin{equation}
\begin{aligned}
\sum_{k=k_0}^{K-1}
\Bigg\{
I_k\Big[
(\widehat\rho-\rho)
\|w_k-\widehat w(w_k)\|_2^2
-\lambda(w_k)G(w_k)
\Big]
+
(1-I_k)G(w_k)
\Bigg\}\le
(K-k_0)(\theta-\Gamma_{\mathrm{off}}).
\end{aligned}
\label{eq:proof-momentum-off-master}
\end{equation}

Again let $\mathcal A:=\{k_0\le k<K:I_k=1\}$ and $\mathcal B:=\{k_0,\ldots,K-1\}\setminus\mathcal A.$ Now, if $\mathcal A=\varnothing$, then $Z_k>\theta$ throughout the reporting window. The tracking bound therefore gives
\[
G(w_k)>\theta-\Gamma_{\mathrm{off}}
\qquad
\text{for every }k=k_0,\ldots,K-1,
\]
which contradicts \eqref{eq:proof-momentum-off-master}. Thus
$\mathcal A\neq\varnothing$.

For every $k\in\mathcal A$, $Z_k\le\theta$, and hence
\[
G(w_k)
\le
\theta+\Gamma_{\mathrm{off}}.
\]
For every $k\in\mathcal B$,
\[
G(w_k)
>
\theta-\Gamma_{\mathrm{off}}.
\]
Consequently, \eqref{eq:proof-momentum-off-master} implies
\begin{align*}
\sum_{k\in\mathcal A}
\Big[
(\widehat\rho-\rho)
\|w_k-\widehat w(w_k)\|_2^2
-\lambda(w_k)G(w_k)
\Big]
&\le
(K-k_0)(\theta-\Gamma_{\mathrm{off}})
-\sum_{k\in\mathcal B}G(w_k)\\
&<
|\mathcal A|(\theta-\Gamma_{\mathrm{off}}).
\end{align*}
Since
$\theta+\Gamma_{\mathrm{off}}\ge0$, the objective-round constraint bound together with \eqref{eq:proof-multiplier-bound} implies that, for every $k\in\mathcal A$,
\[
\lambda(w_k)G(w_k)
\le
\frac{LD+(\widehat\rho/2)D^2}{\sigma_{\mathrm s}}
\bigl(\theta+\Gamma_{\mathrm{off}}\bigr).
\]
Therefore,
\begin{align*}
(\widehat\rho-\rho)
\sum_{k\in\mathcal A}
\|w_k-\widehat w(w_k)\|_2^2
\le
|\mathcal A|
\Bigg[
\theta-\Gamma_{\mathrm{off}}
+
\frac{LD+(\widehat\rho/2)D^2}{\sigma_{\mathrm s}}
\bigl(\theta+\Gamma_{\mathrm{off}}\bigr)
\Bigg].
\end{align*}
Dividing by
$(\widehat\rho-\rho)|\mathcal A|$ and substituting the explicit expression for $\Gamma_{\mathrm{off}}$ proves
\eqref{eq:momentum-off-stationarity-square}.

Cauchy--Schwarz then gives
\[
\frac1{|\mathcal A|}
\sum_{k\in\mathcal A}
\|w_k-\widehat w(w_k)\|_2
\le
\left(
\frac1{|\mathcal A|}
\sum_{k\in\mathcal A}
\|w_k-\widehat w(w_k)\|_2^2
\right)^{1/2},
\]
which proves the stated average proximal-distance bound. Moreover, convexity of $G$ and the pointwise objective-round estimate give
\[
G(\bar w)
\le
\frac1{|\mathcal A|}
\sum_{k\in\mathcal A}G(w_k)
\le
\theta+\Gamma_{\mathrm{off}},
\]
which is \eqref{eq:momentum-off-feasibility}.

Finally, suppose $\rho=0$. The convex objective-gap conclusion can be obtained without Assumption~\ref{as:slater}. Indeed, applying the same projected one-step argument on
$k=k_0,\ldots,K-1$ with the fixed feasible comparator $w^\star$, and combining convexity with
\eqref{eq:common-directional},
\eqref{eq:common-participation}, and
\eqref{eq:common-server-square}, yields
\[
\sum_{k=k_0}^{K-1}
\Bigl(
I_k[F(w_k)-F(w^\star)]
+
(1-I_k)G(w_k)
\Bigr)
\le
(K-k_0)(\theta-\Gamma_{\mathrm{off}}).
\]
The same $\mathcal A$--$\mathcal B$ switching argument used above first shows that $\mathcal A\neq\varnothing$ and then gives
\[
F(\bar w)-F(w^\star)
\le
\theta-\Gamma_{\mathrm{off}},
\]
which is \eqref{eq:momentum-off-convex-objective}. This specialization uses no Slater condition or proximal multiplier.
\end{proof}

\subsubsection{STORM switching}
\begin{lemma}[Same-sample STORM tracking in the two branches]
\label{lem:storm-tracking-both}
Under the standing assumptions and Assumption~\ref{as:same_sample},
fix $k_0\in\{1,\ldots,K-1\}$, $\alpha\in(0,1]$, and $\delta\in(0,1)$. For the compressed branch, additionally suppose Assumption~\ref{as:compression_biased} holds. For every $k\ge1$ and every $\lambda\in\mathbb R$,
\begin{align}
&\E\!\left[
\exp\!\left(
\lambda[(\widehat G_k-\widehat G_k^-)-(G(w_k)-G(w_{k-1}))]
\right)
\middle|\mathcal F_k^{\mathrm{pre}}
\right]
\le
\exp\!\left\{
\frac{\lambda^2}{2}
\left(
L^2\chi_{n,\mvalue}^2
+\frac{\sigma_\Delta^2}{\mvalue\Bvalue}
\right)
\|w_k-w_{k-1}\|_2^2
\right\}.
\label{eq:proof-storm-difference}
\end{align}
With probability at least \(1-\delta/2\), the uniform-gradient event with level \(\delta/6\) holds and, simultaneously for every \(k=k_0,\ldots,K-1\),

\noindent\textbf{Compressed branch}
\begin{align}
\nonumber
&|Z_k-G(w_k)|
\le 
(1-\alpha)^{k_0}\sqrt{2V_{\mathrm{val}}\log\frac{12}{\delta}}
+\Bigg\{
\frac{4\log(12K/\delta)}{\alpha(2-\alpha)}
\Bigg[
\alpha^2V_{\mathrm{val}}
\\
&+\left(
L^2\chi_{n,\mvalue}^2
+\frac{\sigma_\Delta^2}{\mvalue\Bvalue}
\right)
\eta^2
\left(
L+\frac{2\sigmagradient}{\sqrt{\Bgradient}}
\sqrt{\log\frac{12K\mupdate}{\delta}}
\right)^2
\left(
\frac{16\sqrt{(1-\qserver)(2-\qserver)}}{\qserver\qclient}
+\frac{8\sqrt2}{\qclient}
\right)^2
\Bigg]
\Bigg\}^{1/2}
\label{eq:storm-track-on}
\end{align}
for the compressed branch, while

\noindent\textbf{Uncompressed branch}
\begin{align}
|Z_k-G(w_k)|
\le{}&
(1-\alpha)^{k_0}\sqrt{2V_{\mathrm{val}}\log\frac{12}{\delta}}
\nonumber\\
&+\Bigg\{
\frac{4\log(12K/\delta)}{\alpha(2-\alpha)}
\Bigg[
\alpha^2V_{\mathrm{val}}
+\left(
L^2\chi_{n,\mvalue}^2
+\frac{\sigma_\Delta^2}{\mvalue\Bvalue}
\right)
\eta^2
\left(
L+\frac{2\sigmagradient}{\sqrt{\Bgradient}}
\sqrt{\log\frac{12K\mupdate}{\delta}}
\right)^2
\Bigg]
\Bigg\}^{1/2}
\label{eq:storm-track-off}
\end{align}
for the uncompressed branch.
\end{lemma}

\begin{proof}
Before we go into the compressed and uncompressed branch proofs specifically, as we did in earlier proofs, we need to simplify the same-sample MGF and the STORM innovation and how to control them.

\noindent\textbf{Same-sample conditional MGF.}\quad
Fix $k\ge1$ and condition on $\mathcal F_k^{\mathrm{pre}}$. Then
$w_k$ and $w_{k-1}$ are fixed, while the constraint-evaluation set
$C_k$ and the round-$k$ evaluation samples are fresh. Define
\[
D_k
:=
(\widehat G_k-\widehat G_k^-)
-
(G(w_k)-G(w_{k-1})).
\]
We decompose $D_k=P_k+R_k$, where
\[
P_k
:=
\frac1{\mvalue}
\sum_{i\in C_k}
\bigl[g_i(w_k)-g_i(w_{k-1})\bigr]
-
\bigl[G(w_k)-G(w_{k-1})\bigr]
\]
is the partial-participation error and
\begin{align*}
R_k
:=
\frac1{\mvalue\Bvalue}
\sum_{i\in C_k}\sum_{b=1}^{\Bvalue}
\Big(
[g_i(w_k,\xi_{k,i,b})-g_i(w_{k-1},\xi_{k,i,b})]-[g_i(w_k)-g_i(w_{k-1})]
\Big)
\end{align*}
is the same-sample stochastic error.

By Assumption~\ref{as:cons_noise}(ii), the partial-participation
estimate is unbiased at both $w_k$ and $w_{k-1}$, and hence
$\E[P_k\mid\mathcal F_k^{\mathrm{pre}}]=0.$
Moreover, by $L$-Lipschitz continuity, $|g_i(w_k)-g_i(w_{k-1})|\le L\|w_k-w_{k-1}\|_2.$
Applying the same without-replacement Doob-martingale argument used in
the proof of \eqref{eq:common-participation}, now to the fixed values
$g_i(w_k)-g_i(w_{k-1})$, gives, for every $\lambda\in\mathbb R$,
\[
\E\!\left[
e^{\lambda P_k}
\middle|
\mathcal F_k^{\mathrm{pre}}
\right]
\le
\exp\!\left(
\frac{\lambda^2}{2}
L^2\chi_{n,\mvalue}^2
\|w_k-w_{k-1}\|_2^2
\right).
\]
When $\mvalue=n$, $P_k=0$ identically and the same bound holds because
$\chi_{n,n}=0$.

Next condition additionally on the realized set $C_k$. For
$i\in C_k$ and $b\in[\Bvalue]$, we define
\[
Y_{k,i,b}
:=
[g_i(w_k,\xi_{k,i,b})-g_i(w_{k-1},\xi_{k,i,b})]
-
[g_i(w_k)-g_i(w_{k-1})].
\]
By Assumption~\ref{as:cons_noise}(i), $\E[ Y_{k,i,b} \mid \mathcal F_k^{\mathrm{pre}}\vee\sigma(C_k) ] =0.$
Assumption~\ref{as:same_sample} and the freezing argument give
\[
\E\!\left[
\exp\!\left(
\frac{|Y_{k,i,b}|^2}
{\sigma_\Delta^2\|w_k-w_{k-1}\|_2^2}
\right)
\middle|
\mathcal F_k^{\mathrm{pre}}\vee\sigma(C_k)
\right]
\le2,
\]
with the difference identically zero when $w_k=w_{k-1}$. Hence
Lemma~\ref{lem:square-to-mgf} implies
\[
\E\!\left[
\exp\!\left(
\frac{\lambda}{\mvalue\Bvalue}Y_{k,i,b}
\right)
\middle|
\mathcal F_k^{\mathrm{pre}}\vee\sigma(C_k)
\right]
\le
\exp\!\left(
\frac{\lambda^2\sigma_\Delta^2
\|w_k-w_{k-1}\|_2^2}
{2\mvalue^2\Bvalue^2}
\right).
\]
Conditional on $C_k$, the evaluation samples are independent. Multiplying
the preceding bounds over the $\mvalue\Bvalue$ samples therefore gives
\[
\E\!\left[
e^{\lambda R_k}
\middle|
\mathcal F_k^{\mathrm{pre}}\vee\sigma(C_k)
\right]
\le
\exp\!\left(
\frac{\lambda^2}{2}
\frac{\sigma_\Delta^2}{\mvalue\Bvalue}
\|w_k-w_{k-1}\|_2^2
\right).
\]
The right-hand side does not depend on the realized set $C_k$, while
$P_k$ is measurable with respect to
$\mathcal F_k^{\mathrm{pre}}\vee\sigma(C_k)$. Thus the tower property,
combined with the preceding bound for $P_k$, yields
\begin{align}
\E\!\left[
\exp\!\left(
\lambda D_k
\right)
\middle|
\mathcal F_k^{\mathrm{pre}}
\right]
\le
\exp\!\left\{
\frac{\lambda^2}{2}
\left(
L^2\chi_{n,\mvalue}^2
+
\frac{\sigma_\Delta^2}{\mvalue\Bvalue}
\right)
\|w_k-w_{k-1}\|_2^2
\right\},
\label{eq:proof-storm-difference-rederived}
\end{align}
which is \eqref{eq:proof-storm-difference}.

\noindent\textbf{STORM innovation.}\quad
Let $e_k:=Z_k-G(w_k).$ From the STORM recursion,
\begin{align*}
e_k
=
(1-\alpha)e_{k-1}
+
\alpha\bigl(\widehat G_k^--G(w_{k-1})\bigr)
+
D_k.
\end{align*}
Define the complete round-$k$ innovation $\varsigma_k := \alpha\bigl(\widehat G_k^--G(w_{k-1})\bigr) + D_k.$
Then $e_k=(1-\alpha)e_{k-1}+\varsigma_k,$ and therefore
\begin{equation}
e_k
=
(1-\alpha)^k e_0
+
\sum_{t=1}^k
(1-\alpha)^{k-t}\varsigma_t.
\label{eq:proof-storm-unrolled}
\end{equation}

By the same conditional-MGF argument as in
\eqref{eq:common-value-mgf}, now applied at the
$\mathcal F_k^{\mathrm{pre}}$-measurable point $w_{k-1}$,
Assumption~\ref{as:cons_noise}(i)--(ii) gives
\[
\E\!\left[
\exp\!\left(
\lambda[\widehat G_k^--G(w_{k-1})]
\right)
\middle|
\mathcal F_k^{\mathrm{pre}}
\right]
\le
\exp\!\left(
\frac{\lambda^2V_{\mathrm{val}}}{2}
\right).
\]
The two terms defining $\varsigma_k$ use the same
constraint-evaluation set and the same evaluation samples, so they need
not be conditionally independent. Conditional Cauchy--Schwarz gives
\begin{align*}
\E\!\left[
e^{\lambda\varsigma_k}
\middle|
\mathcal F_k^{\mathrm{pre}}
\right]
&\le
\left(
\E\!\left[
e^{2\lambda\alpha(\widehat G_k^--G(w_{k-1}))}
\middle|
\mathcal F_k^{\mathrm{pre}}
\right]
\right)^{1/2}
\left(
\E\!\left[
e^{2\lambda D_k}
\middle|
\mathcal F_k^{\mathrm{pre}}
\right]
\right)^{1/2}\\
&\le
\exp\!\left\{
\lambda^2
\left[
\alpha^2V_{\mathrm{val}}
+
\left(
L^2\chi_{n,\mvalue}^2
+
\frac{\sigma_\Delta^2}{\mvalue\Bvalue}
\right)
\|w_k-w_{k-1}\|_2^2
\right]
\right\}.
\end{align*}
Equivalently, $\varsigma_k$ has conditional variance proxy
\begin{equation}
2\left[
\alpha^2V_{\mathrm{val}}
+
\left(
L^2\chi_{n,\mvalue}^2
+
\frac{\sigma_\Delta^2}{\mvalue\Bvalue}
\right)
\|w_k-w_{k-1}\|_2^2
\right].
\label{eq:proof-storm-innovation-proxy}
\end{equation}

For convenience, set
\[
H_\delta
:=
L+\frac{2\sigmagradient}{\sqrt{\Bgradient}}
\sqrt{\log\frac{12K\mupdate}{\delta}}
\]
and
\[
C_\Delta
:=
L^2\chi_{n,\mvalue}^2
+
\frac{\sigma_\Delta^2}{\mvalue\Bvalue}.
\]

\noindent\textbf{Compressed branch.}\quad
Define $M_{\mathrm{on}} := \eta H_\delta \left[ \frac{16\sqrt{(1-\qserver)(2-\qserver)}} {\qserver\qclient} + \frac{8\sqrt2}{\qclient} \right],$ which is the deterministic right-hand side of \eqref{eq:movement-on}. Let
\[
\tau_{\mathrm{on}}
:=
\inf\left\{
k\ge1:
\|w_k-w_{k-1}\|_2>M_{\mathrm{on}}
\right\},
\]
with $\tau_{\mathrm{on}}=\infty$ if the set is empty. Since
$w_k$ and $w_{k-1}$ are
$\mathcal F_k^{\mathrm{pre}}$-measurable,
the event $\{k<\tau_{\mathrm{on}}\}$ is also
$\mathcal F_k^{\mathrm{pre}}$-measurable.

Define the stopped innovation $\widetilde\varsigma_k:=\ind_{\{k<\tau_{\mathrm{on}}\}}\varsigma_k.$ Before the stopping time,
$\|w_k-w_{k-1}\|_2\le M_{\mathrm{on}}$, while from the stopping time
onward $\widetilde\varsigma_k=0$. Hence
\eqref{eq:proof-storm-innovation-proxy} gives, for every
$\lambda\in\mathbb R$,
\[
\E\!\left[
e^{\lambda\widetilde\varsigma_k}
\middle|
\mathcal F_k^{\mathrm{pre}}
\right]
\le
\exp\!\left\{
\frac{\lambda^2}{2}
\,2\bigl[
\alpha^2V_{\mathrm{val}}
+
C_\Delta M_{\mathrm{on}}^2
\bigr]
\right\}.
\]

For a fixed $k$, let us define $S_k^{\mathrm{on}} := \sum_{t=1}^{k}(1-\alpha)^{k-t}\widetilde\varsigma_t.$ Iterating the preceding conditional MGF bounds over the rounds gives
\begin{align*}
\E\exp(\lambda S_k^{\mathrm{on}})
&\le
\exp\!\left\{
\frac{\lambda^2}{2}
\,2\bigl[
\alpha^2V_{\mathrm{val}}
+
C_\Delta M_{\mathrm{on}}^2
\bigr]
\sum_{t=1}^k(1-\alpha)^{2(k-t)}
\right\}\\
&\le
\exp\!\left\{
\frac{\lambda^2}{2}
\frac{2}{\alpha(2-\alpha)}
\bigl[
\alpha^2V_{\mathrm{val}}
+
C_\Delta M_{\mathrm{on}}^2
\bigr]
\right\},
\end{align*}
where
\[
\sum_{s=0}^{\infty}(1-\alpha)^{2s}
=
\frac1{\alpha(2-\alpha)}.
\]
Thus, for each fixed reporting round, a two-sided Chernoff bound gives
\[
\Prob\!\left(
|S_k^{\mathrm{on}}|
>
\left\{
\frac{4\log(12K/\delta)}
{\alpha(2-\alpha)}
\bigl[
\alpha^2V_{\mathrm{val}}
+
C_\Delta M_{\mathrm{on}}^2
\bigr]
\right\}^{1/2}
\right)
\le
\frac{\delta}{6K}.
\]
A union bound over at most $K$ reporting rounds shows that, with failure
probability at most $\delta/6$, this bound holds simultaneously for all
$k=k_0,\ldots,K-1$.

Since $Z_0=\widehat G_0$, we have
\[
e_0=\widehat G_0-G(w_0).
\]
By the one-round consequence of \eqref{eq:common-value-mgf},
with failure probability at most $\delta/6$,
\[
|e_0|
\le
\sqrt{
2V_{\mathrm{val}}
\log\frac{12}{\delta}
}.
\]
Finally, on the uniform-gradient event with level $\delta/6$,
\eqref{eq:movement-on} holds in every round, so
$\tau_{\mathrm{on}}=\infty$. Hence
$\widetilde\varsigma_k=\varsigma_k$ throughout the horizon.
Using \eqref{eq:proof-storm-unrolled} and
$(1-\alpha)^k\le(1-\alpha)^{k_0}$ for $k\ge k_0$, we obtain
simultaneously for every $k=k_0,\ldots,K-1$,
\begin{align*}
|Z_k-G(w_k)|
\le{}&
(1-\alpha)^{k_0}
\sqrt{
2V_{\mathrm{val}}
\log\frac{12}{\delta}
}+
\left\{
\frac{4\log(12K/\delta)}
{\alpha(2-\alpha)}
\bigl[
\alpha^2V_{\mathrm{val}}
+
C_\Delta M_{\mathrm{on}}^2
\bigr]
\right\}^{1/2}.
\end{align*}
Substituting the definitions of $C_\Delta$ and $M_{\mathrm{on}}$
gives exactly \eqref{eq:storm-track-on}.

The initialization event, the stopped weighted-innovation event, and the
uniform-gradient event each have failure probability at most
$\delta/6$. Their intersection therefore has probability at least
$1-\delta/2$.

\noindent\textbf{Uncompressed branch.}\quad
Define $M_{\mathrm{off}} := \eta H_\delta,$ the deterministic right-hand side of \eqref{eq:movement-off}, and let $\tau_{\mathrm{off}}:=\inf\left\{k\ge1:\|w_k-w_{k-1}\|_2>M_{\mathrm{off}}\right\}.$ As above, $\{k<\tau_{\mathrm{off}}\}$ is $\mathcal F_k^{\mathrm{pre}}$-measurable. Define
\[
\widetilde\varsigma_k
:=
\ind_{\{k<\tau_{\mathrm{off}}\}}\varsigma_k.
\]
Then
\[
\E\!\left[
e^{\lambda\widetilde\varsigma_k}
\middle|
\mathcal F_k^{\mathrm{pre}}
\right]
\le
\exp\!\left\{
\frac{\lambda^2}{2}
\,2\bigl[
\alpha^2V_{\mathrm{val}}
+
C_\Delta M_{\mathrm{off}}^2
\bigr]
\right\}.
\]
Repeating the same geometrically weighted conditional-MGF argument gives,
with failure probability at most $\delta/6$,
\[
\left|
\sum_{t=1}^k
(1-\alpha)^{k-t}\widetilde\varsigma_t
\right|
\le
\left\{
\frac{4\log(12K/\delta)}
{\alpha(2-\alpha)}
\bigl[
\alpha^2V_{\mathrm{val}}
+
C_\Delta M_{\mathrm{off}}^2
\bigr]
\right\}^{1/2}
\]
simultaneously for every $k=k_0,\ldots,K-1$.

On the uniform-gradient event with level $\delta/6$,
\eqref{eq:movement-off} holds in every round, so
$\tau_{\mathrm{off}}=\infty$ and the stopped and original innovations
coincide. Combining the preceding bound with the same initialization
event gives
\begin{align*}
|Z_k-G(w_k)|
\le{}&
(1-\alpha)^{k_0}
\sqrt{
2V_{\mathrm{val}}
\log\frac{12}{\delta}
}+
\left\{
\frac{4\log(12K/\delta)}
{\alpha(2-\alpha)}
\bigl[
\alpha^2V_{\mathrm{val}}
+
C_\Delta M_{\mathrm{off}}^2
\bigr]
\right\}^{1/2}
\end{align*}
simultaneously for every $k=k_0,\ldots,K-1$.
Substituting
\[
C_\Delta
=
L^2\chi_{n,\mvalue}^2
+
\frac{\sigma_\Delta^2}{\mvalue\Bvalue},
\qquad
M_{\mathrm{off}}
=
\eta
\left(
L+\frac{2\sigmagradient}{\sqrt{\Bgradient}}
\sqrt{\log\frac{12K\mupdate}{\delta}}
\right)
\]
gives exactly \eqref{eq:storm-track-off}.

Again, the initialization event, the stopped weighted-innovation event,
and the uniform-gradient event each fail with probability at most
$\delta/6$. Hence the stated tracking bound and the uniform-gradient
event hold together with probability at least $1-\delta/2$.
\end{proof}

\begin{theorem}[STORM switching: Compressed and Uncompressed guarantees]
\label{thm:storm-combined}
Assume Assumptions~\ref{as:bdd_domain}--\ref{as:sgd_noise} and Assumption~\ref{as:same_sample}, let \(\alpha\in(0,1]\), \(k_0\in\{1,\ldots,K-1\}\), \(\eta>0\), and \(\delta\in(0,1)\). For the compressed statement also assume Assumption~\ref{as:compression_biased} and \(\rho=0\). For the weakly convex proximal-stationarity statement, additionally
suppose Assumption~\ref{as:slater} holds and fix
$\widehat\rho>\rho$.

\noindent\textbf{(a) Compressed branch.}
Choose \(k_0\in\{1,\ldots,K-1\}\), and set
\begin{align}
&\theta =(1-\alpha)^{k_0}\sqrt{2V_{\mathrm{val}}\log\frac{12}{\delta}}
+\Bigg\{
\frac{4\log(12K/\delta)}{\alpha(2-\alpha)}
\Bigg[
\alpha^2V_{\mathrm{val}}
+\left(
L^2\chi_{n,\mvalue}^2
+\frac{\sigma_\Delta^2}{\mvalue\Bvalue}
\right)
\eta^2
\left(
L+\frac{2\sigmagradient}{\sqrt{\Bgradient}}
\sqrt{\log\frac{12K\mupdate}{\delta}}
\right)^2
\nonumber\\
&\times
\left(
\frac{16\sqrt{(1-\qserver)(2-\qserver)}}{\qserver\qclient}
+\frac{8\sqrt2}{\qclient}
\right)^2
\Bigg]
\Bigg\}^{1/2}
+\frac1{K-k_0}\Bigg\{
\frac{D^2}{2\eta}
+Dk_0\left(
L+\frac{2\sigmagradient}{\sqrt{\Bgradient}}
\sqrt{\log\frac{12K\mupdate}{\delta}}
\right)
\nonumber\\
&
+\frac\eta2\Bigg[
(K-k_0)L^2
+\frac{4L\sigmagradient}{\sqrt{\Bgradient}}
\sqrt{(K-k_0)\left((K-k_0)\log2+\frac1{\mupdate}\log\frac6\delta\right)}
+\frac{4\sigmagradient^2}{\Bgradient}
\left((K-k_0)\log2+\frac1{\mupdate}\log\frac6\delta\right)
\Bigg]
\nonumber\\
&\qquad
+\eta\left(
L+\frac{2\sigmagradient}{\sqrt{\Bgradient}}
\sqrt{\log\frac{12K\mupdate}{\delta}}
\right)
\left[
\frac{2n\sqrt{1-\qclient}}{\mupdate\qclient}
+\frac{8\sqrt{(1-\qserver)(2-\qserver)}}{\qserver\qclient}
\right]
\nonumber\\
&\times
\left\{
(K-k_0)\left[
\begin{aligned}
&(K-k_0)L^2\\
&+\frac{4L\sigmagradient}{\sqrt{\Bgradient}}
\sqrt{(K-k_0)\left((K-k_0)\log2+\frac1{\mupdate}\log\frac6\delta\right)}\\
&+\frac{4\sigmagradient^2}{\Bgradient}
\left((K-k_0)\log2+\frac1{\mupdate}\log\frac6\delta\right)
\end{aligned}
\right]
\right\}^{1/2}
\nonumber\\
&\qquad
+\frac{2D\sigmagradient}{\sqrt{\Bgradient}}
\sqrt{\frac{2(K-k_0)}{\mupdate}\log\frac6\delta}
+LD\,\chi_{n,\mupdate}\sqrt{2(K-k_0)\log\frac6\delta}
\Bigg\}.
\label{eq:theta-storm-on}
\end{align}
Then, with probability at least \(1-\delta\), there is at least one objective round in the reporting window. In both branches, set
\[
\bar w=\frac{1}{|\mathcal{A}|}\sum_{k \in \mathcal{A}}w_k.
\]
For the compressed branch,
\begin{align}
F(\bar w)-F(w^\star)
\le{}&
\theta
-(1-\alpha)^{k_0}\sqrt{2V_{\mathrm{val}}\log\frac{12}{\delta}}
-\Bigg\{
\frac{4\log(12K/\delta)}{\alpha(2-\alpha)}
\Bigg[
\alpha^2V_{\mathrm{val}}
+\left(
L^2\chi_{n,\mvalue}^2
+\frac{\sigma_\Delta^2}{\mvalue\Bvalue}
\right)
\eta^2
\nonumber\\
&\times
\left(
L+\frac{2\sigmagradient}{\sqrt{\Bgradient}}
\sqrt{\log\frac{12K\mupdate}{\delta}}
\right)^2
\left(
\frac{16\sqrt{(1-\qserver)(2-\qserver)}}{\qserver\qclient}
+\frac{8\sqrt2}{\qclient}
\right)^2
\Bigg]
\Bigg\}^{1/2}.
\label{eq:storm-on-objective}
\end{align}
and
\begin{align}
\nonumber
G(\bar w)
\le&
\theta
+(1-\alpha)^{k_0}\sqrt{2V_{\mathrm{val}}\log\frac{12}{\delta}}
+\Bigg\{
\frac{4\log(12K/\delta)}{\alpha(2-\alpha)}
\Bigg[
\alpha^2V_{\mathrm{val}}
\\
&+\left(
L^2\chi_{n,\mvalue}^2
+\frac{\sigma_\Delta^2}{\mvalue\Bvalue}
\right)
\eta^2
\left(
L+\frac{2\sigmagradient}{\sqrt{\Bgradient}}
\sqrt{\log\frac{12K\mupdate}{\delta}}
\right)^2
\left(
\frac{16\sqrt{(1-\qserver)(2-\qserver)}}{\qserver\qclient}
+\frac{8\sqrt2}{\qclient}
\right)^2
\Bigg]
\Bigg\}^{1/2}.
\label{eq:storm-on-feasibility}
\end{align}

\noindent\textbf{(b) Uncompressed branch.} Allow arbitrary \(\rho\ge0\), choose \(k_0\in\{1,\ldots,K-1\}\), and set 
\begin{align}
&\theta =(1-\alpha)^{k_0}\sqrt{2V_{\mathrm{val}}\log\frac{12}{\delta}}
\nonumber\\
&\quad+\Bigg\{
\frac{4\log(12K/\delta)}{\alpha(2-\alpha)}
\Bigg[
\alpha^2V_{\mathrm{val}}
+\left(
L^2\chi_{n,\mvalue}^2
+\frac{\sigma_\Delta^2}{\mvalue\Bvalue}
\right)
\eta^2
\left(
L+\frac{2\sigmagradient}{\sqrt{\Bgradient}}
\sqrt{\log\frac{12K\mupdate}{\delta}}
\right)^2
\Bigg]
\Bigg\}^{1/2}
\nonumber\\
&\quad+\frac1{K-k_0}\Bigg\{
\frac{D^2}{2\eta}
+\frac\eta2\Bigg[
(K-k_0)L^2
+\frac{4L\sigmagradient}{\sqrt{\Bgradient}}
\sqrt{(K-k_0)\left((K-k_0)\log2+\frac1{\mupdate}\log\frac6\delta\right)}
\nonumber\\
&
+\frac{4\sigmagradient^2}{\Bgradient}
\left((K-k_0)\log2+\frac1{\mupdate}\log\frac6\delta\right)
\Bigg]
+\frac{2D\sigmagradient}{\sqrt{\Bgradient}}
\sqrt{\frac{2(K-k_0)}{\mupdate}\log\frac6\delta}
+LD\,\chi_{n,\mupdate}\sqrt{2(K-k_0)\log\frac6\delta}
\Bigg\}.
\label{eq:theta-storm-off}
\end{align}
Then, with probability at least \(1-\delta\), there is at least one objective round in the reporting window and
\begin{align}
&\frac{1}{|\mathcal{A}|}
\sum_{k \in \mathcal{A}}
\|w_k-\widehat w(w_k)\|_2^2
\le
\frac1{\widehat\rho-\rho}
\Bigg\{
\theta
-(1-\alpha)^{k_0}\sqrt{2V_{\mathrm{val}}\log\frac{12}{\delta}}
\nonumber\\
&\qquad
-\Bigg\{
\frac{4\log(12K/\delta)}{\alpha(2-\alpha)}
\Bigg[
\alpha^2V_{\mathrm{val}}
+\left(
L^2\chi_{n,\mvalue}^2
+\frac{\sigma_\Delta^2}{\mvalue\Bvalue}
\right)
\eta^2
\left(
L+\frac{2\sigmagradient}{\sqrt{\Bgradient}}
\sqrt{\log\frac{12K\mupdate}{\delta}}
\right)^2
\Bigg]
\Bigg\}^{1/2}
\nonumber\\
&\qquad
+\frac{LD+(\widehat\rho/2)D^2}{\sigma_{\mathrm s}}
\Bigg[
\theta
+(1-\alpha)^{k_0}\sqrt{2V_{\mathrm{val}}\log\frac{12}{\delta}}
\nonumber\\
&
+\Bigg\{
\frac{4\log(12K/\delta)}{\alpha(2-\alpha)}
\Bigg[
\alpha^2V_{\mathrm{val}}
+\left(
L^2\chi_{n,\mvalue}^2
+\frac{\sigma_\Delta^2}{\mvalue\Bvalue}
\right)
\eta^2
\left(
L+\frac{2\sigmagradient}{\sqrt{\Bgradient}}
\sqrt{\log\frac{12K\mupdate}{\delta}}
\right)^2
\Bigg]
\Bigg\}^{1/2}
\Bigg]
\Bigg\}.
\label{eq:storm-off-stationarity-square}
\end{align}
The average proximal distance is bounded by the square root of the right-hand side of \eqref{eq:storm-off-stationarity-square}. Moreover, the average objective-round iterate satisfies
\begin{align}
G(\bar w)
&\le
\theta
+(1-\alpha)^{k_0}\sqrt{2V_{\mathrm{val}}\log\frac{12}{\delta}}
\nonumber\\
&+\Bigg\{
\frac{4\log(12K/\delta)}{\alpha(2-\alpha)}
\Bigg[
\alpha^2V_{\mathrm{val}}
+\left(
L^2\chi_{n,\mvalue}^2
+\frac{\sigma_\Delta^2}{\mvalue\Bvalue}
\right)
\eta^2
\left(
L+\frac{2\sigmagradient}{\sqrt{\Bgradient}}
\sqrt{\log\frac{12K\mupdate}{\delta}}
\right)^2
\Bigg]
\Bigg\}^{1/2}.
\label{eq:storm-off-feasibility}
\end{align}
If \(\rho=0\), the same run also satisfies
\begin{align}
F(\bar w)-F(w^\star)
&\le
\theta
-(1-\alpha)^{k_0}\sqrt{2V_{\mathrm{val}}\log\frac{12}{\delta}}
\nonumber\\
&-\Bigg\{
\frac{4\log(12K/\delta)}{\alpha(2-\alpha)}
\Bigg[
\alpha^2V_{\mathrm{val}}
+\left(
L^2\chi_{n,\mvalue}^2
+\frac{\sigma_\Delta^2}{\mvalue\Bvalue}
\right)
\eta^2
\left(
L+\frac{2\sigmagradient}{\sqrt{\Bgradient}}
\sqrt{\log\frac{12K\mupdate}{\delta}}
\right)^2
\Bigg]
\Bigg\}^{1/2}.
\label{eq:storm-off-convex-objective}
\end{align}
This convex objective-gap statement does not require Slater's condition.

\end{theorem}

\begin{proof}
For either branch, Lemma~\ref{lem:storm-tracking-both} supplies the
corresponding STORM tracking bound together with the uniform-gradient
event, with total failure probability at most $\delta/2$. Take failure
probability $\delta/6$ in each of the windowed squared-gradient,
directional-gradient, and update-participation events from
Lemma~\ref{lem:common-concentration} on the interval
$k=k_0,\ldots,K-1$. By a union bound, the intersection of these events
has probability at least $1-\delta$. We work on this intersection
throughout.

\noindent\textbf{Compressed branch.}\quad
Let $\Gamma_{\mathrm{on}}$ denote the right-hand side of
\eqref{eq:storm-track-on}. Thus, simultaneously for every
$k=k_0,\ldots,K-1$, $|Z_k-G(w_k)| \le \Gamma_{\mathrm{on}}.$

Use \eqref{eq:proof-da-late} with $p=w^\star$. On the uniform-gradient
event, $D\left\| \sum_{t=0}^{k_0-1} \frac1{\mupdate}\sum_{i\in S_t}\widehat h_t^i \right\|_2 \le Dk_0 ( L+\frac{2\sigmagradient}{\sqrt{\Bgradient}} \sqrt{\log\frac{12K\mupdate}{\delta}}).$
Adding and subtracting the corrected projected point and applying
Cauchy--Schwarz as in the direct and momentum analyses gives
\begin{align*}
\sum_{k=k_0}^{K-1}
\left\langle
\frac1{\mupdate}\sum_{i\in S_k}\widehat h_k^i,
\,w_k-w^\star
\right\rangle
&\le
\frac{D^2}{2\eta}
+
D\left\|
\sum_{t=0}^{k_0-1}
\frac1{\mupdate}\sum_{i\in S_t}\widehat h_t^i
\right\|_2
+
\frac{\eta}{2}
\sum_{k=k_0}^{K-1}
\left\|
\frac1{\mupdate}\sum_{i\in S_k}\widehat h_k^i
\right\|_2^2\\
&
+
\left(
\max_{k_0\le k<K}
\left\|
w_k-
\Pi_\mathcal{X}\!\left(
y_k-\frac{\eta}{\mupdate}\sum_{i=1}^n r_k^{(i)}
\right)
\right\|_2
\right)
\sum_{k=k_0}^{K-1}
\left\|
\frac1{\mupdate}\sum_{i\in S_k}\widehat h_k^i
\right\|_2.
\end{align*}
Using \eqref{eq:proof-operational-tracking},
\eqref{eq:common-server-square}, and
\[
\sum_{k=k_0}^{K-1}
\left\|
\frac1{\mupdate}\sum_{i\in S_k}\widehat h_k^i
\right\|_2
\le
\sqrt{
(K-k_0)
\sum_{k=k_0}^{K-1}
\left\|
\frac1{\mupdate}\sum_{i\in S_k}\widehat h_k^i
\right\|_2^2
},
\]
gives exactly the optimization terms inside the braces of
\eqref{eq:theta-storm-on}.

For the corresponding lower bound, we have $h_k^i = I_k\nabla f_i(w_k) + (1-I_k)\nabla g_i(w_k).$ Since $\rho=0$, both $f_i$ and $g_i$ are convex, so
\[
\langle h_k^i,w_k-w^\star\rangle
\ge
I_k[f_i(w_k)-f_i(w^\star)]
+
(1-I_k)[g_i(w_k)-g_i(w^\star)].
\]
Combining this clientwise inequality with
\eqref{eq:common-directional}, using the fixed comparator
$p_k=w^\star$, and then applying
\eqref{eq:common-participation} with the same comparator gives
\begin{align*}
\sum_{k=k_0}^{K-1}
\left\langle
\frac1{\mupdate}\sum_{i\in S_k}\widehat h_k^i,
\,w_k-w^\star
\right\rangle
&\ge
\sum_{k=k_0}^{K-1}
\Bigl(
I_k[F(w_k)-F(w^\star)]
+
(1-I_k)[G(w_k)-G(w^\star)]
\Bigr)\\
&
-\frac{2D\sigmagradient}{\sqrt{\Bgradient}}
\sqrt{
\frac{2(K-k_0)}{\mupdate}\log\frac6\delta
}
-
LD\,\chi_{n,\mupdate}
\sqrt{
2(K-k_0)\log\frac6\delta
}.
\end{align*}
Since $w^\star$ is feasible, $G(w^\star)\le0$, and hence $G(w_k)-G(w^\star)\ge G(w_k).$ Combining the upper and lower estimates and substituting the explicit
bounds appearing in \eqref{eq:theta-storm-on} therefore yields
\begin{equation}
\sum_{k=k_0}^{K-1}
\Bigl(
I_k[F(w_k)-F(w^\star)]
+
(1-I_k)G(w_k)
\Bigr)
\le
(K-k_0)(\theta-\Gamma_{\mathrm{on}}).
\label{eq:proof-storm-on-master}
\end{equation}

Recall that
\[
\mathcal A
=
\{k_0\le k<K:I_k=1\},
\qquad
\mathcal B
:=
\{k_0,\ldots,K-1\}\setminus\mathcal A.
\]
Suppose first that $\mathcal A=\varnothing$. Then every round in the
reporting window is a constraint round, so $Z_k>\theta$. Since
$|Z_k-G(w_k)|\le\Gamma_{\mathrm{on}}$,
\[
G(w_k)
>
\theta-\Gamma_{\mathrm{on}}
\qquad
\text{for every }k=k_0,\ldots,K-1.
\]
Consequently, $\sum_{k=k_0}^{K-1}G(w_k)>(K-k_0)(\theta-\Gamma_{\mathrm{on}}),$ contradicting \eqref{eq:proof-storm-on-master}. Thus $\mathcal A\neq\varnothing$.

Moreover, for every $k\in\mathcal B$, $G(w_k) > \theta-\Gamma_{\mathrm{on}}.$
Therefore, by \eqref{eq:proof-storm-on-master},
\begin{align*}
\sum_{k\in\mathcal A}
[F(w_k)-F(w^\star)]
&\le
(K-k_0)(\theta-\Gamma_{\mathrm{on}})
-\sum_{k\in\mathcal B}G(w_k)\\
&<
(K-k_0)(\theta-\Gamma_{\mathrm{on}})
-|\mathcal B|(\theta-\Gamma_{\mathrm{on}})\\
&=
|\mathcal A|(\theta-\Gamma_{\mathrm{on}}).
\end{align*}
Dividing by $|\mathcal A|$ and using convexity of $F$ gives
\[
F(\bar w)-F(w^\star)
\le
\frac1{|\mathcal A|}
\sum_{k\in\mathcal A}
[F(w_k)-F(w^\star)]
\le
\theta-\Gamma_{\mathrm{on}},
\]
which is \eqref{eq:storm-on-objective} after substituting the explicit
expression for $\Gamma_{\mathrm{on}}$.

Finally, if $k\in\mathcal A$, then $Z_k\le\theta$, and the tracking
bound gives $G(w_k)\le Z_k+\Gamma_{\mathrm{on}}\le \theta+\Gamma_{\mathrm{on}}.$
Thus, by convexity of $G$,
\[
G(\bar w)
\le
\frac1{|\mathcal A|}
\sum_{k\in\mathcal A}G(w_k)
\le
\theta+\Gamma_{\mathrm{on}},
\]
which is \eqref{eq:storm-on-feasibility}.

\noindent\textbf{Uncompressed branch.}\quad
Let $\Gamma_{\mathrm{off}}$ denote the right-hand side of
\eqref{eq:storm-track-off}. Thus, simultaneously for every
$k=k_0,\ldots,K-1$, $|Z_k-G(w_k)| \le \Gamma_{\mathrm{off}}.$

Applying Lemma~\ref{lem:projected-master} on the interval
$a=k_0$, $b=K$ with failure level $\delta/6$, and substituting the
explicit squared-direction, directional-gradient, and participation
bounds appearing in \eqref{eq:theta-storm-off}, gives
\begin{equation}
\begin{aligned}
\sum_{k=k_0}^{K-1}
\Bigg\{
I_k\Big[
(\widehat\rho-\rho)
\|w_k-\widehat w(w_k)\|_2^2
-\lambda(w_k)G(w_k)
\Big]
+
(1-I_k)G(w_k)
\Bigg\} \le
(K-k_0)(\theta-\Gamma_{\mathrm{off}}).
\end{aligned}
\label{eq:proof-storm-off-master}
\end{equation}

Again let $\mathcal B := \{k_0,\ldots,K-1\}\setminus\mathcal A.$ If $\mathcal A=\varnothing$, then $Z_k>\theta$ for every round in the
reporting window. Hence
\[
G(w_k)
>
\theta-\Gamma_{\mathrm{off}}
\qquad
\text{for every }k=k_0,\ldots,K-1,
\]
which contradicts \eqref{eq:proof-storm-off-master}. Therefore
$\mathcal A\neq\varnothing$.

For every $k\in\mathcal A$, $Z_k\le\theta$, so $G(w_k) \le \theta+\Gamma_{\mathrm{off}}$ and for every $k\in\mathcal B$, $Z_k>\theta$, so $G(w_k) > \theta-\Gamma_{\mathrm{off}}.$
Consequently, \eqref{eq:proof-storm-off-master} implies
\begin{align*}
\sum_{k\in\mathcal A}
\Big[
(\widehat\rho-\rho)
\|w_k-\widehat w(w_k)\|_2^2
-\lambda(w_k)G(w_k)
\Big]
&\le
(K-k_0)(\theta-\Gamma_{\mathrm{off}})
-\sum_{k\in\mathcal B}G(w_k)\\
&\qquad<
|\mathcal A|(\theta-\Gamma_{\mathrm{off}}).
\end{align*}

Since $\theta+\Gamma_{\mathrm{off}}\ge0$, the objective-round constraint
bound together with \eqref{eq:proof-multiplier-bound} gives, for every
$k\in\mathcal A$,
\[
\lambda(w_k)G(w_k)
\le
\frac{LD+(\widehat\rho/2)D^2}{\sigma_{\mathrm s}}
\bigl(\theta+\Gamma_{\mathrm{off}}\bigr).
\]
Therefore,
\begin{align*}
(\widehat\rho-\rho)
\sum_{k\in\mathcal A}
\|w_k-\widehat w(w_k)\|_2^2
\le
|\mathcal A|
\Bigg[
\theta-\Gamma_{\mathrm{off}}
+
\frac{LD+(\widehat\rho/2)D^2}{\sigma_{\mathrm s}}
\bigl(\theta+\Gamma_{\mathrm{off}}\bigr)
\Bigg].
\end{align*}
Dividing by
$(\widehat\rho-\rho)|\mathcal A|$ and substituting the explicit
expression for $\Gamma_{\mathrm{off}}$ proves
\eqref{eq:storm-off-stationarity-square}.

Cauchy--Schwarz gives
\[
\frac1{|\mathcal A|}
\sum_{k\in\mathcal A}
\|w_k-\widehat w(w_k)\|_2
\le
\left(
\frac1{|\mathcal A|}
\sum_{k\in\mathcal A}
\|w_k-\widehat w(w_k)\|_2^2
\right)^{1/2},
\]
which proves the stated average proximal-distance bound.

Moreover, by convexity of $G$ and the pointwise objective-round estimate,
\[
G(\bar w)
\le
\frac1{|\mathcal A|}
\sum_{k\in\mathcal A}G(w_k)
\le
\theta+\Gamma_{\mathrm{off}},
\]
which is \eqref{eq:storm-off-feasibility} after substituting the explicit
expression for $\Gamma_{\mathrm{off}}$.

Finally, suppose $\rho=0$. The convex objective-gap conclusion can be
obtained without Assumption~\ref{as:slater}. Applying the projected
one-step inequality on $k=k_0,\ldots,K-1$ with the fixed feasible
comparator $w^\star$ gives
\[
\sum_{k=k_0}^{K-1}
\left\langle
\frac1{\mupdate}\sum_{i\in S_k}\widehat h_k^i,
\,w_k-w^\star
\right\rangle
\le
\frac{D^2}{2\eta}
+
\frac{\eta}{2}
\sum_{k=k_0}^{K-1}
\left\|
\frac1{\mupdate}\sum_{i\in S_k}\widehat h_k^i
\right\|_2^2.
\]
Since $\rho=0$, convexity of $f_i$ and $g_i$, together with
\eqref{eq:common-directional} and
\eqref{eq:common-participation} with $p_k=w^\star$, gives the
corresponding lower bound. Using $G(w^\star)\le0$ and
\eqref{eq:common-server-square} therefore yields
\[
\sum_{k=k_0}^{K-1}
\Bigl(
I_k[F(w_k)-F(w^\star)]
+
(1-I_k)G(w_k)
\Bigr)
\le
(K-k_0)(\theta-\Gamma_{\mathrm{off}}).
\]
The same $\mathcal A$--$\mathcal B$ switching argument used above first
shows that $\mathcal A\neq\varnothing$ and then gives
\[
F(\bar w)-F(w^\star)
\le
\theta-\Gamma_{\mathrm{off}},
\]
which is \eqref{eq:storm-off-convex-objective}. No Slater condition or
proximal multiplier is used in this specialization.
\end{proof}

\newpage
\subsubsection{Proof of the Rate Summaries}
\label{sec:rate_proofs}

We first isolate the two compression factors that appear in the
compressed analysis. Define
\begin{align}
C_{\mathrm{mov}}
&:=
\frac{16\sqrt{(1-\qserver)(2-\qserver)}}
{\qserver\qclient}
+
\frac{8\sqrt2}{\qclient},
\label{eq:Cmov_rate}
\\
C_{\mathrm{op}}
&:=
\frac{2n\sqrt{1-\qclient}}
{\mupdate\qclient}
+
\frac{8\sqrt{(1-\qserver)(2-\qserver)}}
{\qserver\qclient}.
\label{eq:Cop_rate}
\end{align}
The first factor controls the one-step model movement in
\eqref{eq:movement-on} and therefore enters the momentum and STORM
tracking errors. The second controls the discrepancy between the
operational model and the corrected state in
\eqref{eq:proof-operational-tracking} and therefore enters the
optimization remainder.

For convenience, we define
\begin{equation}
\kappa_{\mathrm{dir}}
:=
1+C_{\mathrm{op}},
\qquad
\kappa_{\mathrm{tr}}
:=
1+C_{\mathrm{mov}}+C_{\mathrm{op}}
\label{eq:kappa_comp}
\end{equation}
in the compressed branch, and set
$\kappa_{\mathrm{dir}}=\kappa_{\mathrm{tr}}=1$ in the uncompressed
branch. Since $\qclient,\qserver\in(0,1]$,
\begin{align}
C_{\mathrm{mov}}
&\le
\frac{24\sqrt2}{\qclient\qserver},
\\
C_{\mathrm{op}}
&\le
\frac{2n/\mupdate+8\sqrt2}
{\qclient\qserver}.
\end{align}
Consequently,
\begin{align}
\kappa_{\mathrm{dir}}
&\le
\frac{1+8\sqrt2+2n/\mupdate}
{\qclient\qserver},
\label{eq:kappa_dir_product}
\\
\kappa_{\mathrm{tr}}
&\le
\frac{1+32\sqrt2+2n/\mupdate}
{\qclient\qserver}.
\label{eq:kappa_tr_product}
\end{align}
Thus, when $n/\mupdate$ is fixed,
\[
\sqrt{\kappa_{\mathrm{dir}}},
\sqrt{\kappa_{\mathrm{tr}}}
=
\mathcal{O}\!\left(
\frac1{\sqrt{\qclient\qserver}}
\right).
\]

\begin{corollary}
\label{cor:rates}
Assume $D,L>0$ and fix $\delta\in(0,1)$. Keep $n$, the participation
sizes $\mvalue,\mupdate$, the mini-batch sizes $\Bvalue,\Bgradient$,
and the noise parameters $\sigmavalue,\sigmaparticipation,\sigmagradient$ fixed independently of $K$; for STORM, also keep $\sigma_\Delta$ fixed. For weakly convex case we have $\rho$, $\widehat\rho$, and the Slater margin $\sigma_{\mathrm s}$ fixed independently of $K$.

In the compressed branch, let $\qclient,\qserver\in(0,1]$. Further, let $\kappa_{\mathrm{dir}}$ and $\kappa_{\mathrm{tr}}$ be defined by~\eqref{eq:kappa_comp}. In the uncompressed branch, set both equal to one. The $\widetilde{\mathcal{O}}(\cdot)$ notation below hides factors polylogarithmic in $K/\delta$ and constants depending on the remaining fixed problem, noise, and participation parameters, but the dependence on $\qclient,\qserver$ is displayed explicitly through $\kappa_{\mathrm{dir}}$ and $\kappa_{\mathrm{tr}}$.

\begin{enumerate}

\item \textbf{Direct switching.}
Choose $ \eta = \frac{D} {L\sqrt{\kappa_{\mathrm{dir}}K}}. $ Then the optimization remainder is $ \widetilde{\mathcal{O}}\!\left( \sqrt{\kappa_{\mathrm{dir}}}\,K^{-1/2} \right)$. In particular, in the compressed convex branch,
\begin{align*}
F(\bar w)-F(w^\star)
&\le
\widetilde{\mathcal{O}}\!\left(
\sqrt{\kappa_{\mathrm{dir}}}\,K^{-1/2}
\right),
\\
G(\bar w)
&\le
\widetilde{\mathcal{O}}\!\left(
\sqrt{\kappa_{\mathrm{dir}}}\,K^{-1/2}
+\sqrt{V_{\mathrm{val}}}
\right).
\end{align*}
Hence, by~\eqref{eq:kappa_dir_product}, the optimization contribution has compression dependence $\widetilde{\mathcal{O}}\!\left( \frac{K^{-1/2}} {\sqrt{\qclient\qserver}} \right)$ when $n/\mupdate$ is fixed.

For fixed $V_{\mathrm{val}}>0$, the direct constraint-value term does
not vanish with $K$, so direct-switching feasibility and weakly convex
proximal-stationarity are not guaranteed to vanish. If the constraint
value is observed exactly in every round, this term is zero. In the
uncompressed weakly convex branch, the average squared proximal
distance is then $\widetilde{\mathcal{O}}(K^{-1/2})$, and hence the average
proximal distance is $\widetilde{\mathcal{O}}(K^{-1/4})$. When $\rho=0$, the
uncompressed convex objective gap is also
$\widetilde{\mathcal{O}}(K^{-1/2})$.

\item \textbf{Momentum switching.}
Choose
\[
\alpha=K^{-1/2},
\qquad
k_0=
\left\lceil
\frac12\sqrt K\log K
\right\rceil,
\qquad
\eta
=
\frac{D}
{L\sqrt{\kappa_{\mathrm{tr}}}K^{3/4}},
\]
for $K$ such that $k_0\le K-1$. Then
$K-k_0=\Theta(K)$ and, in the compressed convex branch,
\[
\max\{
F(\bar w)-F(w^\star),G(\bar w)
\}
\le
\widetilde{\mathcal{O}}\!\left(
\sqrt{\kappa_{\mathrm{tr}}}\,K^{-1/4}
\right).
\]
Consequently, by~\eqref{eq:kappa_tr_product},
\[
\max\{
F(\bar w)-F(w^\star),G(\bar w)
\}
\le
\widetilde{\mathcal{O}}\!\left(
\frac{K^{-1/4}}
{\sqrt{\qclient\qserver}}
\right)
\]
when $n/\mupdate$ is fixed.

In the uncompressed weakly convex branch,
$\kappa_{\mathrm{tr}}=1$, feasibility and the average squared proximal
distance are $\widetilde{\mathcal{O}}(K^{-1/4})$, and the average proximal
distance is $\widetilde{\mathcal{O}}(K^{-1/8})$. When $\rho=0$, the uncompressed
convex objective gap is also $\widetilde{\mathcal{O}}(K^{-1/4})$.

\item \textbf{STORM switching.}
Choose
\[
\alpha=K^{-2/3},
\qquad
k_0=
\left\lceil
\frac12K^{2/3}\log K
\right\rceil,
\qquad
\eta
=
\frac{D}
{L\sqrt{\kappa_{\mathrm{tr}}}K^{2/3}},
\]
for $K$ such that $k_0\le K-1$. Then
$K-k_0=\Theta(K)$ and, in the compressed convex branch,
\[
\max\{
F(\bar w)-F(w^\star),G(\bar w)
\}
\le
\widetilde{\mathcal{O}}\!\left(
\sqrt{\kappa_{\mathrm{tr}}}\,K^{-1/3}
\right).
\]
Consequently,
\[
\max\{
F(\bar w)-F(w^\star),G(\bar w)
\}
\le
\widetilde{\mathcal{O}}\!\left(
\frac{K^{-1/3}}
{\sqrt{\qclient\qserver}}
\right)
\]
when $n/\mupdate$ is fixed.

In the uncompressed weakly convex branch,
$\kappa_{\mathrm{tr}}=1$, feasibility and the average squared proximal
distance are $\widetilde{\mathcal{O}}(K^{-1/3})$, and hence the average proximal
distance is $\widetilde{\mathcal{O}}(K^{-1/6})$. When $\rho=0$, the uncompressed
convex objective gap is also $\widetilde{\mathcal{O}}(K^{-1/3})$.

\end{enumerate}
\end{corollary}

\begin{proof}

Throughout the proof, for any $N\le K$ of order $K$, let us define the squared-direction quantity
\begin{align*}
P_N
:=&
NL^2
+
\frac{4L\sigmagradient}{\sqrt{\Bgradient}}
\sqrt{
N\left(
N\log2+\frac1{\mupdate}\log\frac{c}{\delta}
\right)
}
+
\frac{4\sigmagradient^2}{\Bgradient}
\left(
N\log2+\frac1{\mupdate}\log\frac{c}{\delta}
\right),
\end{align*}
where $c>0$ denotes the fixed numerical constant appearing in the
corresponding theorem. Since $\delta,\mupdate,\Bgradient,L$, and
$\sigmagradient$ are fixed,
\[
P_N=\mathcal{O}(N).
\]
Likewise, every uniform-gradient factor of the form
\[
M_K
:=
L+
\frac{2\sigmagradient}{\sqrt{\Bgradient}}
\sqrt{
\log\frac{cK\mupdate}{\delta}
}
\]
satisfies
\[
M_K=\mathcal{O}(\sqrt{\log K})
=\widetilde{\mathcal{O}}(1).
\]
We use these two observations below while keeping all compression
dependence explicit.

\noindent\textbf{Direct switching.}\quad
Consider first the compressed branch. From
\eqref{eq:theta-direct-on}, after subtracting the direct
constraint-tracking radius, the optimization remainder has the form
\begin{align}
R_{\mathrm{dir}}
\le{}&
\frac{D^2}{2\eta K}
+
\frac{\eta}{2K}P_K
+
\frac{\eta M_K C_{\mathrm{op}}}{K}
\sqrt{KP_K}
+
\mathcal{O}(K^{-1/2}).
\label{eq:rate-direct-remainder}
\end{align}
The final $\mathcal{O}(K^{-1/2})$ term collects the directional-gradient and
update-participation terms, whose parameters are fixed independently
of $K$.

Since $P_K=\mathcal{O}(K)$ and $M_K=\widetilde{\mathcal{O}}(1)$,
\begin{equation}
R_{\mathrm{dir}}
=
\widetilde{\mathcal{O}}\!\left(
\frac1{\eta K}
+
\eta
+
\eta C_{\mathrm{op}}
+
K^{-1/2}
\right).
\label{eq:rate-direct-simplified}
\end{equation}
Because
$\kappa_{\mathrm{dir}}=1+C_{\mathrm{op}}$, the choice $\eta=\frac{D}{L\sqrt{\kappa_{\mathrm{dir}}K}}$ balances these terms at the level of their $K$ and compression dependence. More explicitly,
\begin{align*}
\frac1{\eta K}
&=
\mathcal{O}\!\left(
\sqrt{\kappa_{\mathrm{dir}}}\,K^{-1/2}
\right),
\\
\eta(1+C_{\mathrm{op}})
&=
\mathcal{O}\!\left(
\sqrt{\kappa_{\mathrm{dir}}}\,K^{-1/2}
\right).
\end{align*}
Since $\kappa_{\mathrm{dir}}\ge1$, the remaining
$K^{-1/2}$ term is absorbed into the same bound. Hence $R_{\mathrm{dir}} = \widetilde{\mathcal{O}}\!\left( \sqrt{\kappa_{\mathrm{dir}}}\,K^{-1/2} \right).$

The direct compressed objective bound
\eqref{eq:direct-on-objective} subtracts the value-tracking radius from
the threshold, and therefore contains only this optimization
remainder:
\[
F(\bar w)-F(w^\star)
\le
\widetilde{\mathcal{O}}\!\left(
\sqrt{\kappa_{\mathrm{dir}}}\,K^{-1/2}
\right).
\]
The feasibility bound \eqref{eq:direct-on-feasibility} contains the
same optimization remainder plus twice the direct tracking radius.
Consequently,
\[
G(\bar w)
\le
\widetilde{\mathcal{O}}\!\left(
\sqrt{\kappa_{\mathrm{dir}}}\,K^{-1/2}
+
\sqrt{V_{\mathrm{val}}}
\right).
\]

In the uncompressed branch,
$\kappa_{\mathrm{dir}}=1$ and the same calculation gives the original
$\widetilde{\mathcal{O}}(K^{-1/2})$ optimization scale. If the constraint value
is observed exactly, the tracking radius is zero; the weakly convex
proximal-stationarity bound then gives average squared proximal
distance $\widetilde{\mathcal{O}}(K^{-1/2})$, and Cauchy--Schwarz yields average
proximal distance $\widetilde{\mathcal{O}}(K^{-1/4})$.

Finally, from~\eqref{eq:kappa_dir_product},
\[
\sqrt{\kappa_{\mathrm{dir}}}
\le
\frac{
\sqrt{1+8\sqrt2+2n/\mupdate}
}
{\sqrt{\qclient\qserver}},
\]
which gives the stated
$1/\sqrt{\qclient\qserver}$ compression dependence.

\noindent\textbf{Momentum switching.}\quad
Let
\[
\alpha=K^{-1/2},
\qquad
k_0=
\left\lceil
\frac{\log K}{2\alpha}
\right\rceil,
\qquad
\eta
=
\frac{D}
{L\sqrt{\kappa_{\mathrm{tr}}}K^{3/4}}.
\]
Then
\[
k_0=\mathcal{O}(\sqrt K\log K),
\qquad
K-k_0=\Theta(K),
\]
and
\[
(1-\alpha)^{k_0}
\le
e^{-\alpha k_0}
\le
K^{-1/2}.
\]

We first control the momentum tracking radius. From
\eqref{eq:momentum-track-on},
\begin{align}
\Gamma_{\mathrm{mom}}
\le
(1-\alpha)^{k_0}
\sqrt{2V_{\mathrm{val}}\log\frac{12}{\delta}}
+
\sqrt{
\frac{2\alpha}{2-\alpha}
V_{\mathrm{val}}
\log\frac{12K}{\delta}
}
+
\frac{1-\alpha}{\alpha}
L\eta M_K C_{\mathrm{mov}}.
\label{eq:rate-mom-track}
\end{align}
The initialization term is
$\widetilde{\mathcal{O}}(K^{-1/2})$, while the geometrically averaged value-noise
term is
\[
\widetilde{\mathcal{O}}(\sqrt{\alpha})
=
\widetilde{\mathcal{O}}(K^{-1/4}).
\]
For the movement term,
\begin{align*}
\frac{\eta C_{\mathrm{mov}}}{\alpha}
=
\frac{D}{L}
\frac{C_{\mathrm{mov}}}
{\sqrt{\kappa_{\mathrm{tr}}}}
K^{-1/4}
\le
\frac{D}{L}
\sqrt{\kappa_{\mathrm{tr}}}
K^{-1/4},
\end{align*}
because $C_{\mathrm{mov}}\le\kappa_{\mathrm{tr}}$.
Since $\kappa_{\mathrm{tr}}\ge1$,
\[
\Gamma_{\mathrm{mom}}
=
\widetilde{\mathcal{O}}\!\left(
\sqrt{\kappa_{\mathrm{tr}}}\,K^{-1/4}
\right).
\label{eq:rate-mom-track-final}
\]

We next control the optimization remainder. Set
$N:=K-k_0=\Theta(K)$. Reading the terms inside
\eqref{eq:theta-momentum-on} gives
\begin{align}
R_{\mathrm{mom}}
=
\widetilde{\mathcal{O}}\!\left(
\frac1{\eta N}
+
\frac{k_0}{N}
+
\eta
+
\eta C_{\mathrm{op}}
+
N^{-1/2}
\right).
\label{eq:rate-mom-rem}
\end{align}
Indeed, $P_N=\mathcal{O}(N)$ gives the $\eta$ contribution from the
squared-server-direction term, while
\[
\frac{\eta M_KC_{\mathrm{op}}}{N}
\sqrt{NP_N}
=
\widetilde{\mathcal{O}}(\eta C_{\mathrm{op}})
\]
is the operational-tracking contribution.

Under the stated choice of $\eta$,
\begin{align*}
\frac1{\eta N}
&=
\widetilde{\mathcal{O}}\!\left(
\sqrt{\kappa_{\mathrm{tr}}}\,K^{-1/4}
\right),
\\
\frac{k_0}{N}
&=
\widetilde{\mathcal{O}}(K^{-1/2}),
\\
\eta
&=
\mathcal{O}\!\left(
\frac{K^{-3/4}}
{\sqrt{\kappa_{\mathrm{tr}}}}
\right),
\\
\eta C_{\mathrm{op}}
&=
\mathcal{O}\!\left(
\frac{C_{\mathrm{op}}}
{\sqrt{\kappa_{\mathrm{tr}}}}
K^{-3/4}
\right)
\le
\mathcal{O}\!\left(
\sqrt{\kappa_{\mathrm{tr}}}\,
K^{-3/4}
\right),
\\
N^{-1/2}
&=
\mathcal{O}(K^{-1/2}).
\end{align*}
Thus the reciprocal stepsize term is the leading polynomial term and
$R_{\mathrm{mom}} = \widetilde{\mathcal{O}}\!\left( \sqrt{\kappa_{\mathrm{tr}}}\,K^{-1/4} \right).$

The compressed objective bound subtracts the momentum tracking radius from the threshold, giving $R_{\mathrm{mom}}$. The feasibility bound equals $R_{\mathrm{mom}}+2\Gamma_{\mathrm{mom}}$. Therefore
\[
\max\{
F(\bar w)-F(w^\star),G(\bar w)
\}
=
\widetilde{\mathcal{O}}\!\left(
\sqrt{\kappa_{\mathrm{tr}}}\,K^{-1/4}
\right).
\]

For the uncompressed branch,
$\kappa_{\mathrm{tr}}=1$, so this reduces to
$\widetilde{\mathcal{O}}(K^{-1/4})$. The weakly convex theorem then yields
\[
\frac1{|\mathcal A|}
\sum_{k\in\mathcal A}
\|w_k-\widehat w(w_k)\|_2^2
=
\widetilde{\mathcal{O}}(K^{-1/4}),
\]
and hence
\[
\frac1{|\mathcal A|}
\sum_{k\in\mathcal A}
\|w_k-\widehat w(w_k)\|_2
=
\widetilde{\mathcal{O}}(K^{-1/8}).
\]

Finally,~\eqref{eq:kappa_tr_product} gives $\sqrt{\kappa_{\mathrm{tr}}}\le\frac{\sqrt{1+32\sqrt2+2n/\mupdate}}{\sqrt{\qclient\qserver}}$, which yields the stated $1/\sqrt{\qclient\qserver}$ compression dependence.

\noindent\textbf{STORM switching.}\quad
Let
\[
\alpha=K^{-2/3},
\qquad
k_0=
\left\lceil
\frac{\log K}{2\alpha}
\right\rceil,
\qquad
\eta
=
\frac{D}
{L\sqrt{\kappa_{\mathrm{tr}}}K^{2/3}}.
\]
Then
\[
k_0=\mathcal{O}(K^{2/3}\log K),
\qquad
K-k_0=\Theta(K),
\]
and
\[
(1-\alpha)^{k_0}
\le
e^{-\alpha k_0}
\le
K^{-1/2}.
\]

We first consider the STORM tracking radius. Write $C_\Delta:=L^2\chi_{n,\mvalue}^2+\frac{\sigma_\Delta^2}{\mvalue\Bvalue},$ which is fixed independently of $K$. From
\eqref{eq:storm-track-on},
\begin{align}
\Gamma_{\mathrm{STORM}}
\le{}&
(1-\alpha)^{k_0}
\sqrt{2V_{\mathrm{val}}\log\frac{12}{\delta}}
+
\left\{
\frac{4\log(12K/\delta)}
{\alpha(2-\alpha)}
\left[
\alpha^2V_{\mathrm{val}}
+
C_\Delta
\eta^2M_K^2
C_{\mathrm{mov}}^2
\right]
\right\}^{1/2}.
\label{eq:rate-storm-track}
\end{align}

The initialization contribution is
$\widetilde{\mathcal{O}}(K^{-1/2})$. Since $2-\alpha\ge1$, the value-noise part
inside the second term satisfies
\[
\left(
\frac{\alpha^2}
{\alpha(2-\alpha)}
\right)^{1/2}
=
\mathcal{O}(\sqrt{\alpha})
=
\mathcal{O}(K^{-1/3}).
\]
The movement part satisfies
\begin{align*}
\frac{\eta C_{\mathrm{mov}}}{\sqrt{\alpha}}
=
\frac{D}{L}
\frac{C_{\mathrm{mov}}}
{\sqrt{\kappa_{\mathrm{tr}}}}
K^{-1/3}
\le
\frac{D}{L}
\sqrt{\kappa_{\mathrm{tr}}}
K^{-1/3}.
\end{align*}
Therefore
\begin{equation}
\Gamma_{\mathrm{STORM}}
=
\widetilde{\mathcal{O}}\!\left(
\sqrt{\kappa_{\mathrm{tr}}}\,K^{-1/3}
\right).
\label{eq:rate-storm-track-final}
\end{equation}

We now turn to the optimization remainder. With
$N:=K-k_0=\Theta(K)$, the terms inside
\eqref{eq:theta-storm-on} have exactly the same optimization structure
as in the momentum case:
\begin{equation}
R_{\mathrm{STORM}}
=
\widetilde{\mathcal{O}}\!\left(
\frac1{\eta N}
+
\frac{k_0}{N}
+
\eta
+
\eta C_{\mathrm{op}}
+
N^{-1/2}
\right).
\label{eq:rate-storm-rem}
\end{equation}
Substituting the STORM stepsize gives
\begin{align*}
\frac1{\eta N}
&=
\widetilde{\mathcal{O}}\!\left(
\sqrt{\kappa_{\mathrm{tr}}}\,K^{-1/3}
\right),
\\
\frac{k_0}{N}
&=
\widetilde{\mathcal{O}}(K^{-1/3}),
\\
\eta
&=
\mathcal{O}\!\left(
\frac{K^{-2/3}}
{\sqrt{\kappa_{\mathrm{tr}}}}
\right),
\\
\eta C_{\mathrm{op}}
&=
\mathcal{O}\!\left(
\frac{C_{\mathrm{op}}}
{\sqrt{\kappa_{\mathrm{tr}}}}
K^{-2/3}
\right)
\le
\mathcal{O}\!\left(
\sqrt{\kappa_{\mathrm{tr}}}\,
K^{-2/3}
\right),
\\
N^{-1/2}
&=
\mathcal{O}(K^{-1/2}).
\end{align*}
Since $\kappa_{\mathrm{tr}}\ge1$, all of these terms are bounded by
$\widetilde{\mathcal{O}}(
\sqrt{\kappa_{\mathrm{tr}}}K^{-1/3})$. Hence
\begin{equation}
R_{\mathrm{STORM}}
=
\widetilde{\mathcal{O}}\!\left(
\sqrt{\kappa_{\mathrm{tr}}}\,K^{-1/3}
\right).
\label{eq:rate-storm-rem-final}  
\end{equation}

Combining \eqref{eq:rate-storm-track-final} and \eqref{eq:rate-storm-rem-final} with the compressed convex theorem gives $\max\{ F(\bar w)-F(w^\star),G(\bar w) \} = \widetilde{\mathcal{O}}\!\left( \sqrt{\kappa_{\mathrm{tr}}}\,K^{-1/3} \right)$.

In the uncompressed branch,
$\kappa_{\mathrm{tr}}=1$, and the weakly convex theorem gives $\frac1{|\mathcal A|} \sum_{k\in\mathcal A} \|w_k-\widehat w(w_k)\|_2^2 = \widetilde{\mathcal{O}}(K^{-1/3})$, and therefore $\frac1{|\mathcal A|} \sum_{k\in\mathcal A} \|w_k-\widehat w(w_k)\|_2 = \widetilde{\mathcal{O}}(K^{-1/6})$.

Finally, using~\eqref{eq:kappa_tr_product}, $\sqrt{\kappa_{\mathrm{tr}}}=\mathcal{O}\!\left(\frac1{\sqrt{\qclient\qserver}}\right)$ for fixed $n/\mupdate$. Thus the compressed STORM rate can be written as
\[
\widetilde{\mathcal{O}}\!\left(
\frac{K^{-1/3}}
{\sqrt{\qclient\qserver}}
\right).
\]
This completes the proof of Corollary~\ref{cor:rates}.
\end{proof}
\newpage
\section{Technical distinction from prior stochastic switching analyses}
\label{sec:safeef_discussion}

We briefly clarify why the stochastic-switching guarantees developed in
this work do not follow directly from the analysis of
\citet{islamov2025safe}.  The issues below concern the proof mechanism
rather than the underlying idea of switching between objective and
constraint updates.

Recall our notation
\[
I_k=\ind\{Z_k\leq\theta\},
\qquad
\mathcal A:=\{k:I_k=1\},
\qquad
\mathcal B:=\{k:I_k=0\},
\]
where $\mathcal A$ denotes the objective rounds and $\mathcal B$ the
constraint-correction rounds.  Let
\[
\Delta_k:=F(w_k)-F(w^\star).
\]

\paragraph{Randomized compression and the switching history.}
A first distinction arises when the compressor itself is randomized.
Suppose that a compressor $\calQ$ satisfies only an expected contraction
condition of the form
\[
\E_{\calQ}
\bigl[
\|\calQ(v)-v\|^2
\bigr]
\leq
(1-q)\|v\|^2.
\]
Then the compressed iterate $w_k$ is random even in the absence of
stochastic gradient or constraint oracles.  Consequently, the switching
variable
\[
I_k=\ind\{G(w_k)\leq\theta\}
\]
and hence the sets $\mathcal A$ and $\mathcal B$ are themselves random.

This matters when an argument first takes expectation with respect to the compression randomness and subsequently reasons about the cardinality of the realized switching sets. For example, suppose one obtains an unconditional bound of the form
\[
\E\!\left[
\eta\sum_{k\in\mathcal B}\theta
+
\eta\sum_{k\in\mathcal A}\Delta_k
\right]
\leq
\mathcal R_K,
\]
where $\mathcal R_K$ collects the remaining deterministic or previously controlled terms in the recursion. Such a bound does not imply pathwise that $\mathcal A\neq\varnothing$, nor does it imply that $|\mathcal A|\geq K/2$, because the set $\mathcal A$ itself depends on the same compression randomness over which the expectation is taken.

This issue is relevant to the deterministic-oracle result of
\citet[Theorem~2 and Appendix~D]{islamov2025safe}, whose stated compressor
class permits randomized contractive compressors.  Restricting the
compressors to deterministic operators, or assuming an almost-sure
contractive inequality and carrying out the argument pathwise, would
remove this particular difficulty, but changes the scope of the stated
result.

\paragraph{Signed objective gaps prevent removing nonconstant weights.}
A separate issue occurs in the high-probability stochastic analysis.
The proof of \citet{islamov2025safe} obtains a weighted inequality and
then removes the weights using deterministic upper and lower bounds on
them.  In our notation, the relevant step has the following abstract
form.  Suppose that on some event one has
\begin{equation}
\label{eq:safeef_weighted_form}
2\eta
\sum_{k\in\mathcal B}
\omega_k\theta
+
2\eta
\sum_{k\in\mathcal A}
\omega_k\Delta_k
\leq
R_K,
\end{equation}
where the deterministic weights satisfy
\[
\frac{1}{2C}
\leq
\omega_k
\leq
\frac{1}{C}.
\]
The subsequent argument replaces
\eqref{eq:safeef_weighted_form} by an unweighted inequality of the form
\begin{equation}
\label{eq:safeef_unweighted_form}
\frac{\eta}{C}
\sum_{k\in\mathcal B}\theta
+
\frac{\eta}{C}
\sum_{k\in\mathcal A}\Delta_k
\leq
R_K.
\end{equation}
The first replacement is valid because $\theta>0$.  The second one,
however, would require
\[
2\omega_k\Delta_k
\geq
\frac{1}{C}\Delta_k
\qquad
\text{for every } k\in\mathcal A.
\]
This implication holds when $\Delta_k\geq0$, but it reverses whenever
$\Delta_k<0$.

In functionally constrained optimization there is no general reason for $\Delta_k$ to be nonnegative.  The point $w^\star$ minimizes $F$ only
over the feasible region $\mathcal{W}=\{w:G(w)\leq0\}$ not over the entire domain $\mathcal{X}$.  Moreover, an objective round is triggered by the approximate condition $Z_k\leq\theta$, with $\theta>0$ and, in the stochastic setting, $Z_k\neq G(w_k)$ in general. Thus, an objective-round iterate may satisfy
\[
F(w_k)<F(w^\star).
\]
Consequently, the weight bounds alone do not justify
\eqref{eq:safeef_unweighted_form}.

Equivalently,
\[
2\sum_{k\in\mathcal A}\omega_k\Delta_k
-
\frac1C\sum_{k\in\mathcal A}\Delta_k
=
\sum_{k\in\mathcal A}
\left(2\omega_k-\frac1C\right)\Delta_k,
\]
and although every coefficient
$2\omega_k-\frac1C$ is nonnegative, the quantities $\Delta_k$ have no
fixed sign.  Hence the right-hand side has no fixed sign.

Importantly, this problem is not resolved by subsequently considering the case $\sum_{k\in\mathcal A}\Delta_k>0$. A positive total sum does not imply that the individual $\Delta_k$ are positive, and the nonconstant weights may correlate with positive and negative terms.  Thus the passage corresponding to Eqs.~(48)--(50) of \citet{islamov2025safe} cannot be obtained from the stated weight bounds alone. A repair would require, for example, an additional assumption ensuring $\Delta_k\geq0$ on every objective round, a uniform lower bound on negative objective gaps, or a different high-probability argument that retains the weights.

\paragraph{Origin of the weighted argument.}
The weighting mechanism used in the stochastic analysis of \citet{islamov2025safe} is motivated by the high-probability framework of \citet{liu2023high}, which was developed to avoid bounded-domain assumptions and obtain guarantees depending on the initial distance to an optimizer rather than on the diameter of the domain.  In that setting, the relevant objective gaps are nonnegative because the reference point is an optimizer of the objective over the optimization domain.

The same sign structure is no longer available in functionally constrained optimization.  Here $w^\star$ minimizes $F$ only over the feasible region,
and an iterate $w_k$ selected for an objective update need not itself be feasible.  Consequently, $F(w_k)-F(w^\star)$ may be negative. Therefore, although the weighting technique is well suited to controlling unbounded iterates, directly replacing the iteration-dependent weights by uniform lower bounds in a sum of such signed objective gaps is not justified.  This is precisely the difficulty in the passage corresponding to Eqs.~(48)--(49) of \citet{islamov2025safe}.
\newpage
\section{Theoretical challenges}
\label{sec:theoretical_challenges}
The main difficulty in our setting is that stochastic constraint estimation does not merely perturb an update direction: it determines which update the algorithm takes. At round $k$, the switching variable $I_k=\ind\{Z_k\leq\theta\}$ decides whether the method follows an objective subgradient or performs a constraint-correction step. Hence, an error in estimating $G(w_k)$ changes the optimization dynamics themselves, and the sets of objective and constraint rounds are random and coupled to the observed constraint values. The analysis must therefore control the constraint estimator uniformly along an adaptive trajectory, show that objective rounds actually occur, and establish progress along the resulting random subsequence.

This point also prevents us from directly relying on the closest existing federated switching analysis, Safe EF \citep{islamov2025safe}. As detailed in Appendix \ref{sec:safeef_discussion}, its stated guarantees are not established by the corresponding proofs. With randomized contractive
compression, the iterates and hence the realized switching sets are random; an inequality obtained after averaging over the same compression randomness does not by itself imply pathwise properties such as nonemptiness or a cardinality lower bound for the realized objective-round set. A separate issue appears in its high-probability stochastic argument, where iteration-dependent weights are removed from terms of the form $F(w_k)-F(w^\star)$. Such terms do not have a fixed sign in functionally constrained optimization: $w^\star$ minimizes $F$ only over the feasible region, whereas an iterate selected for an objective update need not itself be feasible. Consequently,
\[
    2\omega_k\bigl(F(w_k)-F(w^\star)\bigr)
    \geq
    \frac{1}{C}\bigl(F(w_k)-F(w^\star)\bigr)
\]
does not follow from $\omega_k\geq 1/(2C)$ when $F(w_k)-F(w^\star)<0$. Thus, the existing proof mechanism cannot simply be extended to the stochastic, partially participating setting considered here;
we instead work directly with the realized switching history and retain the signed progress terms until the switching argument is completed.

Recent work on unified variance-reduced estimation provides a useful
starting point for tracking stochastic constraint values \citep{luo2026unified}, but its expectation-constrained switching analysis is centralized and assumes uniformly bounded stochastic subgradients almost surely. In contrast, we only assume sub-Gaussian stochastic-gradient noise, so individual stochastic subgradients need not be bounded. More importantly, FL introduces a source of uncertainty that is absent in the centralized formulation. Even if every participating client evaluates its local constraint exactly, the server observes only a random subset of the client population. Consequently, the error in a fresh constraint estimate contains both data noise and participation noise,
\[
    V_{\mathrm{val}}
    =
    \frac{\sigma_{\mathrm{pp}}^2}{m_{\mathrm{val}}}
    +
    \frac{\sigma_{\mathrm{val}}^2}
         {m_{\mathrm{val}}B_{\mathrm{val}}}.
\]
The first term remains whenever $m_{\mathrm{val}}<n$, irrespective of how large the local constraint-evaluation batch becomes. Thus, with fixed participation and fixed local batch size, repeatedly forming a fresh constraint estimate does not make the switching decision increasingly accurate as the number of communication rounds grows.

Partial participation creates a second, more subtle difficulty: the client sampling used to decide the switch cannot in general also be treated as an unbiased sample for the subsequent model update. We therefore sample the constraint-evaluation cohort $C_k$ first, form $I_k$, and only afterwards draw an independent update cohort $S_k$. Conditional on the information available after the switching decision, $I_k$ is then fixed while $S_k$ is still a uniform client sample. In particular, if
\[
h_k^i
:=
I_k\nabla f_i(w_k)
+
(1-I_k)\nabla g_i(w_k),
\]
then independence gives
\[
\E\!\left[
    \frac{1}{m_{\mathrm{up}}}
    \sum_{i\in S_k}h_k^i
    \,\middle|\,
    \mathcal F_k^{\mathrm{eval}}
\right]
=
\frac{1}{n}\sum_{i=1}^n h_k^i.
\]
This identity is lost if the same cohort is reused. Indeed, if $S_k=C_k$, then $I_k$ is itself a function of $C_k$ and its constraint observations, so conditioning on the realized switch changes the distribution of the participating clients:
\[
\E\!\left[
    \frac{1}{m}
    \sum_{i\in C_k} h_k^i
    \,\middle|\, I_k
\right]
=
\sum_C
\Prob(C_k=C\mid I_k)
\frac{1}{m}\sum_{i\in C}h_k^i,
\]
which is generally different from
$n^{-1}\sum_{i=1}^n h_k^i$. Hence, reusing the evaluation cohort introduces selection bias between the switching event and the update direction and destroys the conditional martingale structure used in our analysis.

Finally, removing the fixed-budget error floor requires reusing constraint information across rounds, but the quantity being estimated is itself moving. Momentum averages the persistent value and participation noise, yet past observations become stale because the target changes from $G(w_{k-1})$ to $G(w_k)$. This produces the usual variance--tracking-bias tradeoff: stronger averaging reduces stochastic noise but increases lag. Our STORM estimator corrects this movement using same-sample evaluations at $w_k$ and $w_{k-1}$, replacing the leading deterministic tracking drift with a centered-difference error whose magnitude scales with the movement of the
iterate. The estimator parameters and optimization stepsize must therefore be controlled jointly. This coupling is further carried through bi-directional compression in the convex setting and, for weakly convex objectives, through a constrained proximal-stationarity argument rather than a conventional objective-gap analysis.
\newpage
\section{Related works}
\label{sec:rel}
FL has become a foundation in privacy-preserving learning since \citet{mcmahan2017communication} with key applications in healthcare~\citep{rieke2020future, peng2024depth}, battery management systems~\citep{wang2024adaptive,zhu2024collaborative,banabilah2022federated}, finance~\citep{wen2023survey,chatterjee2023federated} to name a few. While most existing FL algorithms focus on unconstrained optimization, many of these applications naturally involve constraints, such as fairness requirements, safety limits, or resource budgets~\citep{du2021fairness,huang2024federated,zhang2024unified,salehi2021federated,yang2022federated}. This motivates the development of constrained FL algorithms, which explicitly incorporate these constraints into the learning process to ensure reliable learning.

FedAvg, introduced in \citet{mcmahan2017communication}, remains the most widely used FL algorithm. In addition, there have been many variants of it to tackle different aspects of learning, such as data heterogeneity~\citep{karimireddy2020scaffold,seo2024relaxed,morafah2024towards}, system heterogeneity~\citep{gong2022fedadmm,li2020federated,wu2024fedbiot}, and fairness~\citep{li2021ditto,badar2024fairtrade}. Beyond FedAvg, ADMM-based FL algorithms have been explored to tackle heterogeneity in FL~\citet{acar2021federated,gong2022fedadmm,zhou2023federated}. Despite the rich FL literature proposed previously, the majority of works primarily focus on unconstrained FL problems, leaving a gap between constrained optimization and FL. 

The most common setting in which~\ref{eq:op_prb} is analyzed is with $n=1$, i.e., the centralized setting. Most works address this scenario by formulating it as a saddle-point problem and solving it using various primal-dual methods~\citep{nemirovski2004prox,chambolle2011first,bertsekas2014constrained,hamedani2021primal,zhang2022solving,hounie2023resilient,boob2024optimal}. While theoretically grounded, these approaches can be highly sensitive to the tuning of the dual variables and typically require prior knowledge of an upper bound on the optimal dual variable. Moreover, many of these methods involve projecting both primal and dual variables onto bounded sets, which are often unknown in practice, making implementation challenging. Recent work~\citep{luo2026unified} applies this recursive-estimation perspective to switching methods for expectation-constrained stochastic optimization with bounded stochastic gradients in the centralized setting. In federated optimization, however, the switching signal must be estimated from a partially observed client population. Thus, in addition to stochastic value noise, the constraint estimate is affected by client sampling, which persists even when the participating clients evaluate their local constraints exactly. Our analysis accounts for this additional source of uncertainty together with stochastic model updates by separating the clients used to evaluate the constraint from those used to form the subsequent update.

Another primal-only approach is the switching (sub-)gradient method (SGM) proposed by \citet{polyak1967general}. This method has been extensively studied across different settings. For instance, \citet{titov2018mirror,stonyakin2019mirror,alkousa2020modification} study mirror descent integrated with SGM. Additionally, \cite{lan2020algorithms} extended SGM to the stochastic regime, establishing iteration complexity results under both convex and strongly convex assumptions. The analysis has also been broadened to weakly convex problems in works such as \citet{huang2023oracle,liu2025single,ma2020quadratically,jia2025first} extend the analysis of SGM to the weakly convex setting. More recently, \citet{pmlr-v337-luo26a} study a distributed stochastic minimax problem with stochastic constraints using softmax-weighted switching under partial participation. Their formulation and assumptions differ substantially from ours. Firstly, they optimize worst-client objective and constraint functions and rely on conditions tailored to sampled maxima, whereas we consider population-average functional constraints, allow weakly convex objectives, and analyze recursive constraint-value tracking under client-sampling and stochastic-evaluation noise. Our focus is therefore complementary, emphasizing fixed-budget momentum/STORM switching, compression, and proximal-stationarity guarantees. We believe our analysis of the recursive constraint estimator could be applied to their setting to improve the oracle complexity. In the federated functionally constrained setting, \citet{islamov2025safe} provides a seminal extension of SGM incorporating communication compression, though their framework does not address local multi-step updates or partial client participation.

\newpage

\section{Experiment details}
\label{sec:exp_apdx}
The experiments were run on a cluster of 4 A100-40GB GPUs. We treated $\mathcal{X}=\mathbb{R}^d$ for the experiments.
\subsection{NP Classification}
\label{apdx:np_class}
\begin{table}[!ht]
    \caption{Detailed setting of NP classification experiment}
    \label{tab:detail_NP}
    \centering
    \begin{tabular}{lc|lc}
        \toprule
        Hyperparameter & Value & Hyperparameter & Value \\
        \midrule
        Communication rounds ($K$) & 2500 & Tolerance ($\theta$) & 0.05 \\
        Client participation ($\mupdate=\mvalue$) & 6 & Total number of clients ($n$) & 20 \\
        Compression rate ($s/d$) & 1.0 & Classifier dimension ($d$) & 30 \\
        Gradient update batch ($\Bgradient$) & 8 & Constraint value batch ($\Bvalue$) & 1 \\
        Step size (centralized/direct, $\eta$) & 1.0 & Step size (momentum/STORM, $\eta$) & 1.5\\
        Momentum parameter ($\alpha$) & 0.35 & Number of runs & 3\\
        \bottomrule
    \end{tabular}
\end{table}
We consider the constrained optimization problem in~\eqref{eq:op_prb}, where the objective is to minimize the empirical loss on the majority class while ensuring that the loss on the minority class remains below a prescribed tolerance. For each client $i$, the local objective and constraint are defined as
\begin{align}
    f_i(w) := \frac{1}{m_{i0}} \sum_{x \in \mathcal{D}_i^{(0)}} \phi(w;(x,0)), 
    \qquad
    g_i(w) := \frac{1}{m_{i1}} \sum_{x \in \mathcal{D}_i^{(1)}} \phi(w;(x,1)),
\end{align}
where $\mathcal{D}_i^{(0)}$ and $\mathcal{D}_i^{(1)}$ denote local samples of class $0$ and class $1$, respectively, and $m_{i0}, m_{i1}$ are their cardinalities. The function $\phi$ is the binary logistic loss,
\begin{equation}
    \phi(w;(x,y)) = -y \, w^\top x + \log\!\left(1 + e^{w^\top x}\right),
    \quad y \in \{0,1\}.
\end{equation}
This captures the NP paradigm: $F(w)$ measures performance on the majority class, while $\theta$ specifies the prescribed tolerance on the minority-class loss and is used as the switching threshold in our experiments.

We use the breast cancer dataset~\citep{breast_cancer_wisconsin_dataset}, containing $569$ samples with $30$ features. We utilize $80\%$ of the data for training and the rest for testing. The training data is then split in an IID fashion between $n=20$ clients, such that each client receives an equal number of samples and the same class ratio. 

For communication efficiency, both on the uplink and downlink sides, we evaluate the use of Top-$s$~\citep{haddadpour2021federated,shulgin2022shifted,gruntkowska2024improving,islamov2025safe} compressor, which transmits a fraction $s/d$ of coordinates. Performance is measured in terms of the majority-class objective loss $F(w)$, the minority-class constraint loss $G(w)$, and the number of rounds in which the feasibility condition $G(w) \leq \theta$ is violated. From Figure~\ref{fig:NPC_main_rounds}, we observe that both momentum and STORM switching achieve convergence of the majority-class objective $F(w)$ and satisfy the constraint $G(w)$, while naive direct estimation suffers from high estimation error.

The detailed hyperparameters are shown in Table \ref{tab:detail_NP}. We tune the learning rate from $\eta\in\{0.5, 1.0, 1.5, 2.0, 3.0\}$ and momentum parameter from $\alpha\in\{0.1, 0.2, 0.3, 0.35, 0.5\}$.

\subsection{Fair Classification}
\label{apdx:fair}
\begin{table}[!ht]
    \caption{Detailed setting of fair classification experiment on Adult dataset}
    \label{tab:detail_fairness}
    \centering
    \begin{tabular}{lc|lc}
        \toprule
        Hyperparameter & Value & Hyperparameter & Value \\
        \midrule
        Dirichlet parameter & 2.0 & Tolerance ($\theta$) & 0.05 \\
        Communication rounds ($K$) & 500 & Number of runs & 3\\
        Neural network dimension ($d$) & 6,529 & Total number of clients ($n$) & 10 \\
        Batch size per client ($\Bgradient=\Bvalue$) & 128 & Client participation ($\mupdate=\mvalue$) & 5 \\
        Learning rate ($\eta$) & 0.1 & Momentum parameter ($\alpha$) & 0.2\\
        Penalty parameter ($\mu$) & 10 & Dual step (IPP-ConEx) & 200\\
        \bottomrule
    \end{tabular}
\end{table}
Fair classification task is formulated by the local objective of binary cross-entropy loss and constraint of demographic parity: $$f_i(w) := \frac{1}{|\mathcal{D}_i|}\sum_{(x,y)\in\mathcal{D}_i}-\left[y\log(\pi(x;w)) + (1-y)\log(1-\pi(x;w))\right]$$
$$G(w) := \left|\frac{1}{|\mathcal{D}_p|}\sum_{x\in\mathcal{D}_p} \pi(x;w) - \frac{1}{|\mathcal{D}_u|}\sum_{x\in\mathcal{D}_u} \pi(x;w)\right|,$$
where the subscript $p$ and $u$ represent protected and unprotected groups respectively. Since the demographic parity is defined as the absolute difference between the average logits of protected and unprotected groups, the aggregation at the server is treated with extra care. At the server, the average logits were aggregated instead of the final constraint value to recalculate the weighted average of logits for the global constraint calculation:
$$\widehat{G}(w) := \left| \frac{1}{|\mathcal D_{C_k,p}|} \sum_{x\in\mathcal D_{C_k,p}}\pi(x;w) - \frac{1}{|\mathcal D_{C_k,u}|} \sum_{x\in\mathcal D_{C_k,u}}\pi(x;w) \right|,$$
where $\mathcal D_{C_k,p}$ and $\mathcal D_{C_k,u}$ denote the protected and unprotected samples belonging to clients in $C_k$. Here, we explore the setting of stochastic data sampling, heterogeneous client data distribution, and deep neural network, making the problem highly non-convex, stochastic, skewed, and non-smooth. The experiments were conducted on Adult dataset \citep{kohavi1996adult}. For the penalty-based methods, the penalty parameter was chosen from $\mu\in[0.1, 1.0, 10.0]$ and only the best results are presented. The detailed hyperparameters are found in Table~\ref{tab:detail_fairness}. We tune the learning rate from $\eta\in\{0.01, 0.03, 0.1\}$ and momentum parameter $\alpha\in\{0.05, 0.2,0.5,1.0\}$

\begin{figure}
    \centering
    \includegraphics[width=\textwidth]{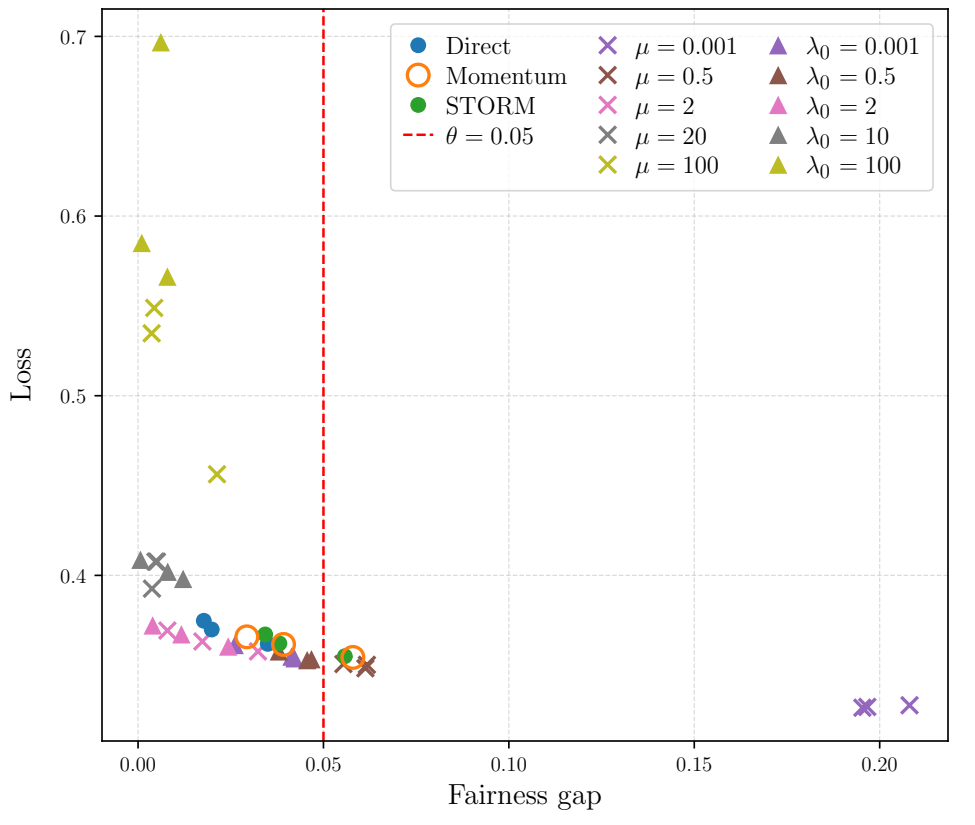}
    \caption{\textbf{Fair classification.} Comparison of our method against penalty-based FedAvg and primal-dual method. Each point represents a result from a separate experimental run with different random seeds. Circles (`$\bullet$') represent ours (direct, momentum and STORM), while crosses (`$\times$') and triangles (`$\blacktriangle$') represent penalty-based and primal-dual FedAvg, respectively.}
    \label{fig:pd_penalty_fairness}
\end{figure} 


\subsection{Constrained MDP}
\label{apdx:cartpole}

We formulate the problem as a constrained Markov decision process (CMDP), defined by the tuple $(\mathcal{S},\mathcal{U},r,c,P,\gamma,\nu_0)$. This extends the standard episodic discounted MDP with a cost function $c:\mathcal{S}\times\mathcal{U}\times\mathcal{S}\rightarrow\mathbb{R}$ that imposes a safety constraint. We assume $n$ clients, where each client $i$ interacts with a unique CMDP instance characterized by a specific safety budget $d_i$. The local objective and constraint are defined as $f_i(w)=-\mathbb{E}_{w,P}\left[\sum_{t=0}^{T-1}\gamma^t r(s_t,a_t,s_{t+1})\right], \ g_i(w)=\mathbb{E}_{w,P}\left[\sum_{t=0}^{T-1}\gamma^t c(s_t,a_t,s_{t+1})\right]-d_i$. We evaluate our method on the continuous Cartpole environment~\citep{brockman2016openaigym, towers2024gymnasium}, augmented with safety constraints. This setting introduces extreme heterogeneity, as the data distribution shifts dynamically due to client drift, stochastic policies, and different initial states.

Figure~\ref{fig:cartpole_main} demonstrates that momentum switching (orange) refines constraint value estimation, driving faster convergence in reward maximization compared to direct switching (blue). Notably, the momentum variant also achieves faster convergence than the full-information centralized baseline (black dashed line). We attribute this accelerated convergence to the noise induced by partial participation and stochastic estimation, acting as implicit regularization and encouraging the exploration necessary to escape starting positions surrounded by prohibited areas. These observations aligns with recent literature \citep{Li2020On,islamov2025safe}, which characterizes similar regularizing effects of these complications.

We omit the STORM variant in this experiment because it requires access to the counterfactual trajectory of the previous iteration's policy. While importance sampling (IS) can theoretically bridge this gap by reweighting the current policy's sample trajectory to approximate the previous policy \citep{yuan2020stochastic,huang2020momentum}, IS remains notoriously unstable in practice due to the high sensitivity of likelihood ratios to small changes in neural policy parameters \citep{huang2020momentum,papini2018stochastic,liu2020understanding}. Since our framework relies on the STORM estimator for discrete switching decisions rather than continuous gradient updates, standard variance reduction techniques, such as clipping importance weights \citep{huang2020momentum} or employing per-decision IS \citep{liu2020understanding}, failed to sufficiently stabilize the switching mechanism.

\begin{figure}[!t]
    \centering
    \includegraphics[width=\linewidth]{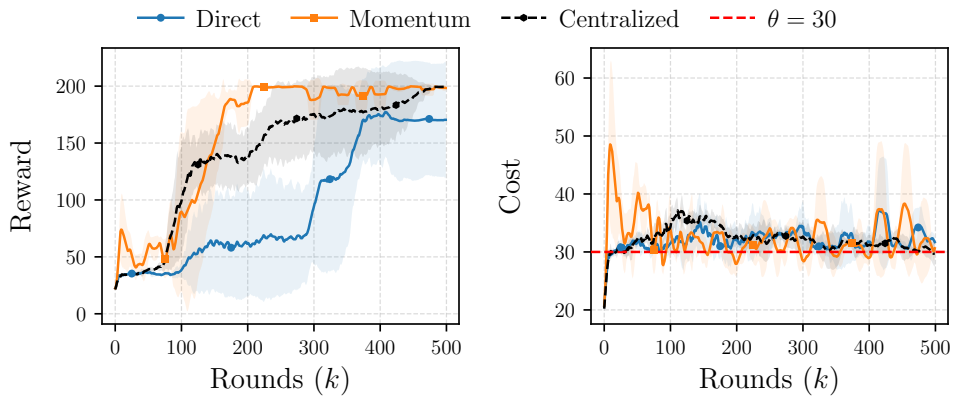}
    \caption{\textbf{Constrained MDP.} Evolution of episodic reward (left) and cost (right) on the Cartpole experiment. We compare direct and momentum switching (blue and orange solid lines) against the centralized baseline (black dashed line). The red dashed line marks the safety budget.}
    \label{fig:cartpole_main}
    \vspace{-6mm}
\end{figure}

\begin{table}[!ht]
    \caption{Detailed setting of CMDP experiment on Cartpole}
    \label{tab:detail_cartpole}
    \centering
    \begin{tabular}{lc|lc}
        \toprule
        Hyperparameter & Value & Hyperparameter & Value \\
        \midrule
        Communication rounds ($K$) & 500 & Tolerance ($\theta$) & 30\\
        Episode length & 200 & Discount factor ($\gamma$) & 1.0 \\
        Neural network dimension ($d$) & 17,410 & Total number of clients ($n$) & 10 \\
        Batch size per client ($\Bgradient=\Bvalue$) & 8 & Client participation ($\mupdate=\mvalue$) & 4\\
        Number of runs & 3 & Momentum paramter ($\alpha$) & 0.2\\
        \bottomrule
    \end{tabular}
\end{table}
We define our constrained Markov decision process (CMDP) by the tuple $(\mathcal{S},\mathcal{U},r,c,P,\gamma,\nu_0)$, which extends a standard episodic discounted MDP with a cost function $c:\mathcal{S}\times\mathcal{U}\times\mathcal{S}\rightarrow\mathbb{R}$ and a corresponding safety constraint. By standard notation, $\mathcal{S}$ and $\mathcal{U}$ are the state and action spaces, respectively, $r:\mathcal{S}\times\mathcal{U}\times\mathcal{S}\rightarrow\mathbb{R}$ is the reward function, $P:\mathcal{S}\times\mathcal{U}\times\mathcal{S}\rightarrow[0,1]$ is the stochastic transition kernel, $\gamma\in(0,1]$ is the discount factor and $\nu_0:\mathcal{S}\rightarrow[0,1]$ is the initial state distribution. Each of the $n$ clients interacts with a unique CMDP instance, distinguished by a client-specific safety budget $d_i$. The trajectories collected by each client are then used to compute their local constraint and gradient. The objective of each client is defined as 
$$f_i(w)=-\mathbb{E}_{w,P}\left[\sum_{\tau=0}^{T-1}\gamma^t r(s_\tau,a_\tau,s_{\tau+1})\right]$$
and the constraint as 
$$g_i(w)=\mathbb{E}_{w,P}\left[\sum_{\tau=0}^{T-1}\gamma^t c(s_\tau,a_\tau,s_{\tau+1})\right]-d_i$$
to solve the global optimization in \ref{eq:op_prb}, where $w$ represents the policy parameter. Constraint follows the two conditions used by \cite{xu2021crpo}: a player is given a cost of 1 every time step for (1) violating five prohibited areas defined as $[-2.4,-2.2], [-1.3,-1.1],[-0.1,0.1],[1.1,1.3],[2.2,2.4]$ or (2) exceeding the pole angle by more than 6 degrees. Detailed hyperparameters of the Cartpole experiment are provided in Table \ref{tab:detail_cartpole}.
\end{document}